\documentclass[nohyperref]{article}

\usepackage[accepted]{icml2023}
\usepackage{hyperref}       
\usepackage{url}   

\usepackage{amsfonts}       
\usepackage{nicefrac}       
\usepackage{microtype}      
\usepackage{adjustbox}
\usepackage{tabularx}
\usepackage{amsmath}
\usepackage{amssymb}
\usepackage{amsthm}
\usepackage{bbm}
\usepackage{blindtext}
\usepackage{enumitem}
\usepackage{multirow}
\usepackage{graphicx}
\usepackage{subcaption}

\theoremstyle{plain}
\newtheorem{assumption}{Assumption}
\newtheorem{theorem}{Theorem}
\newtheorem{definition}{Definition}

\newtheorem{lemma}{Lemma}
\newtheorem{remark}{Remark}

\usepackage{color}
\usepackage{xcolor}         

\renewcommand{\tilde}{\widetilde}

\usepackage[textsize=tiny]{todonotes}

\newcommand{\sayak}[1]{{\footnotesize\color{blue}[Sayak: #1]}}

\newcommand{\yulian}[1] {{\footnotesize\color{green}[Yulian: #1]}}

\usepackage{tikz}
	\usetikzlibrary{positioning}
	\definecolor{DarkGreen}{rgb}{0.2,0.6,0.2}
	\definecolor{DarkRed}{rgb}{0.6,0.2,0.2}
	\definecolor{DarkBlue}{rgb}{0.2,0.2,0.6}
	\definecolor{DarkPurple}{rgb}{0.4,0.2,0.4}   
\hypersetup{
        linktocpage=true,
        colorlinks=true,				
        linkcolor=DarkBlue,				
        citecolor=DarkBlue,				
        urlcolor=DarkBlue,			
    }

\icmltitlerunning{Differentially Private Episodic Reinforcement Learning with Heavy-tailed Rewards }
\begin{document}

\twocolumn[
\icmltitle{Differentially Private Episodic Reinforcement Learning \\with Heavy-tailed Rewards }


\icmlsetsymbol{equal}{*}

\begin{icmlauthorlist}
\icmlauthor{Yulian Wu}{prada,yyy}
\icmlauthor{Xingyu Zhou}{comp}
\icmlauthor{Sayak Ray Chowdhury}{sch}
\icmlauthor{Di Wang}{prada,yyy}

\end{icmlauthorlist}
\icmlaffiliation{prada}{ Provable Responsible AI and Data Analytics Lab}
\icmlaffiliation{yyy}{King Abdullah University of Science and Technology,
Saudi Arabia}
\icmlaffiliation{comp}{Wayne State University,  USA}
\icmlaffiliation{sch}{Microsoft Research, India}
\icmlcorrespondingauthor{Di Wang}{di.wang@kaust.edu.sa}

\icmlkeywords{differential privacy, Markov decision processes, heavy-tailed rewards}

\vskip 0.3in
]


\printAffiliationsAndNotice{}  
\begin{abstract}
In this paper, we study the problem of (finite horizon tabular) Markov decision processes (MDPs) with heavy-tailed rewards under the constraint of differential privacy (DP). Compared with the previous studies for private reinforcement learning that typically assume rewards are sampled from some bounded or sub-Gaussian distributions to ensure DP, we consider the setting where reward distributions have only finite $(1+v)$-th moments with some $v \in (0,1]$. By resorting to  robust mean estimators for rewards, we first propose two frameworks for heavy-tailed MDPs, i.e., one is for value iteration and another is for policy optimization. Under each framework, we consider both joint differential privacy (JDP) and local differential privacy (LDP) models.  Based on our frameworks, we provide regret upper bounds for both JDP and LDP cases and show that the moment of distribution and privacy budget both have significant impacts  on regrets. Finally, we establish a lower bound of regret minimization for heavy-tailed MDPs in JDP model by reducing it to the instance-independent lower bound of heavy-tailed multi-armed bandits  in DP model. We also show the lower bound for the problem in  LDP by adopting some private minimax methods. Our results reveal that there are fundamental differences between the problem of private RL with sub-Gaussian and that with heavy-tailed rewards.
\end{abstract}



\section{Introduction}
As a fundamental paradigm in decision-making problems, \textit{reinforcement learning} (RL), where an agent aims to maximize its long-term reward from interacting with an environment, has been widely applied in various fields such as finance\cite{xu2022towards}, healthcare \cite{gottesman2019guidelines}, transportation \cite{li2019efficient} and online recommendations \cite{zong2016cascading}. However, in these applications, there are some privacy issues as they always involve sensitive data \cite{pan2019you}. 
And users are less willing to disclose information and are more concerned about how their personal data are used \cite{das2021panel}. These privacy concerns in RL require us to design some reinforcement learning algorithms that achieve good performance while  can also protect the training environment's privacy.

In order to protect the information of sensitive data, \textit{differential privacy} (DP)\cite{dwork2006calibrating} has become a de facto technique for private data analysis. Over the past decade, differentially private reinforcement learning (DP-RL) has been extensively studied for various settings, including full information RL \cite{jain2012differentially} and partial information RL, e.g., \citet{mishra2015nearly,chowdhury2022shuffle,chowdhury2022distributed,vietri2020private}.
For private RL problems, it is natural to first consider achieving privacy protection in the standard DP model in \citet{dwork2006calibrating} where we treat each episode as a specific user and the agent collects the raw data of users and aims to  achieve privacy protection for history trajectories. However, \citet{shariff2018differentially} and \citet{dubey2021no} show that the standard  differential privacy provably incurs linear regret in contextual bandits (hence in RL as well) and such linear dependency is unavoidable.
Therefore, for general DP-RL, it is more reasonable to consider a relaxation of DP, namely  \textit{joint differential privacy} (JDP) \cite{kearns2014mechanism}, in which for each user $i$, knowledge of the other users does not reveal much about user $i$'s data. Moreover, in some situations, users are unwilling to share their data with the agent, so it is also common to consider another DP model, \textit{local differential privacy} (LDP), which 
perturbs  users' data locally before sending them to the central server so that only the data owners can access the original data.

As we mentioned above,  sensitive information is  contained in the states and rewards of trajectories for RL. Several methods of preserving the privacy of rewards have been proposed in the past few years.  However, these methods always need to assume the
rewards are sampled from light-tailed distributions, such as sub-Gaussian distributions, to ensure DP. However, in a wide variety of real-world systems such as economics \cite{ibragimov2015heavy}, medicine \cite{zhao2020type}, and market crashes \cite{schluter2008identifying}, rewards are often generated by certain heavy-tailed distributions. {For example, it has been shown that  RL is quite suitable to recommendation systems \cite{afsar2022reinforcement}, and more and
more companies such as Google \cite{chen2019top} are utilizing the power of RL to recommend better items to their customers. However, in such scenarios, the rewards, which correspond to users' feedback such as click, not click, or rating, always contain sensitive information  and 
always follow heavy-tailed or even long-tailed distributions  \cite{park2008long,celma2010long}.} 
Therefore, it is necessary 
to design private algorithms for  these RL problems with heavy-tailed rewards.

Motivated by these facts, in this paper, we focus on one fundamental model in RL under DP constraints, i.e., private (finite horizon tabular) Markov decision processes (MDPs) \cite{sutton2018reinforcement}, with heavy-tailed rewards in that the reward distribution of each state-action pair has only bounded $(1+v)$-th moment for some $v \in (0,1]$. To the best of our knowledge, we are the first to consider MDPs with heavy-tailed rewards in both joint and local DP models. And our contributions can be summarized as follows. 


 \noindent 
 1. We start from the most commonly used method, i.e., the value-iteration (VI) algorithm. Specifically, we present a general framework, Private-Heavy-UCBVI, for designing private optimistic VI algorithms based on some robust mean estimator for heavy-tailed distributions. To guarantee privacy, we use an adaptive version of the Tree-based mechanism and a new noise allocation method to achieve JDP, and we use the Laplacian mechanism for LDP. Based on this framework, we establish regret upper bounds for the problem in both JDP and LDP models to measure the performance of the framework.

  \noindent 
  2. Based on our private mechanisms, we then consider policy optimization (PO) based algorithms and propose a framework, namely Private-Heavy-UCBPO. 
    We also analyze the regret bounds for both JDP and LDP models based on this framework by developing some new theoretical techniques as byproducts. It is notable that this is also the first 
    PO-based algorithm for heavy-tailed MDPs, and there is even no previous PO algorithm  in the non-private case.

 \noindent 
 3. Finally, we study the lower bounds of our problem in JDP and LDP. In particular, for the JDP model, two unique challenges arise when applying the standard reduction from RL to MAB for minimax lower bounds. First, it is still open for the lower bound in the MAB case as highlighted in \citet{tao2022optimal}. We resolve this open problem by deriving the first lower bound for heavy-tailed MAB under the central DP model. Second, additional care is required to translate the lower bound for MAB under DP to the lower bound for RL under JDP. We resolve these challenges by using the notion of JDP with the public initial state as a bridge. 
For the LDP case, we derive the lower bound by providing some new hard instances of MDPs and using some private minimax methods. All the instances and methods can also be used in other private RL problems.

 We summarize our theoretical results in Table \ref{tab:my_label}. 
 Due to space limitations, some additional {algorithms} and sections, and all {proofs} and experiments  are included in Appendix.
\begin{table*}[htb]
    \centering
    \resizebox{\textwidth}{!}{
    \begin{tabular}{|c|c|c|c|c|c|}
    \hline
      \textbf{Problem}   & \textbf{Reward/ Cost} &  \textbf{DP} & \textbf{Algorithm} & \textbf{Upper bound} & \textbf{Lower bound}\\
     \hline
     \multirow{2}{*}{MAB} & \multirow{2}{*}{Heavy-tailed}& $\epsilon$-DP& Robust-SE & $\tilde{O}\left(\left(\frac{A}{\epsilon}\right)^{\frac{v}{1+v}}K^{\frac{1}{1+v}}\right)$ \cite{tao2022optimal} & $\color{blue}{\Omega\left(\left(\frac{A}{\epsilon}\right)^{\frac{v}{1+v}}K^{\frac{1}{1+v}}\right)}$\\
     \cline{3-6}
    &&$\epsilon$-LDP & Robust-SE&$\tilde{O}\left(\left(\frac{A}{\epsilon^2}\right)^{\frac{v}{1+v}}K^{\frac{1}{1+v}}\right)$ \cite{tao2022optimal} & $\Omega\left(\left(\frac{A}{\epsilon^2}\right)^{\frac{v}{1+v}}K^{\frac{1}{1+v}}\right)$ \cite{tao2022optimal}\\
    \hline
     \multirow{4}{*}{MDPs} & \multirow{4}{*}{Bounded} & \multirow{2}{*}{$\epsilon$-JDP} &UCB-VI & $\widetilde{O}\left(\sqrt{SAH^3T} +  S^2AH^3/\epsilon\right)$\cite{chowdhury2021differentially}& \multirow{2}{*}{$\Omega\left(\sqrt{HSAT}+\frac{SAH\log T}{\epsilon}\right)$\cite{vietri2020private}}\\
     \cline{4-5}
     & & &UCB-PO & $\widetilde{O}\left(\sqrt{S^2AH^3T} +  S^2AH^3/\epsilon\right)$\cite{chowdhury2021differentially}&\\
     \cline{3-6}
     & & \multirow{2}{*}{$\epsilon$-LDP} & UCB-VI &$\widetilde{O}\left(\sqrt{SAH^3T} +  S^2A\sqrt{H^5 T}/\epsilon\right)$\cite{chowdhury2021differentially}&\multirow{2}{*}{$\Omega\left(\frac{H\sqrt{SAK}}{\min\{e^\epsilon-1,1\}}\right)$ \cite{garcelon2021local}}\\
     \cline{4-5}
     & & &UCB-PO & $\widetilde{O}\left(\sqrt{S^2AH^3T} +  S^2A\sqrt{H^5T}/\epsilon\right)$\cite{chowdhury2021differentially}&\\
    \hline
    \multirow{4}{*}{MDPs} & \multirow{4}{*}{Heavy-tailed} & \multirow{2}{*}{$\epsilon$-JDP} &UCB-VI & \color{blue}{$\tilde{O}\left(\sqrt{SAH^3T} + \frac{S^2AH^3}{\epsilon} + \left(\frac{SAH^2}{\epsilon}\right)^{\frac{v}{1+v}}T^{\frac{1}{1+v}}\right)$ }& \multirow{2}{*}{\color{blue}{$\Omega\left(\left(\frac{SA}{\epsilon}\right)^{\frac{v}{1+v}}K^{\frac{1}{1+v}}\right)$}}\\
     \cline{4-5}
     & & &UCB-PO & \color{blue}{$\tilde{O}\left(\sqrt{S^2AH^3T}+\frac{S^2AH^3}{\epsilon}+\left(\frac{SAH^2}{\epsilon}\right)^{\frac{v}{1+v}}T^{\frac{1}{1+v}}\right)$} &\\
     \cline{3-6}
     & & \multirow{2}{*}{$\epsilon$-LDP} & UCB-VI & \color{blue}{$\tilde{O}\left(\sqrt{SAH^3T}  + \frac{S^2A\sqrt{H^5T}}{\epsilon} + \left(\frac{H^3SA}{\epsilon^2}\right)^{\frac{v}{2(1+v)}}T^{\frac{2+v}{2(1+v)}}\right)$}&\multirow{2}{*}{\color{blue}{$\Omega\left(\left(\frac{SA}{\epsilon^2}\right)^{\frac{v}{1+v}}K^{\frac{1}{1+v}}\right)$}}\\
     \cline{4-5}
     & & &UCB-PO & \color{blue}{$\tilde{O}\left(\sqrt{S^2AH^3T}+\frac{S^2A\sqrt{H^5T}}{\epsilon} +\left(\frac{H^3SA}{\epsilon^2}\right)^{\frac{v}{2(1+v)}}T^{\frac{2+v}{2(1+v)}}\right)$}&\\
    \hline
    \end{tabular}} 
    \caption{\textit{Summary of our results and regret comparisons for private RL. All results are in the expected regret form. For the heavy-tailed reward distribution case, we assume the $(1+v)$-th raw moment of each reward distribution is bounded by 1 for some known $v \in (0,1]$. For the bounded reward case, we assume the rewards are in $[0,1]$. Here $T=KH$ is the total number of steps, where $K$ is the total number of episodes and $H$ is the number of steps per episode. $S$ is the number of states and $A$ is the number of actions. $\epsilon \in (0,1]$ is the privacy budget. $\tilde{O}(\cdot)$ hides $poly\log(S,A,T,1/\delta)$ factors, where $\delta \in (0,1]$ is the desired confidence level. In MAB problem, $S=1,H=1$. We highlight our results in blue color.}}
    \label{tab:my_label}
\end{table*}
\section{Preliminaries}\label{sec:pre}
\subsection{MDPs with Heavy-tailed Rewards}
In a finite horizon Markov decision process (MDP), an agent needs to interact with the environment to make sequential decisions. We can formalize the problem by a tuple $(\mathcal{S},\mathcal{A},H,(P_h)_{h=1}^H,(r_h)_{h=1}^H)$, where $\mathcal{S}$ and $\mathcal{A}$ is the state and the action space with cardinality $S$ and $A$ respectively, $H \in \mathbb{N}$ is the episode length, $P_h(s^\prime|s,a)$ is the probability of transitioning to state $s^\prime$ from state $s$ provided action $a$ is taken at step $h$ and $r_h(s,a)$ is the mean of the reward distribution at step $h$. The actions are chosen following some policy $\pi=(\pi_h)_{h=1}^H$, where each $\pi_h$ is a mapping from the state space $\mathcal{S}$ into a probability distribution over the action space $\mathcal{A}$ i.e. $\pi_h(a|s) \ge 0$ and $\sum_{a\in \mathcal{A}}\pi_h(a|s)=1$ for each $s \in  \mathcal{S}$. Solving a reinforcement learning task means finding a policy $\pi$ that maximizes the long-term expected reward starting from every state $s\in \mathcal{S}$ and every step $h \in [H]$, defined as $
V_{h}^{\pi}(s):=\mathbb{E}\left[\sum_{h^{\prime}=h}^{H} r_{h^{\prime}}\left(s_{h^{\prime}}, a_{h^{\prime}}\right) \mid s_{h}=s, \pi\right],
$
where the expectation takes over the randomness of the transition kernel $P=(P_h)_{h=1}^H$ and the policy $\pi$. We call $V_{h}^{\pi}(s)$ as the \emph{value function} of a state $s$ under policy $\pi$ at step $h$. Now, defining the \emph{Q-function} of taking action $a$ in state $s$ under policy $\pi$ at step $h$ as $
Q_{h}^{\pi}(s, a):=\mathbb{E}\left[\sum_{h^{\prime}=h}^{H} r_{h^{\prime}}\left(s_{h^{\prime}}, a_{h^{\prime}}\right) \mid s_{h}=s, a_{h}=a, \pi\right],
$ we obtain $Q_{h}^{\pi}(s, a)=r_h(s,a)+\sum_{s^\prime \in \mathcal{S}}V_{h+1}^{\pi}(s^\prime)P_h(s^\prime|s,a)$ and $V_{h}^{\pi}(s)=\sum_{a \in \mathcal{A}} Q_{h}^{\pi}(s, a) \pi_{h}(a|s)$.

We call a policy $\pi^*$ optimal if it maximizes the value function of all states $s$ and steps $h$ simultaneously, and the corresponding optimal value function is denoted by $V_{h}^{*}(s)=\max_{\pi \in \Pi}V_{h}^{\pi}(s)$ for all $h\in [H]$, where $\Pi$ is the set of all non-stationary policies. The agent interacts with the environment for $K$ episodes to learn the unknown transition probabilities $P_h(s^\prime|s,a)$ and mean rewards $r_h(s,a)$, and thus, in turn, the optimal policy $\pi^*$. At each episode $k$, the agent chooses a policy $\pi^k = (\pi_h^k)_{h=1}^H$ and samples a trajectory $\left\{s_{1}^{k}, a_{1}^{k}, r_{1}^{k}, \ldots, s_{H}^{k}, a_{H}^{k}, r_{H}^{k}, s_{H+1}^{k}\right\}$ by interacting with the MDP using this policy. Here, at a given step $h$, $s_h^k$ denotes the state of the MDP, $a_h^k \sim \pi_h^k(\cdot|s_h^k)$ denotes the action taken by the agent, $r_h^k$ denotes the random reward obtained by the agent with the mean value $r_h(s_h^k,a_h^k)$ and $s_{h+1}^k \sim P_h(\cdot|s_h^k,a_h^k)$ denotes the next state. 

We consider a heavy-tailed setting in this paper where the reward distribution of each state-action pair $(s,a)$ at step $h$ only has the finite raw moment of order $1+v$ for some $v \in (0,1]$. Concretely, we assume that there is a constant $u>0$ such that at step $h$, for each state-action pair $(s,a)$ reward distribution $\mathcal{R}_h(s,a)$,
$\mathbb{E}_{X \sim \mathcal{R}_h(s,a)}[|X|^{1+v}]\le u.$
In this paper, we assume both $v$ and $u$ are known constants. Since this raw moment of $(1+v)$ for reward is finite, the expectation of reward random variable is also finite, and we denote 
$|r_h(s,a)|=|\mathbb{E}_{X \sim \mathcal{R}_h(s,a)}[X]| \le \tau.$
where $\tau$ is a known constant. 

We measure the agent's performance by using the cumulative regret accumulated over $K$ episodes, which is  defined as $
Reg(T):=\sum_{k=1}^{K}\left[V_{1}^{*}\left(s_{1}^{k}\right)-V_{1}^{\pi^{k}}\left(s_{1}^{k}\right)\right],
$
where $T=KH$ denotes the total number of steps and $s_1^k$ is the initial state. 

\subsection{DP in Episodic Reinforcement Learning}

In the episodic RL setting described above, it is natural to view each episode $k \in [K]$ as a trajectory associated with a specific user. To this end, we let $U_k=(u_1,\ldots,u_K) \in \mathcal{U}^K$ to denote a sequence of $K$ unique users participating in the private RL protocol with an RL agent $\mathcal{M}$, where $\mathcal{U}$ is the set of all users. Each user $u_k$ is identified by the reward and state responses $(r_h^k,s_{h+1}^k)_{h=1}^H$ she/he gives to the action $(a_h^k)_{h=1}^H$ chosen by the agent. We let $\mathcal{M}(U_K)=(a_1^1,\ldots,a_H^K) \in \mathcal{A}^{KH}$ to denote the set of all actions chosen by the agent $\mathcal{M}$ when interacting with the user sequence $U_K$ and let $\mathcal{M}_{-k}\left(U_{K}\right):=\mathcal{M}\left(U_{K}\right) \backslash\left(a_{h}^{k}\right)_{h=1}^{H}$ to denote all the actions chosen by the agent $\mathcal{M}$ excluding those recommended to $u_k$.



\begin{definition}[Joint Differential Privacy \cite{kearns2014mechanism}]
For any $\epsilon \ge 0$, a mechanism $\mathcal{M}:\mathcal{U}^K \rightarrow \mathcal{A}^{KH}$ is $\epsilon$-joint differential privacy (JDP) if for all $k \in [K]$, for all user sequences $U_K,U_K^\prime \in \mathcal{U}^K$ differing only on the $k$-th user and for all set of actions $\mathcal{A}_{-k} \subset \mathcal{A}^{(K-1)H}$ given to all but the $k$-th user
$
\mathbb{P}\left[\mathcal{M}_{-k}\left(U_{K}\right) \in \mathcal{A}_{-k}\right] \leq \exp (\varepsilon) \mathbb{P}\left[\mathcal{M}_{-k}\left(U_{K}^{\prime}\right) \in \mathcal{A}_{-k}\right].
$
\end{definition}

Local differential privacy is a more user-friendly model since it requires protecting each user's data $X=(s_h^k,a_h^k,r_h^k,s_{h+1}^k)_{h=1}^H$ before collection. We define local differential privacy for finite-horizon RL as follows:

\begin{definition}[Local Differential Privacy \cite{duchi2013local}] For any $\epsilon \ge 0$, a mechanism $\mathcal{M}^\prime$ is $\epsilon$-local differentially private (LDP) if for all trajectories $X, X^\prime \in \mathcal{X}$ and for all possible subsets $\mathcal{E}_{0} \subset\left\{\mathcal{M}^{\prime}(X) \mid X \in \mathcal{X}\right\}$ we have 
$
\mathbb{P}\left[\mathcal{M}^{\prime}(X) \in \mathcal{E}_{0}\right] \leq \exp (\varepsilon) \mathbb{P}\left[\mathcal{M}^{\prime}\left(X^{\prime}\right) \in \mathcal{E}_{0}\right]. 
$
\end{definition}


We introduce some notations for later analysis. We denote  the number of times that the agent has visited the state-action pair $(s,a)$ at step $h$ before episode $k$ as $N_h^k(s,a):=\sum_{k^\prime=1}^{k-1}\mathbbm{I}\{s_h^{k^\prime}=s,a_h^{k^\prime}=a\}$. Similarly, $N_h^k(s,a,s^\prime):=\sum_{k^\prime=1}^{k-1}\mathbbm{I}\{s_h^{k^\prime}=s,a_h^{k^\prime}=a,s_{h+1}^{k^\prime}=s^\prime\}$ denotes the count of going to state $s^\prime$ from $s$ upon playing action $a$ at step $h$ before episode $k$. Finally, we denotes the total \textbf{truncated} rewards obtained by taking action $a$ on state $s$ and  at step $h$ before episode $k$ as
\vspace{-1.5ex}{\footnotesize
\begin{equation}
\vspace{-1ex}
\label{trunctedR}
    R_h^k(s,a)=\sum_{k^\prime=1}^{k-1}\mathbbm{I}\{s_h^{k^\prime}=s,a_h^{k^\prime}=a,|r_h^{k^\prime}| \le B_{N_h^{k^\prime}(s,a)}\}r_h^{k^\prime},
\end{equation}}
where $B_{N_h^{k^\prime}(s,a)}$ is truncation threshold which is a non-decreasing function of ${N_h^{k^\prime}(s,a)}$ and to be set later. In the non-private case, these counters are sufficient to find estimates of the transition kernels $(P_h)_h$ and mean reward functions $(r_h)_h$ to design policy $(\pi_h^k)_h$ for episode $k$ for model-based MDP by using the table lookup model \cite{agarwal2019reinforcement}. However, in the private case, the challenge is that these counters depend on users’ states and reward responses to suggest further actions, which are considered as sensitive information. Therefore, we  must release these counts through some privacy-preserving mechanism namely {\bf PRIVATIZER} on which the learning agent would rely. To this end, we let $\Tilde{N}_h^k(s,a)$, $\Tilde{R}_h^k(s,a)$, and $\Tilde{N}_h^k(s,a,s^\prime)$ to denote the \textbf{privatized} version of $N_h^k(s,a)$, $R_h^k(s,a)$, and $N_h^k(s,a,s^\prime)$, respectively.

Now we make a general assumption on the counts released by the PRIVATIZER, which roughly means that the errors of private counts w.r.t the actual ones are bounded by some terms with high probability. Later on, we will show the specific PRIVATIZERs we use will automatically satisfy such an assumption.

\begin{assumption}[Properties of private counts]
\label{Assum1}
For any $\epsilon>0$ and $\delta \in (0,1]$, there exist functions $E_{\epsilon,\delta,1},E_{\epsilon,\delta,k,2},E_{\epsilon,\delta,3}>0$ such that with probability at least $1-\delta$, uniformly over all $(s,a,h,k)$, the private counts returned by the PRIVATIZER (both LOCAL and CENTRAL) satisfy: (i)$|\Tilde{N}_h^k(s,a)-N_h^k(s,a)| \le E_{\epsilon,\delta,1}$, (ii) $|\Tilde{R}_h^k(s,a)-R_h^k(s,a)| \le E_{\epsilon,\delta,k,2}$, and (iii)$|\Tilde{N}_h^k(s,a,s^\prime)-N_h^k(s,a,s^\prime)| \le E_{\epsilon,\delta,3}.$
\end{assumption}
Based on the above, then we define the private mean empirical rewards and private empirical transition probabilities  for all $(s,a,h,k)$ as 
{\footnotesize
\begin{equation}\label{PrivateMean}
\begin{aligned}
    \widetilde{r}_{h}^{k}(s,\! a)\! =\!\frac{\widetilde{R}_{h}^{k}(s, a)}{ 1 \!\vee\! ( \widetilde{N}_{h}^{k}(s, a)\!+\!E_{\varepsilon, \delta, 1\!})}, \!
\widetilde{P}_{h}^{k}\!(s^{\prime}\!|\!s,\! a) \!=\!\frac{\widetilde{N}_{h}^{k}\left(s, a, s^{\prime}\right)}{\! 1 \!\vee\! (\widetilde{N}_{h}^{k}(s, a)\!+\!E_{\varepsilon, \delta, 1}\!)}.
\end{aligned}
\end{equation}}
We refer the readers to Table \ref{tab:notation} in Appendix for all the above notations for convenience.

\section{Private Value-iteration for Heavy-tailed Rewards}\label{Sec:VI}

In the standard non-private RL setting, a straightforward way to get the optimal policy is to find the optimal value function, which can be determined by a simple iterative algorithm called \textit{value iteration} (VI) that has been shown to converge to the correct $V^*$ \cite{bellman1957dynamic,beutler1989dynamic}. Based on \textit{Upper Confidence Bound} (UCB) philosophy, the non-private UCB-VI method is proposed by \citet{azar2017minimax}, which takes some value-aware concentration results as the exploration bonus. Due to its simplicity, there are several private valued-based RL algorithms for private RL with bounded rewards \cite{chowdhury2021differentially,vietri2020private,garcelon2021local}. However, there is no UCB-VI algorithm for private RL with heavy-tailed rewards. 
In this section, we will start by designing value iteration-based algorithms for our problem in both JDP and LDP models.

Our general framework, Private-Heavy-UCBVI algorithm, is presented in Algorithm \ref{UCBVI}. The key idea of our algorithm is that we first establish $SAH$  parallel private counters for each tuple $(s,a,h)$ and $S^2AH$ parallel private counters for each tuple $(s,a,h,s^\prime)$, by using the (adaptive) Tree-based mechanism in JDP or the Laplacian mechanism in LDP to guarantee privacy. Based on these private counts, we design our algorithm by using a private and robust version of UCB at steps $7-9$ where the UCB bonus term is from the concentration results for our private estimators in Lemma \ref{ConcenPrivateVI} or Lemma \ref{ConcenPrivateVI_LDP}. At steps $9-10$, we compute private versions of the $Q$-function  and  the value function by using the optimistic Bellman recursion. Then a greedy policy $\pi^k$ is obtained by maximizing the private estimated $Q$-function at step 12. After rolling out the trajectory by acting the policy $\pi^k$, we truncate the rewards by an adaptive and non-decreasing truncation threshold $B_{N_h^k(s,a)}$ and translate all non-private statistics into private ones. 

\begin{algorithm}[htb]
    \caption{Private-Heavy-UCBVI}
    \label{UCBVI}
    \begin{algorithmic}[1]
        \REQUIRE Number of episodes $K$, time horizon $H$, privacy level $\epsilon >0$, reward mean bound $\tau$, a PRIVATIZER (Local or Central) and confidence level $\delta \in (0,1]$
        \STATE Initialize private counts $\Tilde{R}_h^1(s,a)=0,\Tilde{N}_h^1(s,a)=0,\Tilde{N}_h^1(s,a,s^\prime)=0$ for all $(s,a,s^\prime,h)$
        \STATE Set precision levels $E_{\epsilon,\delta,1},E_{\epsilon,\delta,k,2},E_{\epsilon,\delta,3}$ of the PRIVATIZER
        \FOR {$k=1,\dots,K$}
           \STATE Initialize private value estimates: $\Tilde{V}_{H+1}^k(s)=0$
           \FOR {$h=H,H-1,\dots,1$}
              \STATE Compute $\widetilde{r}_{h}^{k}(s, a)$ and $\widetilde{P}_{h}^{k}\left(s^{\prime}|s, a\right)$ for $\forall (s,a,s^\prime)$ as in \eqref{PrivateMean} using the private counts
              \STATE Set exploration bonus using Lemma \ref{ConcenPrivateVI}: $\beta_h^k(s,a)=\beta_h^{k,r}(s,a)+\beta_h^{k,pv}(s,a) \forall (s,a)$
              \STATE Compute: $\forall (s,a),$ $ 
              \Tilde{Q}_h^k(s,a)=
              \max\{-(H-h+1)\tau,\min\{(H-h+1)\tau, 
              \widetilde{r}_{h}^{k}(s, a)+\sum_{s^\prime \in \mathcal{S}}\Tilde{V}_{h+1}^k(s^\prime)\widetilde{P}_{h}^{k}\left(s^{\prime}|s, a\right)+\beta_h^k(s,a) \}\}
              $
             \STATE Compute private value function: $\forall s, \Tilde{V}_{h}^k(s)=\max_{a \in \mathcal{A}}\Tilde{Q}_h^k(s,a)$ 
           \ENDFOR
           \STATE Compute policy: $\forall (s,h), \pi_h^k(s)=\arg\max_{a \in \mathcal{A}}\Tilde{Q}_h^k(s,a)$
           \STATE Roll out a trajectory $(s_1^k,a_1^k,r_1^k,\dots,s_{H+1}^k)$ by acting the policy $\pi^k=(\pi_h^k)_{h=1}^H$
           \STATE Receive private counts $\Tilde{N}_h^{k+1}(s,a),\Tilde{N}_h^{k+1}(s,a,s^\prime)$ and truncated private        rewards summation $\Tilde{R}_h^{k+1}(s,a)$
        \ENDFOR
    \end{algorithmic}
\end{algorithm}

\subsection{Heavy-tailed Value-iteration in JDP}
As we mentioned earlier, different instances of PRIVATIZER correspond to different privacy models. For JDP, we will use CENTRAL-PRIVATIZER, which runs an adaptive version of the binary-tree mechanism  for each count $N_h^k(s,a),R_h^k(s,a),N_h^k(s,a,s^\prime)$, i.e., it uses $2SAH+S^2AH$ counters in total. In Algorithm \ref{Algo-Tree} of Appendix we provide the details of the mechanism. 
In total,  based on the composition theorem of DP we have the following result. 

\begin{lemma}[Privacy and Utility Guarantees of CENTRAL-PRIVATIZER]
\label{centralError}
For any $\epsilon>0$,   the CENTRAL-PRIVATIZER we mentioned above with Laplace noise $\text{Lap}\left(\frac{6B_k H\log K}{\epsilon}\right)$ for $\tilde{R}_h^k(s,a)$ and $\text{Lap}\left(\frac{3H\log K}{\epsilon}\right)$ for $\tilde{N}_h^k(s,a)$ and $\tilde{N}_h^k(s,a,s^\prime)$ is $\epsilon$-DP. Furthermore, for any $\delta \in (0,1]$, it satisfies Assumption \ref{Assum1} with $E_{\epsilon, \delta, 1}=\frac{3 H \log ^{1.5} K \ln \frac{3 S A T}{ \delta}}{\epsilon} , \quad E_{\epsilon, \delta,k, 2}= \frac{6B_k H \log ^{1.5} K \ln \frac{3 S A T}{ \delta}}{\epsilon} , \quad E_{\epsilon, \delta, 3}=\frac{3 H\log ^{1.5} K \ln \frac{3 S^{2} A T }{\delta}}{\epsilon} .$
\end{lemma}

\textbf{Failure of using total $\ell_1$ distance of all streams to get the counter error of rewards.}  It is notable that our PRIVATIZER is quite different from the previous methods in \citet{vietri2020private,chowdhury2021differentially}. If we adopt their methods then we will get larger errors. In detail, if we use their methods, then we need  to allocate or add the same noise for each episode. For each counter, it will take  the data stream $\sigma_h(s,a) \in [-B_K,B_K]^K$ as input since the $B_k$ is non-decreasing on $k$,  where the $j$-th entry 
\begin{equation}\label{eq:sigma}
    \sigma_h^j(s,a) :=\mathbbm{I}\{s_h^{j}=s,a_h^{j}=a,|r_h^{j}| \le B_{N_h^{j}(s,a)}\}r_h^{j}
\end{equation}
 denotes whether the pair $(s,a)$ is encountered or not at step $h$ of episode $j$
and if the pair is encountered, we will take the truncated reward. Consider its one adjacent data stream $\sigma_h^\prime(s,a) \in [-B_K,B_K]^K$ which differs from $\sigma_h(s,a)$ only in one entry, then we will have $\|\sigma_h(s,a)-\sigma_h^\prime(s,a)\|_1 \le 2B_K$. Furthermore, since at every episode at most $H$ state-action pairs are encountered, we obtain 
$
\sum_{(s, a, h) \in \mathcal{S} \times \mathcal{A} \times[H]}\left\|\sigma_{h}(s, a)-\sigma_{h}^{\prime}(s, a)\right\|_{1} \leq 2HB_K. 
$
Then we will get in each episode $k \in [K]$, we need to add noise $\text{Lap}\left(\frac{6B_K H\log K}{\epsilon}\right)$ which will make us get a loose error bound for reward count as in Lemma \ref{centralError} we just need $\text{Lap}(\frac{6B_k H\log K}{\epsilon})$ with $k\leq K$.

Based on the Billboard lemma in \citet{hsu2016private}, the composition of all $K$ episodes satisfies $\epsilon$-JDP if the policy $\pi^k$ is computed with an $\epsilon$-DP mechanism for all $k \in [K]$. 
\begin{theorem}
\label{thm:JDPpriv}
    For any $\epsilon> 0$, Algorithm \ref{UCBVI} is $\epsilon$-JDP if we use the CENTRAL-PRIVATIZER in Lemma \ref{centralError}. 
\end{theorem}

In the following, we will show the regret bound of our algorithm in the JDP model. 
\begin{theorem}[Regret Bound for Private-Heavy-UCBVI in JDP]
\label{thm:RegretVI}
For any $\epsilon \in (0,1]$ and $\delta \in (0,1]$ and take $B_n=\left(\frac{\epsilon u n}{H \log^{1.5}K \log (3SAT/\delta)}\right)^{\frac{1}{1+v}}$ in \eqref{trunctedR}. Then  if we use the CENTRAL-PRIVATIZER in Lemma \ref{centralError}, with probability $1-\delta$ the regret of Private-Heavy-UCBVI is upper bounded by
{$$  \tilde{O}\left(\sqrt{SAH^3T} + \frac{S^2AH^3}{\epsilon} + u^{\frac{1}{1+v}}\left(\frac{SAH^2}{\epsilon}\right)^{\frac{v}{1+v}}T^{\frac{1}{1+v}}\right).$$}
\end{theorem}
\begin{remark}
In the above bound, there are three terms. 
The first one corresponds to regret due to the uncertainty in transition probabilities. The second term comes from the estimation error of private counts and the third term is caused by the heavy-tailed nature of rewards. \citet{tao2022optimal} gives a regret rate of $O\left( \left(\frac{A\log T}{\epsilon}\right)^{\frac{v}{1+v}}T^{\frac{1}{1+v}}\right)$ for private heavy-tailed MAB  in the DP model where there are no transition probabilities and counts and $S=H=1$. Thus, our regret in JDP matches their bound in DP. 

In the non-private case of our problem, \citet{zhuang2021no} establishes a regret bound of $
    \tilde{O}(\sqrt{H^3SAT}+H^2(SA)^{\frac{v}{1+v}}T^{\frac{1}{1+v}}+\sqrt{H^9S^3A^3}
    +\sqrt{H^{\frac{1+4v}{v}}S^2A^2}+H^{\frac{1+3v}{v}}\sqrt{S^3A^3})$ 
by proposing the Heavy-Q-Learning with UCB-Bernstein algorithm. Compared with it, we can see the price of privacy is an additional factor of $\frac{1}{\epsilon}$ for private counts and $\left(\frac{1}{\epsilon}\right)^{\frac{v}{1+v}}$ for private estimation of heavy-tailed reward distributions. Besides, for the terms related to heavy-tailed distributions estimation 
, the dependency on $H$ is better in our bound by a factor of $H^{\frac{2}{1+v}}$ ($H^\frac{2v}{1+v}$ v.s. $H^2$). 

Compared with the result of $\tilde{O}\left(\sqrt{SAH^3T} + \frac{S^2AH^3}{\epsilon}\right)$ given by \citet{chowdhury2021differentially} for private MDPs with bounded rewards in JDP, we can see  there is an additional term due to the assumption of the heavy-tailed reward in our problem. This is because \citet{chowdhury2021differentially} assumes the rewards are in $[0,1]$ so that these rewards have the same sensitivity as counts. Hence, the regret corresponding to estimating the mean of heavy-tailed rewards is also bounded by ${S^2AH^3}/{\epsilon}$. Compared with the regret bound of $\widetilde{O}(\sqrt{S A H^2 T}+S^2 A H^3 / \epsilon)$ given by \citet{qiao2023near} under JDP by adopting DP-UCBVI for bounded rewards, the difference of the first term comes from the type of bonus since \citet{qiao2023near} adopts Bernstein type but we use Hoeffding type. And the additional third term in our bound is due to the heavy-tailed nature of the rewards.
    
\end{remark}

\subsection{Heavy-tailed Value-iteration in LDP}
Next, we consider the LDP case. We first introduce its corresponding PRIVATIZER, namely LOCAL-PRIVATIZER. For each episode $k$,  LOCAL-PRIVATIZER releases the private counts by injecting Laplace noise into each data stream. At each episode $j$, given privacy parameter $\epsilon'>0$, LOCAL-PRIVATIZER perturbs each entry $\sigma_h^j(s,a)$ of the data stream $\sigma_h(s,a)$ with an independent Laplace noise $\text{Lap}(\frac{1}{\epsilon'})$, i.e.,  $\tilde{\sigma}_h^j(s,a)=\sigma_h^j(s,a)+\text{Lap}(\frac{1}{\epsilon'})$, where $\sigma_h^j(s,a)$ is in (\ref{eq:sigma}). The private counts for the $k$-th episode are computed as $\widetilde{ N}_h^k(s,a)=\sum_{j=1}^{k-1}\widetilde{\sigma}_h^j(s,a)$. The private counts corresponding to empirical rewards $R_h^k(s,a)$ and state transitions $N_h^k(s,a,s')$ are computed similarly. On the one hand, for $N_h^k(s,a)$ with a fixed episode $k \in [K]$, we run $SAH$ parallel private counters, one for each tuple $(s,a,h)$. Thus, from the DP parallel Lemma (Lemma \ref{parallel}), if we want to guarantee the privacy mechanism satisfying $\epsilon/3$-LDP for all $\tilde{N}_h^k(s,a)$, we just need to make every  $\tilde{N}_h^k(s,a)$ be  $\epsilon/3$-LDP. On the other hand, at each episode at most $H$ state-action pairs are encountered, so according to the composition theorem (Lemma \ref{compositionThm}), we need to guarantee $\widetilde{\sigma}_h^j(s,a)$  is $\frac{\epsilon}{3H}$-LDP at each step $h$.  Then from the Laplacian mechanism, we use independent noise $\text{Lap}(\frac{3H}{\epsilon})$ in $N_h^k(s,a)$ and   $N_h^k(s,a,s')$. Similarly we set independent noise as $\text{Lap}(\frac{6HB_j}{\epsilon})$ to protect the privacy of $R_h^k(s,a)$. Based on the concentration property of Laplacian distributions (Lemma \ref{alemma5}), we can get the following error bounds for counts under LOCAL-PRIVATIZER.

\begin{lemma}[Privacy and utility guarantees of LOCAL-PRIVATIZER]
\label{ErrorLDP}
For any $\epsilon \in (0,1]$, the LOCAL-PRIVATIZER above is $\epsilon$-LDP. Furthermore, for any $\delta \in (0,1]$, it satisfies Assumption \ref{Assum1} with 
$
    E_{\epsilon,\delta,1}=\frac{6H}{\epsilon}\sqrt{K \log\frac{6SAT}{\delta}},
    E_{\epsilon,\delta,k,2}=\frac{12HB_k}{\epsilon}\sqrt{k \log\frac{6SAT}{\delta}},
    E_{\epsilon,\delta,3}=\frac{6H}{\epsilon}\sqrt{K \log\frac{6S^2AT}{\delta}}
$
\end{lemma}

\begin{remark}
Compared with the errors of private counts under JDP which depend on $\log K$ in Lemma \ref{centralError}, the above errors under LDP depend on polynomial of $K$. The reason is that in the LDP model, we add noise  to the data of each user, so we add more noise in total than it in the JDP model to guarantee stronger privacy.
\end{remark}

\begin{theorem}[Regret Bound for Private-Heavy-UCBVI in LDP]
\label{thm:RegretVI2}
For any $\epsilon \in (0,1]$ and $\delta \in (0,1]$ and take $B_n=\left(\frac{u\epsilon\sqrt{n}}{H \log(6SAT/\delta)}\right)^{\frac{1}{1+v}}$ in equation \eqref{trunctedR}. Then if we use the LOCAL-PRIVATIZER in Lemma \ref{ErrorLDP}, with probability $1-\delta$ the regret $Reg(T)$ of Private-Heavy-UCBVI satisfies
{\scriptsize $$\tilde{O}\left(\sqrt{SAH^3T}  + \frac{S^2A\sqrt{H^5T}}{\epsilon} + u^{\frac{1}{(1+v)}}\left(\frac{H^3SA}{\epsilon^2}\right)^{\frac{v}{2(1+v)}}T^{\frac{2+v}{2(1+v)}}\right).$$}
\end{theorem}

\begin{remark}
Compared with the regret bound for JDP, the differences of the second term and the third term in the above bound for LDP come from the fact that the noise magnitude of private count is  $\log k$ for JDP while it will be $\sqrt{k}$ for LDP. Moreover, compared with LDP heavy-tailed MAB, our bound cannot recover the optimal rate of  $O\left(\left({A\log T}/{\epsilon^2}\right)^{\frac{v}{1+v}}T^{\frac{1}{1+v}}\right)$ in \citet{tao2022optimal} with $S=H=1$. 
This is because to achieve the optimal rate for MAB, 
\citet{tao2022optimal} proposes a locally private and robust version of the successive elimination algorithm while we use UCB-based method for our LDP case. Since  the key ideas of these algorithms are quite different, we cannot adopt successive elimination methods for RL problems directly because of the existence of transition probabilities and states. We may get some regret bounds that  match the optimal rate in LDP heavy-tailed MAB by using some variants of  the UCB method, and we leave it as an open problem. Compared with the bound of $\widetilde{O}\left(\sqrt{S A H^2 T}+S^2 A \sqrt{H^5 T} / \epsilon\right)$ provided by \citet{qiao2023near} under LDP based on value iteration method for bounded rewards case, except for the difference in the first term as in JDP, the third term is different due to heavy-tailed rewards and the noise magnitude of private count is $\sqrt{k}$ for LDP.
\end{remark}


\section{Private Policy Optimization for Heavy-tailed Rewards }\label{Sec:PO}

In the previous section, we proposed a value-iteration-based framework for private heavy-tailed MDPs. However, in our algorithm, the value iteration function runs through all possible actions to find the maximum action value, i.e., Algorithm \ref{UCBVI} is computationally heavy. In the non-private MDPs with bounded/sub-Gaussian rewards case, it is well known that  policy optimization (PO) based algorithms proposed by \citet{pashenkova1996value} are more efficient than the value iteration based ones from the computational perspective. And due to this, in practice, researchers are more willing to use PO-based algorithms. Thus, a natural question is whether we  can design some private PO-based algorithm for MDPs with heavy-tailed rewards. 
 In this section, we give an affirmative answer to the question by proposing a policy-optimization-based framework, namely Private-Heavy-UCBPO, for private MDPs with heavy-tailed rewards. See Algorithm \ref{alg:PO} for details. 

The key idea of Private-Heavy-UCBPO is that we start by choosing a  policy $\pi^1$ from the uniform distribution and then we iteratively evaluate and improve the policy. In the policy evaluation stage, we use the UCB framework to compute a $Q$-function estimation where we use the estimation errors of private empirical reward means and private empirical transition probabilities as the exploration bonus, which is  similar to Algorithm \ref{UCBVI}. Then based on Bellman expectation equation, we compute the corresponding value function at step 11. Next, we roll out a new trajectory by acting the policy and receive the private sum of truncated rewards and private counts from the same PRIVATIZER as in Algorithm \ref{UCBVI}. Finally, in the policy improvement stage, we update the policy by leveraging a standard mirror-descent step.

Similar to the previous section, we will show the  regrets for Algorithm \ref{alg:PO} in JDP and LDP models, respectively.
\begin{theorem}
    Given $\epsilon>0$, by using the same CENTRAL-PRIVATIZER  (LOCAL-PRIVATIZER) as in Lemma \ref{centralError} (Lemma \ref{ErrorLDP}), Algorithm \ref{alg:PO} is $\epsilon$-JDP ($\epsilon$-LDP). 
\end{theorem}

\begin{theorem}[Regret bound of Private-Heavy-UCBPO in JDP]
\label{thm:regPO}
Fix any $\epsilon \in (0,1]$ and $\delta \in (0,1]$ and set $\eta=\sqrt{2\log A/(\tau^2H^2 K)}$ and take $B_n=\left(\frac{\epsilon u n}{H \log^{1.5}K \log (3SAT/\delta)}\right)^{\frac{1}{1+v}}$ in equation \eqref{trunctedR}. Then, if we use the CENTRAL-PRIVATIZER in Lemma \ref{centralError}, with probability at least $1-\delta$, the cumulative regret of Private-Heavy-UCBPO (Algorithm \ref{alg:PO}) is upper bounded by
$$\tilde{O}\left(\sqrt{S^2AH^3T}+\frac{S^2AH^3}{\epsilon}+u^{\frac{1}{1+v}}\left(\frac{SAH^2}{\epsilon}\right)^{\frac{v}{1+v}}T^{\frac{1}{1+v}}\right).$$
\end{theorem}

\begin{theorem}[Regret bound of Private-Heavy-UCBPO in LDP]
\label{thm:regPO2}
Fix any $\epsilon \in (0,1]$, $\delta \in (0,1]$ and set $\eta=\sqrt{2\log A/(\tau^2H^2 K)}$. Take $B_n=\left(\frac{u\epsilon\sqrt{n}}{H \log(6SAT/\delta)}\right)^{\frac{1}{1+v}}$ in \eqref{trunctedR}. Then if we use the same LOCAL-PRIVATIZER as in Lemma \ref{ErrorLDP}, with probability at least $1-\delta$, the cumulative regret of Private-Heavy-UCBPO (Algorithm \ref{alg:PO}) is upper bounded by 
{\scriptsize $$ \tilde{O}\left(\sqrt{S^2AH^3T}+\frac{S^2A\sqrt{H^5T}}{\epsilon} +u^{\frac{1}{1+v}}\left(\frac{H^3SA}{\epsilon^2}\right)^{\frac{v}{2(1+v)}}T^{\frac{2+v}{2(1+v)}}\right).$$}
\end{theorem}
Compared with the regret bounds of Private-Heavy-UCBVI, there is an additional factor of $\sqrt{S}$ in the leading privacy-independent term in the bounds of Private-Heavy-UCBPO. This follows the same pattern as in private bounded MDPs case \cite{chowdhury2021differentially}.  Actually,  there is no previous policy-optimization-based algorithm, even  for non-private heavy-tailed MDPs. Therefore, our proof techniques (such as the results of Lemma \ref{lem:heavyCount} in Appendix) can be considered as byproducts that can be used to  design policy-optimization-based algorithms for non-private heavy-tailed RL problems.

\section{Lower Bounds}
\label{Sec:LowerBou}

\subsection{Regret Lower Bound under JDP}
We first focus on establishing a lower bound on the regret for heavy-tailed MDPs under JDP constraint. In the literature, a common approach for proving lower bounds in RL is via a reduction to MAB. However, in our setting, two challenges will arise when we adopt the above approach. First, the regret of our problem corresponds to the instance-independent
regret in MAB. Unfortunately, there is no existing private minimax (instance-independent) regret lower bound for heavy-tailed MAB in the central DP model. In fact, such a lower bound is listed as an open problem in ~\citet{tao2022optimal}. Second, in the private case, the reduction from RL to MAB needs additional care. This is due to the difference in privacy definition, i.e., JDP for RL and DP for MAB. In other words, even if one has established a regret lower bound for  MAB in the central DP model, it cannot be directly used to establish a lower bound for RL in JDP. 

To address the aforementioned challenges, we take the following steps in this section. First, we establish the first private minimax lower bound for heavy-tailed MAB in central DP, hence resolving the open problem in ~\citet{tao2022optimal}. The key step behind this result is  a \emph{private} version of KL-divergence for the hard instances. Then, to tackle the second challenge, we use the notion called JDP with \emph{public} initial state as a bridge between the classic JDP for RL and DP for MAB, which allows us to build upon our lower bound for MAB under DP to establish the lower bound for RL in JDP. Here we only show our intuition for overcoming  the first challenge and our hard instances for heavy-tailed MDPs. See Appendix \ref{Sec:LowerBou} for the details of overcoming the second challenge. 



We now start with the MAB lower bound. In particular, we consider the agent interacts with the environment sequentially for $K$ rounds. The agent is faced with a set of $A$ independent arms $\{1,\dots,A\}$. In each round $k \in K$, the agent selects an arm $a_k \in [A]$ to pull and obtains a reward that is drawn i.i.d. from a fixed but unknown heavy-tailed distribution associated with the chosen arm.

\begin{definition}[Differential Privacy (DP) for MAB \cite{vietri2020private}]
For any $\epsilon \ge 0$, a mechanism $\mathcal{M}:\mathcal{U}^K \rightarrow \mathcal{A}^{K}$ is $\epsilon$-DP if for all user sequences $U_K,U_K^\prime \in \mathcal{U}^K$ differing only in a single user and for all events $E \subset \mathcal{A}^{K}$,we have
$
\mathbb{P}\left[\mathcal{M}{\left(U_{K}\right) \in E}\right] \leq e^\varepsilon \mathbb{P}\left[\mathcal{M}\left(U_{K}^{\prime}\right) \in E\right].
$
\end{definition}

\begin{theorem}[Instance-independent Lower Bound for DP Heavy-tailed MAB]
\label{thm:LowBounMAB}
There exists a heavy-tailed $A$-armed bandit instance with the $(1+v)$-th bounded moment of each reward distribution is bounded by 1. Moreover, if $K$ is large enough, for any  $\epsilon$-DP algorithm $\mathcal{M}$ with $\epsilon \in (0,1]$, the expected regret must satisfy
$${Reg}_K \ge \Omega\left(\left(A/\epsilon\right)^{\frac{v}{1+v}}K^{\frac{1}{1+v}}\right).$$
\end{theorem}


\begin{remark}
~\citet{tao2022optimal} gives an instance-independent upper bound of $O\left(\left({A\log K}/{\epsilon}\right)^{\frac{v}{1+v}}K^{\frac{1}{1+v}}\right)$ for heavy-tailed MAB in central DP by proposing a private and robust version of the successive elimination algorithm. Our above lower bound almost matches the upper bound up to a factor of $(\log K)^{\frac{v}{1+v}}$. Thus, the above lower bound is already near optimal. 
\end{remark}

Now, inspired by \cite{vietri2020private}, we construct hard instances for MDPs as depicted in Figure \ref{JDPfigure}. Within the class of MDPs,  state space is denoted as $\mathcal{S}:=[n]\cup\{+ ,-\}$ and action space is denoted as $\mathcal{A}:=[m]$. During each episode, the agent initiates from one of the initial states, randomly selected with equal probability from the set. Now, we construct hard instances for MDPs as depicted in Figure \ref{JDPfigure}. Within the class of MDPs,  state space is denoted as $\mathcal{S}:=[n]\cup\{+ ,-\}$ and action space is denoted as $\mathcal{A}:=[m]$. During each episode, the agent initiates from one of the initial states, randomly selected with equal probability from the set. For each initial state, the agent faces $m$ potential actions, and transitions can only lead it to either of the two terminal states, $\{+,-\}$.

\begin{figure}[htb]
    \centering
    \vspace{-1cm}
    \hspace{-1cm}\includegraphics[scale=0.28]{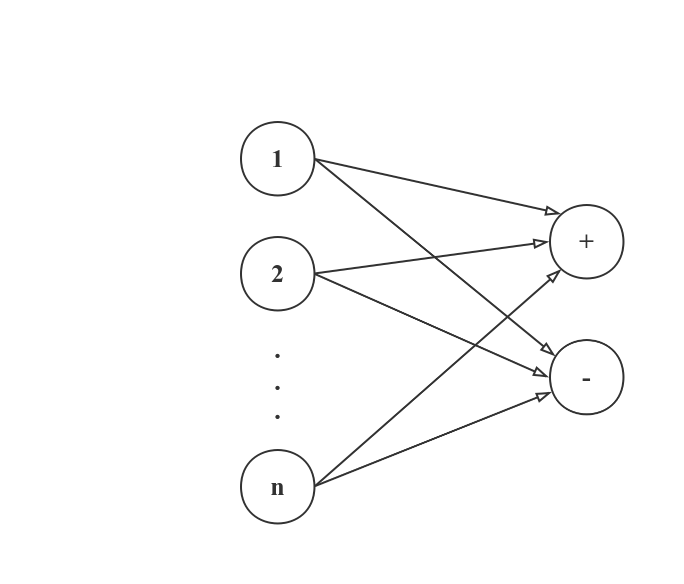}
    \caption{Hard MDP instance for JDP model}
    \label{JDPfigure}
\end{figure}

Such an instance can be regarded as the composition of $n$ parallel MAB instances. The transition probabilities between the initial state $s \in \{1,\dots,n\}$ and the absorbing states $\{+, -\}$ are determined by each instance of the MAB. We assign an index to each MAB instance based on its optimal arm (or action) $I_s \in \{1,\dots,m\}$ within each initial state. Then we transfer the randomness of the Bernoulli reward distribution in the MAB to transition probabilities in the MDP. Therefore, based on the instances for proof of Theorem \ref{thm:LowBounMAB}, we define the transition probabilities in Figure \ref{JDPfigure} such that in the MAB instance where the optimal arm $I_s=1$, we have $P(+|s,1)= \frac{5}{10} \gamma^{1+v},\ P(-|s,1)=1-P(+|s,1)$ and $ P(+|s,a)=\frac{3}{10}\gamma^{1+v},\  P(-|s,a)=1-P(+|s,a) \ \text{for} \  a \neq 1.$
And for the MAB instance with $I_s \neq 1$, we have similar transition probabilities as above except the $I_s$-th arm, i.e.,  
$P(+|s,I_s)=\frac{7}{10} \gamma^{1+v}, \ P(-|s,I_s)=1-P(+|s,I_s),$
where $\gamma>0$ is a parameter to be determined later. Every action which transits to state $+$ provides reward $1/\gamma$ while actions transitioning to state $-$ provide reward $0$.

Based on all the above results and instances, now we have our main theorem. 
\begin{theorem}[JDP Regret Lower Bound for Heavy-tailed MDPs]
\label{JDPlowerBoun}
For any $\epsilon$-JDP algorithm $\mathcal{M}$ there exists a heavy-tailed MDP with $S$ states and $A$ actions over $H(=1)$ time steps per episode such that for any initial state $s\in \mathcal{S}$ the expected regret of $\mathcal{M}$ after $K$ episodes satisfies 
$$Reg(T) \geq \Omega\left(\left(SA/\epsilon\right)^{\frac{v}{1+v}}K^{\frac{1}{1+v}}\right).$$
\end{theorem}

\begin{remark}
For episodic MDPs in JDP model with bounded rewards,  the regret lower bound 
 is $\Omega\left(\sqrt{HSAK}+{SAH\log K}/{\epsilon}\right)$  \cite{vietri2020private},  where the additional term of $\Omega\left({SAH\log K}/{\epsilon}\right)$ is the price of privacy compared with the non-private case. However, in the case of the heavy-tailed reward, 
  the lower bound in the above theorem shows that the price to pay for JDP is a factor of $\left(\frac{1}{\epsilon}\right)^{\frac{v}{1+v}}$ compared with the lower bound in the non-private case \cite{zhuang2021no}. Our above lower bound also matches the error due to private estimation for heavy-tailed distributions in our upper bounds (the third term of the result in Theorem \ref{thm:RegretVI} and the third term of the result in Theorem \ref{thm:regPO}) for JDP model on parameters $\epsilon, S, A, K$. It is also notable that here we only present the lower bound for JDP MDPs with heavy-tailed rewards under the case where there is just one step in each episode. 
  For the multiple-step MDPs with heavy-tailed rewards case,  the  lower bound is more challenging.
  We leave the problem of tight lower bound for  multiple-step MDPs with heavy-tailed rewards in JDP model as an open problem.
\end{remark}
\subsection{ Regret Lower Bound under LDP}
 Unlike the techniques in the  JDP case which make a reduction from MDPs to MAB, here we can directly construct an instance of MDPs to get the lower bound. 

Inspired by \cite{garcelon2021local}, we  construct the following MDP instance for a given number of states $S$ and actions $A$. Such  MDP instance can be represented by a tree whose root is the  initial state $0$ with $A$ actions that deterministically lead to the next state. Moreover, each node in the tree has $A$ children and there are exactly $S-2$ states, excluding terminal states. The leaves of the tree, denoted as $\mathcal{L}={x_1,\dots,x_L}$, represent the set of possible transitions from the intermediate states to two terminal states, labeled as $+$ and $-$. At the terminal states, the agent will receive the reward of $1/\gamma$ and $0$, respectively, where $\gamma>0$ is a parameter to be determined later. The tree without nodes $+$ and $-$ is a perfect $A$-ary tree. We show the instance with $S=15$ and $A=3$ in Figure \ref{LDPfigure} as an example. 

\begin{figure}[!htb]
    \centering
    \vspace{-0.1in}
    \includegraphics[scale=0.2]{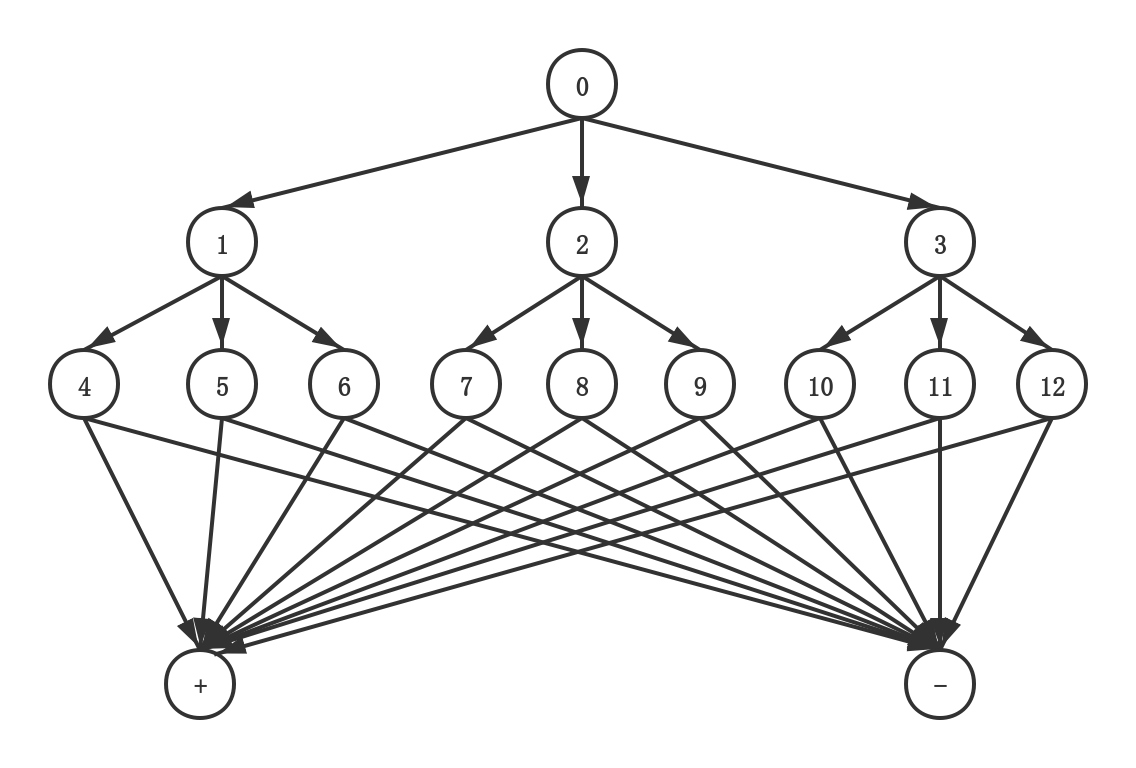}
    \caption{Hard MDP instance for LDP model}
    \label{LDPfigure}
    \vspace{-0.1in}
\end{figure}

Assume $d>0$ represents the depth of the tree where the depth of the tree with $S-2$ nodes is $d-1$ and nodes $+,-$ are located at depth $d$. Without loss of generality, we assume that all leaves, $x_1,\dots,x_L$, are positioned at depth $d-1$, implying that the number of leaves satisfies $L=A^{d-1}\ge (S-2)/2$. 
Furthermore, we make the assumption that $H=d+1$. That means once the agent arrives at $+$ or $-$, it arrives at the end of the episode. Next, we provide the distributions for these leaves. 

We assume there exists a unique action $a^*$ and leaf $x_{i^*}$ such that: 
$
   P(+|x_{i^*},a^*)=\gamma^{1+v} \ \text{and} \ P(-|x_{i^*},a^*)=1-\gamma^{1+v}. 
$
where $\gamma^{1+v}\in (0,\frac{3}{4}]$.
 Each of the other leaves has  a transition  probability
 $
     P(+|x_{i},a)=\frac{1}{2}\gamma^{1+v} \ \text{and} \ P(-|x_{i},a)=1-\frac{1}{2}\gamma^{1+v}.
 $
 We denote above instance by $\mathbb{P}_{(x_{i^*},a^*)}$.
 In order to get the regret lower bound, we also consider another instance $\mathbb{P}_0$ where for all leaf states and any action, the transition probabilities are 
$
 P(+|x_{i},a)=\frac{1}{2}\gamma^{1+v} \ \text{and} \ P(-|x_{i},a)=1-\frac{1}{2}\gamma^{1+v}.
 $ Based on the above instances, we provide a lower bound on the regret for our problem.

\begin{theorem}[LDP Regret Lower Bound]
\label{LDPlowerBoun}
For any $\epsilon$-LDP algorithm $\mathcal{M}$ where $\epsilon \in (0,1]$, there exists a heavy-tailed MDP instance with $S(\ge 3)$ states, $A(\ge2)$ actions  and one-step heavy-tailed reward per episode such that the expected regret of $\mathcal{M}$ after $K$ episodes is
$$Reg(T) \ge \Omega\left(\left(SA/\epsilon^2\right)^{\frac{v}{1+v}}K^{\frac{1}{1+v}}\right).$$
\end{theorem}

\begin{remark}
Compared with the lower bound of $\Omega\left(\frac{H\sqrt{SAK}}{\min\{e^\epsilon-1,1\}}\right)$ for the bounded case in  ~\citet{garcelon2021local}, we can see the price of privacy in the heavy-tailed case is a factor of $\left(\frac{1}{\epsilon^2}\right)^{\frac{v}{1+v}}$. When $v=1$ and $H=1$, our lower bound can recover the lower bound of the bounded case. Compared with the optimal instance-independent lower bound of $\Omega\left(\left({A}/{\epsilon^2}\right)^{\frac{v}{1+v}}K^{\frac{1}{1+v}}\right)$ for heavy-tailed MAB in LDP model given by ~\citet{tao2022optimal}, our above lower bound with $S=1$ can also recover the result.
\end{remark}



 \section{Conclusion and Future Work}
 In this work, we provided the first study on finite horizon Markov decision processes (MDPs) with heavy-tailed rewards in both joint (JDP) and local differential privacy (LDP) models. We mainly focused on the case where the reward distributions have only finite $(1+v)$-th raw moment for $v \in (0,1]$. 
We first proposed a private and robust version of both UCB-based value-iteration and policy-optimization algorithms. 
To guarantee privacy, we adopted the adaptive Tree-based mechanism for JDP and the  Laplacian mechanism for LDP. 
Based on the algorithm, we established regret bounds for both JDP and LDP cases. Finally, we established lower bounds for  finite-horizon MDPs with heavy-tailed rewards in  both JDP and LDP models. 
Through these results, we also found some differences between our problem and the problem of private MDPs with bounded rewards and some differences between our problem and the problem in the non-private case.  All of our ideas, techniques, and frameworks  can also potentially be applied to other related private reinforcement learning problems.

There are still some open problems left. First of all, in the whole paper, we assume the parameters $u$ and $v$ are known in advance, but how to deal with the case where $u$ and $v$ are unknown which is a more practical situation in the real world. Secondly, the order of $T$ is $\frac{2+v}{2(1+v)}$ in our regret upper bound under the LDP model, which is larger than the order of $\frac{1}{1+v}$ in our lower bound for the problem under LDP. Thus, can we further close the gap between these two bounds by designing other algorithms? Finally, in this paper, while we mainly focus on the finite-horizon problem, we proposed some private mean estimators and some new hard instances to prove the lower bounds. Can we extend these techniques and ideas to other related problems such as private MDPs with finite diameters with heavy-tailed rewards or some model-free reinforcement learning problems?

\section*{Acknowledgments}
Yulian Wu and Di Wang are supported in part by  BAS/1/1689-01-01, URF/1/4663-01-01, FCC/1/1976-49-01 of King Abdullah University of Science and Technology (KAUST).  Di Wang is also supported by the funding of the SDAIA-KAUST Center of Excellence in Data Science and Artificial Intelligence (SDAIA-KAUST AI). Xingyu Zhou is supported in part by NSF CNS-2153220.

  \bibliography{sample-base}
  \bibliographystyle{icml2022}


\newpage
\appendix
\onecolumn
\section*{Appendix}
\section{Related work and challenges}
Besides the work we mentioned above, there are other numerous previous studies on either private RL/bandits with bounded/sub-Gaussian rewards \cite{sajed2019optimal,liao2021locally, chen2021differential,lei2020privacy,zheng2020locally,zhou2022differentially,ren2020multi,zhou2021local,li2022differentially,li2023private} or non-private RL/bandits with heavy-tailed rewards  \cite{zhuang2021no,bubeck2013bandits,lee2020optimal,yu2018pure,lattimore2017scale,agrawal2021regret,vakili2013deterministic,agrawal2020optimal,li2023variance}. In the following, we only discuss the work that is most related to ours.

For the studies of non-private RL with heavy-tailed rewards, \citet{bubeck2013bandits}  first considers the finite-armed bandit problem in which the reward distributions have only finite $(1+v)$-th moments for some $v \in (0,1]$. It develops a robust UCB algorithm by leveraging several mean estimators for heavy-tailed distributions, such as the truncated mean estimators and the median of mean estimator. Leveraging techniques from these robust mean estimators, \citet{zhuang2021no} considers the heavy-tailed rewards in the problem of undiscounted reinforcement learning and proposes the method of Heavy-UCRL2 and Heavy-Q-learning for  model-based and model-free settings respectively. It also generalizes the algorithms to deep reinforcement learning and presents Heavy-DQN as an example. Motivated by these, we use the idea of a truncated mean estimator as the backbone in our frameworks to deal with heavy-tailed rewards in our problem. 
However, as we mentioned earlier there are several additional challenges by injecting additional noises. Moreover, in this paper, we also propose policy-optimization-based algorithms which have not been studied before for heavy-tailed rewards even in the non-private case. 

For private RL/bandits with heavy-tailed rewards, to the best of our knowledge, \citet{tao2022optimal} is the only one which investigates  MAB with heavy-tailed rewards in both central and local DP models. It proposes robust versions of successive elimination (SE) algorithms for the problem in central DP and local DP models and establishes (near) optimal rates. However, for the problem we studied, it is unsuitable for the central DP model since several recent works \citet{shariff2018differentially,dubey2021no} show that the standard DP model is irreconcilable with sub-linear regret for contextual bandits. Thus, we consider a relaxation of central DP, i.e., the joint differential privacy. Besides, we cannot use an arm elimination algorithm for our problem because of the existence of states. Since our private heavy-tailed  MDPs with $H=1$ are just the private heavy-tailed contextual bandit problem and that with $H=1,S=1$ is a private heavy-tailed MAB problem, our problem could be considered as a more general case. There are also some studies on private MDPs recently. However, all of them only consider the bounded reward case and cannot be extended to the heavy-tailed one. For example, \citet{chowdhury2021differentially} studies DP episodic MDPs with bounded rewards and proposes policy optimization and value iteration frameworks, and presents the regret upper bounds for these frameworks. Motivated by these frameworks, we propose Private-Heavy-UCBVI and Private-Heavy-UCBPO for our problem for both JDP and LDP. However, here we cannot directly use their Tree-based mechanism for JDP and the allocation methods for the privacy budget since the heavy-tailed rewards are now unbounded. Besides, \citet{garcelon2021local} establishes the lower bound of regret minimization for  MDPs with bounded rewards in LDP model by constructing some MDP examples satisfying bounded rewards. However, we cannot directly use the MDP instance to get the regret lower bound since our rewards are heavy-tailed with bounded $(1+v)$-th moment. See Table \ref{tab:my_label} for a detailed comparison. 

Robust and differentially private estimation has drawn much attention in recent years.   \citet{barber2014privacy} provided the first study on private mean estimation for distributions with the bounded moment, which is extended by \citet{kamath2020private,liu2021robust,brunel2020propose} recently. However, all of them need to assume the underlying distribution has the second-order moment, while in this paper we only need to assume the reward distributions have the $(1+v)$-th  moment for some $v\in (0, 1]$. Moreover, all of these works only focus on the central DP model and offline setting and it is unknown whether they could be extended to the stream setting. Thus,  our problem is more general. Besides the mean estimation problem, recently \citet{wang2020differentially,kamath2022improved,jin2022efficient,hu2022high} study differentially private stochastic convex optimization with heavy-tailed data. However, all of them need to assume the distribution of  the gradient has the second-order moment and cannot be used in the stream setting. \citet{wang2022differentially} studies private $\ell_1$-regression where the covariate $x$ has bounded $(1+v)$-th moment. However, its method cannot be generalized to other problems. \citet{tao2022private} considers the differentially private stochastic convex optimization with heavy-tailed data where the distribution of each coordinate of the gradient has bounded $(1+v)$-th moment. However, its method cannot be used in the online stream setting. 

\textbf{Challenges.} Compared with the previous related work, there are additional challenges for private MDPs with heavy-tailed rewards, which are mainly from the following three aspects. \textbf{First}, in the non-private case,  \citet{zhuang2021no} proposes methods that truncate each reward via a certain threshold. However, due to the additional privacy constraint in our problem, we need to be more careful in choosing the threshold, which is related to three kinds of error: an error from truncating rewards, an error due to using  a finite number of truncated rewards to estimate  heavy-tailed reward distributions and an error due to the noise for ensuring privacy.  
In detail, we need to first get a general upper bound and then based on the trade-off among these three terms of errors we can determine the optimal threshold. However, getting such a trade-off in general is still more difficult than in the non-private case, as here we will have more biased estimators due to the noise we added. For example, we need to add noise to the count of each state-action pair and this will make empirical transition probabilities be biased. \textbf{Secondly}, in  private MDPs with bounded rewards case,  \citet{chowdhury2021differentially} assumes that all rewards are bounded by $1$ and  it 
determines and allocates the noise added
for each step in an episode by using the total $\ell_1$ distance of all data streams. However, we cannot adopt the same strategy as the heavy-tailed rewards now are unbounded. 
The problem is still challenging even if we just adopt similar reward-truncating based methods in some previous related work \cite{zhuang2021no,tao2022optimal}, as the thresholds of truncated rewards are increasing, which makes these truncated rewards not uniformly bounded. 
\textbf{Thirdly}, in the stateless case, i.e.,  private heavy-tailed MAB, \citet{tao2022optimal}  achieves the optimal regret rates for both central DP and local DP models via successive elimination algorithms. However, we cannot use similar methods in the RL setting since the states and transition probabilities can also affect policy.

\section{Notations and Technical Lemmas}

\begin{table}[htb]
\linespread{2}
\centering
\caption{List of Notations}
\begin{tabular}{|c|l|}

\hline
\textbf{Notations} &  \textbf{Descriptions}   \\
\hline
 $\mathcal{S}$ & state space with cardinality $S$  \\
 $\mathcal{A}$ & action space with cardinality $A$ \\
 $\Pi$ & set of all non-stationary policies\\ 
 $K$ & number of episodes\\
 $H$ &   episode length   \\
 $P_h$ &  transition kernel at step $h$\\
 $r_h$ & mean of the reward distribution at step $h$\\
 $\pi_h$ & policy at step $h$ \\
 $s_h^k$ & state at episode $k$ and step $h$ \\
 $a_h^k$ & the action taken by agent at episode $k$ and step $h$\\
 $r_h^k$ & random reward at episode $k$ and step $h$\\
 $V_h^\pi (s)$ & value function of a state $s$ under policy $\pi$ at step $h$\\
 $V_h^*(s)$ & optimal value function\\
 $Q_h^\pi (s,a)$ & Q-function of taking action $a$ in state $s$ under policy $\pi$ at step $h$\\
 $\pi^*$ & optimal policy\\
 $\mathcal{R}_h(s,a)$ & reward distribution for state-action pair $(s,a)$ at step $h$\\
 $T$ & $T=KH$ is the total number of steps\\
 $N_h^k(s,a)$ & count of visiting state-action pair $(s,a)$ at step $h$ before episode $k$\\
 $N_h^k(s,a,s^\prime)$ & count of going to state $s^\prime$ from $s$ upon playing action $a$ at step $h$ before episode $k$\\
 $R_h^k(s,a)$ & sum of truncated rewards obtained by taking action $a$ on state $s$ at step $h$ before episode $k$\\
 $\epsilon$ & privacy budget\\
 $\tilde{N}_h^k(s,a)$ & the privatized version of $N_h^k(s,a)$\\
 $\tilde{N}_h^k(s,a,s^\prime)$ & the privatized version of  $N_h^k(s,a,s^\prime)$\\
 $\tilde{R}_h^k(s,a)$ & the privatized version of  $R_h^k(s,a)$\\
 \multirow{2}{*}{$\tilde{r}_h^k(s,a)$} & the private empirical mean estimation of truncated rewards for state-action pair $(s,a)$ \\
 & at step $h$ before episode $k$\\
 $\tilde{P}_h^k(s^\prime|s,a)$ & the private empirical transition probability\\
 $[n]$ & the set of $\{1,2,\dots,n\}$ \\
 $a\vee b$ & the maximal value between $a$ and $b$ \\
 \hline
\end{tabular}
\label{tab:notation}
\end{table}

\begin{lemma}[Parallel Composition]\label{parallel}
    Suppose there are $k$ number of $\epsilon$-differentially private mechanisms $\{\mathcal{M}_i\}_{i=1}^k$ and $k$ disjoint datasets denoted by $\{D_i\}_{i=1}^k$. Then the algorithm, which applies each $\mathcal{M}_i$ on the corresponding $D_i$, preserves $\epsilon$-DP in total.
\end{lemma}

\begin{lemma}[Composition Theorem \cite{dwork2014algorithmic}]
\label{compositionThm}
Let $\mathcal{M}_1,\mathcal{M}_2,\dots,\mathcal{M}_h$ be a sequence of randomized algorithms, where $\mathcal{M}_1: \mathcal{X}^n \rightarrow \mathcal{Y}_1$, $\mathcal{M}_2: \mathcal{Y}_1 \times \mathcal{X}^n \rightarrow \mathcal{Y}_2,\dots$, $\mathcal{M}_h: \mathcal{Y}_1 \times \mathcal{Y}_2 \times \mathcal{Y}_{h-1} \times \mathcal{X}^n \rightarrow \mathcal{Y}_h$. Suppose for every $i \in [h]$ and $y_1 \in \mathcal{Y}_1,y_2 \in \mathcal{Y}_2,\dots,y_h \in \mathcal{Y}_h$, we have $\mathcal{M}_i(y_1,\dots,y_{i-1},\cdot) : \mathcal{X}^n \rightarrow \mathcal{Y}_i$ is $\epsilon_i$-DP. Then the algorithm $\mathcal{M}: \mathcal{X}^n \rightarrow \mathcal{Y}_1 \times  \mathcal{Y}_2 \times  \dots \mathcal{Y}_h$ that runs the algorithm $\mathcal{M}_i$ sequentially is $\epsilon$-DP for $\epsilon=\sum_{i=1}^h \epsilon_i$.
\end{lemma}

\begin{lemma}[Laplace Mechanism]\label{le-lap}
		Given a dataset $D\in\mathcal{X}^n$ and a function $q : \mathcal{X}^n\rightarrow \mathbb{R}^d$, the Laplace Mechanism is defined as $q(D)+ (Y_1, Y_2, \cdots, Y_d),$
		where each $Y_i$ is i.i.d. sampled from the Laplace Distribution $\operatorname{Lap}(\frac{\Delta_1(q)}{\epsilon})$, where $\Delta_1(q)$ is the $\ell_1$-sensitivity of the function $q$, {\em i.e.,}
		$\Delta_1(q)=\sup_{D\sim D'}||q(D)-q(D')||_1.$ The density of the Laplace distribution with parameter $\lambda$ is $\operatorname{Lap}(\lambda) (x)=\frac{1}{2\lambda}\exp(-\frac{x}{\lambda})$. 
		Laplace mechanism preserves $\epsilon$-DP.
\end{lemma}

\begin{lemma}[Hoeffding's inequality]
\label{hoeffding}
Let $X_1,\dots,X_n$ be independent random variables such that ${\displaystyle a_{i}\leq X_{i}\leq b_{i}}$ almost surely. Consider the sum of these random variables, ${\displaystyle S_{n}=X_{1}+\cdots +X_{n}}$, then for all $t>0$, we have
$$
\mathbb{P}\left(\left|S_{n}-\mathbb{E}\left[S_{n}\right]\right| \geq t\right) \leq 2 \exp \left(-\frac{2 t^{2}}{\sum_{i=1}^{n}\left(b_{i}-a_{i}\right)^{2}}\right).
$$
\end{lemma}

\begin{lemma}[Bernstein's Inequality \citep{vershynin2018high}]
\label{bernstein}
Let $X_1, \cdots X_n$ be $n$ independent zero-mean random variables. Suppose $|X_i|\leq M$ and $\mathbb{E}[X_i^2]\leq s$ for all $i \in [n]$. Then for any $t>0$, we have
\begin{equation*}
    \mathbb{P}\left\{\frac{1}{n}\sum_{i=1}^n X_i \geq t \right\}\leq \exp\left(-\frac{\frac{1}{2}t^2n}{s+\frac{1}{3}Mt}\right)
\end{equation*}
\end{lemma}

\begin{lemma}[Lemma 7 in \cite{tao2022optimal}]
\label{alemma4}
Given a random variable $X$ with $\mathbb{E}[|X|^{1+v}]\leq u$ for some $v\in (0, 1]$, for any $B>0$ we have
\begin{equation*}
    \mathbb{E}\left[X \cdot \mathbb{I}_{|X|>B}\right]\leq \frac{u}{B^v}.
\end{equation*}
\end{lemma}

\begin{lemma}[Hölder's inequality]
\label{holder}
For $p,q \in (1,\infty)$ with $1/p+1/q=1$,
$$
\sum_{k=1}^{n}\left|x_{k} y_{k}\right| \leq\left(\sum_{k=1}^{n}\left|x_{k}\right|^{p}\right)^{\frac{1}{p}}\left(\sum_{k=1}^{n}\left|y_{k}\right|^{q}\right)^{\frac{1}{q}}.
$$
\end{lemma}

\begin{lemma}[Concentration of Laplace Variables~\citep{wang2018empirical}]\label{alemma5}
If $X_1, \cdots X_n \sim \operatorname{Lap}(s/\epsilon)$, then with probability at least $1-\beta$, we have
\[
	\left|\frac{1}{n}\sum_{i=1}^{n}X_i\right|\leq \frac{2s}{\epsilon\sqrt{n}}\sqrt{\log \frac{2}{\beta}}.
\]
\end{lemma}

\begin{lemma}[Markov Inequality]
\label{MarkovIne}
If $X$ is a non-negative random variable and $a>0$, then the probability that $X$ is at least $a$ is at most the expectation of $X$ divided by $a$:
$$P(X \ge a) \le \frac{\mathbb{E}(X)}{a}$$
\end{lemma}

\begin{lemma}[Lemma 1 in \cite{garivier2019explore}]
\label{Lemm:KL}
Consider a measurable space $(\Omega,\mathcal{F})$ equipped with two distributions $\mathbb{P}_1$ and $\mathbb{P}_2$. For any $\mathcal{F}$-measurable function $Z: \Omega \rightarrow[0,1]$, we have 
$$
\mathrm{KL}\left(\mathbb{P}_1, \mathbb{P}_2\right) \geq \operatorname{kl}\left(\mathbb{E}_1[Z], \mathbb{E}_2[Z]\right)
$$
where $\mathbb{E}_1$ and $\mathbb{E}_2$ are the expectations under $\mathbb{P}_1$ and $\mathbb{P}_2$ respectively.
\end{lemma}

\section{Algorithms and Proofs of Section \ref{Sec:VI}}
\subsection{Details of the Tree-based Mechanism}
Note that our CENTRAL-PRIVATIZER is established by an adaptive version of the Tree-based mechanism, we first give details of this mechanism. 
\begin{algorithm}[htb]
    \caption{(Adaptive) Tree-based Mechanism}
    \label{Algo-Tree}
    \begin{algorithmic}[1]
        \REQUIRE time horizon $K$, privacy budget $\epsilon$, a stream $\sigma$.
        \ENSURE A private version $\widehat{S}(k)$ for $S(k)=\sum\nolimits_{i=1}^{k}{\sigma(i)}$ at each $k\in[K]$
        \STATE Initialize each p-sum $\alpha_i$ and noisy p-sum $\widehat{\alpha}_i$ to $0$.
        \STATE $\epsilon^\prime\gets\epsilon/\log K$.
        \FOR {$k = 1,\ldots,K$}
            \STATE Express $k$ in binary form: $k=\sum_j {\rm Bin}_j(k)\cdot 2^j$.
            \STATE $i\gets\min\{j:{\rm Bin}_j(k)\neq 0\}$.
            \STATE $\alpha_i\gets\sum_{j<i}{\alpha_j}+\sigma(k)$.
            \FOR{$j= 0,\ldots,i-1$}
                \STATE $\alpha_j\gets0$, $\widehat{\alpha}_j\gets0$.
           \ENDFOR
           \STATE $\widehat{\alpha}_i\gets\alpha_i+{\rm Lap}(2B_k/\epsilon^\prime)$.
           \STATE \textbf{Return} $\hat{S}(k)\gets\sum_{j:{\rm Bin}_j(k)=1}{\widehat{\alpha}_j}$.
        \ENDFOR
    \end{algorithmic}
\end{algorithm} 

\begin{definition}[p-sum]\label{def-psum}
        A p-sum is a partial sum of consecutive data items. Let $1\le i\le j$. For a data stream $\sigma$ of length $K$, we use $\sigma(k)$ to denote the data item at time $k\in[K]$ and $\sum[i,j]\triangleq\sum_{k=i}^j\sigma(k)$ to denote a partial sum involving data items $i$ through $j$. We use the notation $\alpha_i^k$ to denote the p-sum $\sum[k-2^i+1,k]$.
\end{definition}

\begin{lemma}[{(Adaptive) Tree-based Mechanism} \cite{tao2022optimal}]
\label{lemma-tree}
Given a stream $\sigma$ such that $\sigma(k)\in[-B_k, B_k]$ for $\forall k\in[K]$, where $B_k$ is non-decreasing with $k$, we want to privately and continually release the sum of the stream $S(k)\triangleq\sum\nolimits_{i=1}^{k}{\sigma(i)}$ for each $k\in [K]$. The (adaptive) tree-based Mechanism (Algorithm \ref{Algo-Tree}) outputs an estimation $\widehat{S}(k)$ for $S(k)$ at each $k\in[K]$ such that $\widehat{S}(k)$ preserves $\epsilon$-differential privacy and guarantees the following noise bound with probability at least $1-\delta$ for any $\delta>0$,
\begin{equation}
    \left| \widehat{S}(k) - S(k) \right| \leq  \frac{2B_k}{\epsilon} \log^{1.5} K \cdot \ln\frac{1}{\delta}.
\end{equation}
\end{lemma}
It is notable that when each $B_k=m$ for some $m$ then Algorithm \ref{Algo-Tree} is just the standard $m$-bounded tree-based mechanism in \cite{chan2011private}. Thus, Algorithm \ref{Algo-Tree} is a generalization compared with the standard one. To ensure $\epsilon'$-DP for some given $\epsilon'$, for counter $R_h^k(s,a)$, we use adaptive tree-based mechanism and add $\text{Lap}(\frac{2B_k}{\epsilon^\prime})$ for some non-decreasing $B_k$ to every p-sum before releasing them. Then we get private count $\tilde{R}_h^k(s,a)$. For $N_h^k(s,a)$
and $N_h^k(s,a,s^\prime)$, we use $1$-bounded binary-tree mechanisms with Laplace noise $\text{Lap}(\frac{1}{\epsilon^\prime})$ to release the respective private counts $\tilde{N}_h^k(s,a)$
and $\tilde{N}_h^k(s,a,s^\prime)$. 

\subsection{Proof of Lemma \ref{centralError}}

\begin{proof}[\bf Proof of Lemma \ref{centralError}]
Therefore, we first focus on $E_{\epsilon, \delta, 2}$ which is the error bound between the private count for the sum of rewards $\tilde{R}_h^k(s,a)$ and the non-private count $R_h^k(s,a)$.

We start with the privacy guarantee of the CENTRAL-PRIVATIZER. First, note that there are $SAH$ many counters for the sum of rewards $R_h^k(s,a)$, and each counter is a $K$-bounded adaptive tree-based mechanism. For a fixed tuple $(s,a,h)\in \mathcal{S}\times\mathcal{A}\times[H]$, the private count $\tilde{R}_h^k(s,a)$ is the sum of at most $\log K$ noisy p-sums, where each p-sum is corrupted by an independent Laplace noise $\text{Lap}\left(\frac{6B_k H\log K}{\epsilon}\right)$. By Lemma \ref{lemma-tree}, the private counts $\{\tilde{R}_h^k(s,a)\}_{k \in [K]}$ satisfy $\frac{\epsilon}{3H}$-DP.
We leverage the fact that the total change across all counters a user can have scales with the length of the episode $H$ to get the composition of the $SAH$ $\frac{\epsilon}{3H}$-DP counters for $\tilde{R}(\cdot)$ satisfy $\frac{\epsilon}{3}$-DP.

Using similar arguments of Lemma 5.1 in \cite{chowdhury2021differentially}, one can show that composition of the counters for $\tilde{N}_h^k(s,a)$ and $\tilde{N}_h^k(s,a,s^\prime)$ satisfy $\frac{\epsilon}{3}$-DP respectively. Finally, employing the composition property of DP in \cite{dwork2014algorithmic}, we obtain that the CENTRAL-PRIVATIZER is $\epsilon$-DP.

Now we focus on the utility of CENTRAL-PRIVATIZER. First, for a fixed tuple $(s,a,h) \in \mathcal{S}\times \mathcal{A}\times [H]$, we consider the private counts $\tilde{R}_h^k(s,a)$ corresponding to the sum of rewards ${R}_h^k(s,a)$. Note that, at each episode $k \in [K]$, the noise bound $|\tilde{R}_h^k(s,a)-{R}_h^k(s,a)|$ is the sum of at most $\log K$ random variables drawn from the Laplace distribution. From Lemma \ref{lemma-tree}, we have 
$$\mathbb{P}\left[\left|\tilde{R}_{h}^{k}(s, a)-R_{h}^{k}(s, a)\right| \leq \frac{6B_k H}{\epsilon}  \log ^{1.5} K \ln (3SAT / \delta)\right] \geq 1-\delta / (3SAT)$$
By a union bound argument, we can obtain 
$$\mathbb{P}\left[\forall (s,a,k,h),\left|\tilde{R}_{h}^{k}(s, a)-R_{h}^{k}(s, a)\right| \leq \frac{6B_k H}{\epsilon}  \log ^{1.5} K \ln (3SAT / \delta)\right] \geq 1-\delta / 3$$

By using the same arguments, we can get the values of error bounds $E_{\epsilon, \delta, 1}$ and $E_{\epsilon, \delta, 3}$.
$$\mathbb{P}\left[\forall (s,a,k,h),\left|\tilde{N}_{h}^{k}(s, a)-N_{h}^{k}(s, a)\right| \leq \frac{3 H}{\epsilon}  \log ^{1.5} K \ln (3SAT / \delta)\right] \geq 1-\delta / 3$$
$$\mathbb{P}\left[\forall (s,a,k,h),\left|\tilde{N}_{h}^{k}(s, a,s^\prime)-N_{h}^{k}(s, a,s^\prime)\right| \leq \frac{3 H}{\epsilon}  \log ^{1.5} K \ln (3S^2AT / \delta)\right] \geq 1-\delta / 3$$
Combining all three guarantees together using a union bound, we obtain that CENTRAL-PRIVATIZER satisfies Assumption \ref{Assum1}.
\end{proof}

\subsection{Proof of Theorem \ref{thm:JDPpriv}}
\begin{proof}[\bf Proof of Theorem \ref{thm:JDPpriv}]
To prove the JDP guarantee, we use the billboard lemma in \citep[Lemma 9]{hsu2016private} which states that an algorithm is JDP if the output sent to each user is a function of the user's private data and a common quantified computed with standard differential privacy. The formal lemma is stated as follows:
\begin{lemma}[Billboard lemma \cite{hsu2016private}]
Suppose $\mathcal{M}: U \to \mathcal{R}$ is $\epsilon$-differentially private. Consider any set of functions $f_i: U_i \times \mathcal{R} \to \mathcal{R}^\prime$ where $U_i$ is the portion of the database containing the $i$'s user data. then composition $\{f_i(\Pi_i U,\mathcal{M}(U))\}$ is $\epsilon$-joint differentially private, where $\Pi_i: U\to U_i$ is the projection to $i$'s data.
\end{lemma}
Note that by the privacy guarantee in Lemma \ref{centralError} and the post-processing property of DP in \cite{dwork2014algorithmic}, the policies $(\pi^k)_k$ are computed with a $\epsilon$-DP. Therefore, by the above billboard lemma, the composition of the output of all $K$ episodes satisfies $\epsilon$-JDP.

\end{proof}

\subsection{ Proof of Theorem \ref{thm:RegretVI}}

Before showing the proof of Theorem \ref{thm:RegretVI} we first consider the following two lemmas. 

The following lemma shows the estimation errors  of our  private mean empirical rewards and private empirical transition probabilities, which are in step 7 of the algorithm.  
\begin{lemma}[Concentration bounds of private estimators]
\label{ConcenPrivateVI}
Fix any $\epsilon \in (0,1]$ and $\delta \in (0,1)$ and take $B_n=\left(\frac{\epsilon u n}{H \log^{1.5}K \log (3SAT/\delta)}\right)^{\frac{1}{1+v}}$ in equation \eqref{trunctedR}. Then, under Assumption \ref{Assum1}, with probability at least $1- 3\delta$, uniformly over all $(s,a,h,k)$ we have 
{\footnotesize $$ |\widetilde{r}_{h}^{k}(s, a)-{r}_{h}(s, a)|\le \beta_h^{k,r}(s,a),\left|\left(\widetilde{P}_{h}^{k}-P_{h}\right) V_{h+1}^{*}(s, a)\right| \leq \beta_{h}^{k, p v}(s, a),$$ $\left|P_{h}\left(s^{\prime}|s, a\right)-\widetilde{P}_{h}^{k}\left(s^{\prime}|s, a\right)\right|\leq   C\sqrt{\frac{P_h(s^\prime |s,a)\ln\frac{2SAT}{\delta}}{(\tilde{N}_h^k(s,a)+E_{\epsilon,\delta,1})\vee 1}}
+\frac{C\ln\frac{2SAT}{\delta}+2E_{\epsilon,\delta,1}+E_{\epsilon,\delta,3}}{(\tilde{N}_h^k(s,a)+E_{\epsilon,\delta,1})\vee 1}, 
$}
where $C$ is a constant, $a \vee b=\max\{a,b\}$, $(PV_{h+1})(s,a)=\sum_{s^\prime}P(s^\prime|s,a)V_{h+1}(s^\prime)$, and 
$\beta_h^{k,r}(s,a)=\frac{2\tau E_{\varepsilon, \delta, 1}}{\left(\tilde{N}_{h}^{k}(s, a)+E_{\varepsilon, \delta, 1}\right) \vee 1}
    +10 u^{\frac{1}{1+v}}\left(\frac{H \log^{1.5}K \log (3SAT/\delta)}{\epsilon\left((\tilde{N}_h^{k}(s,a)+E_{\varepsilon, \delta, 1})\vee 1\right)}\right)^{\frac{v}{1+v}}$, $
    \beta_{h}^{k, p v}(s, a)=\tau H \sqrt{\frac{2 \ln (4SAT/\delta)}{(\tilde{N}_h^k(s,a)+E_{\epsilon,\delta,1}) \vee 1}}
    +\frac{\tau H(2E_{\epsilon,\delta,1}+SE_{\epsilon,\delta,3})}{(\tilde{N}_h^k(s,a)+E_{\epsilon,\delta,1}) \vee 1}.
$
\end{lemma}

\begin{remark}
Actually, all the above errors are determined by three terms: an error from truncating rewards, an error due to using  the finite number of truncated rewards to estimate  heavy-tailed reward distributions and an error due to the noise for ensuring privacy.  In order to balance these three terms, we  need to take the truncation threshold $B_n=\left(\frac{\epsilon u n}{H \log^{1.5}K \log (3SAT/\delta)}\right)^{\frac{1}{1+v}}$. Compared with the truncation value $B_n=\left(\frac{\epsilon u n }{\log^{1.5}T}\right)^{\frac{1}{1+v}}$ in \cite{tao2022optimal} where the authors focus on DP-MAB with heavy-tailed rewards, the difference comes from that the privacy budget in our problem is $\frac{\epsilon}{H}$ for each step $h$ based on Composition Theorem in Lemma \ref{compositionThm}, while $H=1$ in MAB case. Compared with the truncation threshold $B_n=\left(\frac{un}{\log(2SAT/\delta)}\right)^{\frac{1}{1+v}}$ in non-private RL with heavy-tailed rewards \cite{zhuang2021no}, we can see the efficient number of samples in our problem becomes to $n\epsilon$ due to privacy. From our truncation threshold value $B_n=\left(\frac{\epsilon u n}{H \log^{1.5}K \log (3SAT/\delta)}\right)^{\frac{1}{1+v}}$, we also can check the trade-off between utility and privacy: larger $\epsilon$ value provides weaker privacy guarantee but we will truncate the data with a larger range, so more data information will be used, then the bias will be smaller so that we get better utility .
\end{remark}

\begin{proof}[\bf Proof of Lemma \ref{ConcenPrivateVI}]
We first define the non-private mean empirical rewards: $\bar{r}_{h}^{k}(s, a):=\frac{R_{h}^{k}(s, a)}{N_{h}^{k}(s, a) \vee 1}$. We then get the non-private estimation error. For a fixed tuple $(s,a,h) \in \mathcal{S}\times \mathcal{A}\times[H]$, at every episode $k \in [K]$, from Bernstein's inequality in Lemma \ref{bernstein} for bounded random variables, with probability at least $1-\frac{\delta}{2SAT}$, noting that 
\begin{equation}
  \mathbb{E}(X^2 \cdot \mathbb{I}_{|X|\le B})= \mathbb{E}(X^{1+v}X^{1-v} \cdot \mathbb{I}_{|X|\le B}) \le B^{1-v}\mathbb{E}(X^{1+v} \cdot \mathbb{I}_{|X|\le B}) \le  uB^{1-v}  
\end{equation}
and 
\begin{equation}
    \mathbb{E}[X\mathbb{I}_{|X|>B}] \le \mathbb{E}[|X|^{1+v}|X|^{-v}\mathbb{I}_{|X|>B}] \le \frac{u}{B^v}
\end{equation}
if $\mathbb{E}|X|^{1+v} \leq u$,
\begin{equation}
\label{eq1}
\begin{aligned}
&\left|\bar{r}_{h}^{k}(s, a)-r_{h}(s, a)\right|\\
=& \left|\frac{R_{h}^{k}(s, a)}{N_{h}^{k}(s, a)\vee 1 }-r_{h}(s, a)\right|\\
=&|\frac{1}{N_{h}^{k}(s, a)\vee 1}\sum_{i=1}^{N_h^{k}(s,a)}[r_{h,i}(s,a)\mathbbm{I}\{|r_{h,i}(s,a)| \le B_{i}\}-\mathbb{E}(X_h(s,a)\mathbbm{I}\{|X_h(s,a)| \le  B_{i}\})\\
&+\mathbb{E}(X_h(s,a)\mathbbm{I}\{|X_h(s,a)| \le B_{i}\})-\mathbb{E}(X_h(s,a))]|\\
\le & \left|\frac{1}{N_{h}^{k}(s, a)\vee 1}\sum_{i=1}^{N_h^{k}(s,a)}[r_{h,i}(s,a)\mathbbm{I}\{|r_{h,i}(s,a)| \le B_{i}\}-\mathbb{E}(X_h(s,a)\mathbbm{I}\{|X_h(s,a)| \le  B_{i}\})\right|\\
&+\left|\frac{1}{N_{h}^{k}(s, a)\vee 1}\sum_{i=1}^{N_h^{k}(s,a)}\mathbb{E}(X_h(s,a)\mathbbm{I}\{|X_h(s,a)| \ge  B_{i}\})\right|\\
\le &\sqrt{\frac{2uB_{N_h^{k}(s,a)}^{1-v}\log(\frac{2SAT}{\delta})}{N_h^{k}(s,a)\vee 1}}+\frac{B_{N_h^{k}(s,a)}\log\frac{2SAT}{\delta}}{3(N_h^{k}(s,a)\vee 1)}+  \frac{1}{N_h^{k}(s,a)\vee 1}\sum_{i=1}^{N_h^{k}(s,a)}\frac{u}{B_{i}^v}
\end{aligned}
\end{equation}
Let $L_{r,k}=\sqrt{\frac{2uB_{N_h^{k}(s,a)}^{1-v}\log(\frac{2SAT}{\delta})}{N_h^{k}(s,a)\vee 1}}+\frac{B_{N_h^{k}(s,a)}\log\frac{2SAT}{\delta}}{3(N_h^{k}(s,a)\vee 1)}+  \frac{1}{N_h^{k}(s,a)\vee 1}\sum_{i=1}^{N_h^{k}(s,a)}\frac{u}{B_{i}^v}$, then we denote 
\begin{equation}
F^{c}=\left\{\exists s, a, h, k:\left|\bar{r}_{h}^{k}(s, a)-r_{h}(s, a)\right| \geq L_{r,k} \right\}
\end{equation}

Then, by using union bound over all $s,a,h,k$, we have 
$\mathbb{P}(F^c) \le  \frac{\delta}{2}$. So, 
\begin{equation}
\begin{aligned}
\mathbb{P}(F)&=\mathbb{P}\left(\left\{\forall s, a, h, k:\left|\bar{r}_{h}^{k}(s, a)-r_{h}(s, a)\right| \leq L_{r,k}\right\}\right)\\
&= 1-\mathbb{P}(F^c)\\
&=1-\frac{\delta}{2}.
\end{aligned}
\end{equation}
We first study the concentration of the private reward estimate. Note that under the event in Assumption \ref{Assum1},
\begin{equation}
    \left|\frac{\widetilde{R}_{h}^{k}(s, a)}{\left(\widetilde{N}_{h}^{k}(s, a)+E_{\varepsilon, \delta, 1}\right) \vee 1}-\frac{R_{h}^{k}(s, a)}{\left(\widetilde{N}_{h}^{k}(s, a)+E_{\varepsilon, \delta, 1}\right) \vee 1}\right| \leq \frac{E_{\varepsilon, \delta,k, 2}}{\left(\widetilde{N}_{h}^{k}(s, a)+E_{\varepsilon, \delta, 1}\right) \vee 1},
\end{equation}

since $\widetilde{N}_{h}^{k}(s, a)+E_{\varepsilon, \delta, 1} \ge {N}_{h}^{k}(s, a) \ge 0$ and $|\widetilde{R}_{h}^{k}(s, a)-{R}_{h}^{k}(s, a)| \le E_{\varepsilon, \delta,k, 2}$. 
\begin{equation}
\begin{aligned}
&\left|\frac{R_{h}^{k}(s, a)}{\left(\tilde{N}_{h}^{k}(s, a)+E_{\varepsilon, \delta, 1}\right) \vee 1}-r_{h}(s, a)\right| \\
\leq &\left|r_{h}(s, a)\left(\frac{N_{h}^{k}(s, a) \vee 1}{\left(\widetilde{N}_{h}^{k}(s, a)+E_{\varepsilon, \delta, 1}\right) \vee 1}-1\right)\right|+\left|\frac{N_{h}^{k}(s, a) \vee 1}{\left(\widetilde{N}_{h}^{k}(s, a)+E_{\varepsilon, \delta, 1}\right) \vee 1}\left(\frac{R_{h}^{k}(s, a)}{N_{h}^{k}(s, a) \vee 1}-r_{h}(s, a)\right)\right| \\
 \leq &|r_{h}(s, a)|\left|1-\frac{N_{h}^{k}(s, a) \vee 1}{\left(\tilde{N}_{h}^{k}(s, a)+E_{\varepsilon, \delta, 1}\right) \vee 1}\right|+\frac{N_{h}^{k}(s, a) \vee 1}{\left(\tilde{N}_{h}^{k}(s, a)+E_{\varepsilon, \delta, 1}\right) \vee 1} {L_{r,k}}\\
\leq &\frac{2\tau E_{\varepsilon, \delta, 1}}{\left(\tilde{N}_{h}^{k}(s, a)+E_{\varepsilon, \delta, 1}\right) \vee 1}+\frac{L_{r,k} (N_{h}^{k}(s, a) \vee 1)}{\left(\widetilde{N}_{h}^{k}(s, a)+E_{\varepsilon, \delta, 1}\right) \vee 1}
\end{aligned}
\end{equation}
We put two pieces together, then get
\begin{equation}
\label{eq:conc}
    |\widetilde{r}_{h}^{k}(s, a)-{r}_{h}(s, a)|\le \frac{2\tau E_{\varepsilon, \delta, 1}}{\left(\tilde{N}_{h}^{k}(s, a)+E_{\varepsilon, \delta, 1}\right) \vee 1}+\frac{L_{r,k} (N_{h}^{k}(s, a) \vee 1)}{\left(\widetilde{N}_{h}^{k}(s, a)+E_{\varepsilon, \delta, 1}\right) \vee 1}+\frac{E_{\varepsilon, \delta,k, 2}}{\left(\widetilde{N}_{h}^{k}(s, a)+E_{\varepsilon, \delta, 1}\right) \vee 1}
\end{equation}
where $E_{\epsilon, \delta,k, 2}= \frac{6HB_{N_{h}^{k}(s, a)} }{\epsilon}  \log ^{1.5} K \ln (3SAT / \delta)$.
In order to get the trade-off between mean estimation error and private count error, we set the truncation threshold $B_n=\left(\frac{\epsilon u n}{H \log^{1.5}K \log (3SAT/\delta)}\right)^{\frac{1}{1+v}}$. Then we can bound each term in $L_{r,k}$ for $\epsilon \in (0,1] $ and $K \ge 2$:
$$\begin{aligned}
\sqrt{\frac{2uB_{N_h^{k}(s,a)}^{1-v}\log(\frac{2SAT}{\delta})}{N_h^{k}(s,a)\vee 1}} &\le \frac{\sqrt{2}u^{\frac{1}{1+v}}\epsilon^{\frac{1-v}{2(1+v)}}(\log(3SAT/\delta))^{\frac{v}{1+v}}}{(N_h^{k}(s,a)\vee 1)^{\frac{v}{1+v}}{(H\log^{1.5}K)}^{\frac{1-v}{2(1+v)}}} \\
&\le \sqrt{2}u^{\frac{1}{1+v}}\left(\frac{H\log^{1.5}K \log(3SAT/\delta)}{\epsilon(N_h^{k}(s,a)\vee 1)}\right)^{\frac{v}{1+v}}
\end{aligned},$$

$$\frac{B_{N_h^{k}(s,a)}\log\frac{2SAT}{\delta}}{3(N_h^{k}(s,a)\vee 1)} \le \frac{u^{\frac{1}{1+v}}(\log (3SAT/\delta))^{\frac{v}{1+v}}}{3(N_h^{k}(s,a)\vee 1)^{\frac{v}{1+v}}}\left(\frac{\epsilon}{H \log^{1.5}K}\right)^{\frac{1}{1+v}} \le \frac{u^{\frac{1}{1+v}}}{3}\left(\frac{H\log^{1.5}K \log(3SAT/\delta)}{\epsilon(N_h^{k}(s,a)\vee 1)}\right)^{\frac{v}{1+v}},$$

where the last inequality is based on the fact that 
$x^{\frac{1}{1+v}}\le x^{-\frac{v}{1+v}}$for $x \in (0,1]$,

$$\begin{aligned}
\frac{1}{N_h^{k}(s,a)\vee 1}\sum_{i=1}^{N_h^{k}(s,a)}\frac{u}{B_{i}^v} &\le \frac{u}{N_h^{k}(s,a)\vee 1}\sum_{i=1}^{N_h^{k}(s,a)}\left(\frac{H \log^{1.5}K \log (3SAT/\delta)}{\epsilon u i}\right)^{\frac{v}{1+v}}\\
& \le \frac{u^{\frac{1}{1+v}}}{N_h^{k}(s,a)\vee 1}\left(\frac{H \log^{1.5}K \log (3SAT/\delta)}{\epsilon}\right)^{\frac{v}{1+v}}\sum_{i=1}^{N_h^{k}(s,a)}i^{-\frac{v}{1+v}}\\
& \le \frac{u^{\frac{1}{1+v}}}{N_h^{k}(s,a)\vee 1}\left(\frac{H \log^{1.5}K \log (3SAT/\delta)}{\epsilon}\right)^{\frac{v}{1+v}}\cdot (1+v)\cdot (N_h^{k}(s,a))^{\frac{1}{1+v}}\\
& \le 2u^{\frac{1}{1+v}}\left(\frac{H \log^{1.5}K \log (3SAT/\delta)}{\epsilon(N_h^{k}(s,a)\vee 1)}\right)^{\frac{v}{1+v}}
\end{aligned}.$$

Thus, $L_{r,k} \le 4u^{\frac{1}{1+v}}\left(\frac{H \log^{1.5}K \log (3SAT/\delta)}{\epsilon(N_h^{k}(s,a)\vee 1)}\right)^{\frac{v}{1+v}}.$

And $$E_{\epsilon, \delta,k, 2}= \frac{6HB_{N_{h}^{k}(s, a)} }{\epsilon}  \log ^{1.5} K \ln (3SAT / \delta)=6(uN_{h}^{k}(s, a))^{\frac{1}{1+v}}\left(\frac{H \log^{1.5}K \log (3SAT/\delta)}{\epsilon}\right)^{\frac{v}{1+v}}.$$
Since $\widetilde{N}_{h}^{k}(s, a)+E_{\varepsilon, \delta, 1} \ge {N}_{h}^{k}(s, a) \ge 0$ , we obtain 
\begin{equation}
    |\widetilde{r}_{h}^{k}(s, a)-{r}_{h}(s, a)|\le \frac{2\tau E_{\varepsilon, \delta, 1}}{\left(\tilde{N}_{h}^{k}(s, a)+E_{\varepsilon, \delta, 1}\right) \vee 1}+10 u^{\frac{1}{1+v}}\left(\frac{H \log^{1.5}K \log (3SAT/\delta)}{\epsilon\left((\tilde{N}_h^{k}(s,a)+E_{\varepsilon, \delta, 1})\vee 1\right)}\right)^{\frac{v}{1+v}}.
\end{equation}

To show the second result, we first note that $V^*$ is fixed and $V_h^*(s) \le \tau H$ for all $h$ and $s$. Based on Hoeffding's inequality in Lemma \ref{hoeffding}, we can use similar proof as Lemma 4.1 in \cite{chowdhury2021differentially} to get with probability at least $1-\delta/2$
$$\left|\left(\widetilde{P}_{h}^{k}-P_{h}\right) V_{h+1}^{*}(s, a)\right| \leq \tau H \sqrt{\frac{2 \ln (4SAT/\delta)}{(\tilde{N}_h^k(s,a)+E_{\epsilon,\delta,1}) \vee 1}}+\frac{\tau H(2E_{\epsilon,\delta,1}+SE_{\epsilon,\delta,3})}{(\tilde{N}_h^k(s,a)+E_{\epsilon,\delta,1}) \vee 1}$$
and with probability at least $1-\delta$,
$$\left|P_{h}\left(s^{\prime} \mid s, a\right)-\widetilde{P}_{h}^{k}\left(s^{\prime} \mid s, a\right)\right| \leq C\sqrt{\frac{P_h(s^\prime |s,a)\ln(2SAT/\delta)}{(\tilde{N}_h^k(s,a)+E_{\epsilon,\delta,1})\vee 1}}+\frac{C\ln(2SAT/\delta)+2E_{\epsilon,\delta,1}+E_{\epsilon,\delta,3}}{(\tilde{N}_h^k(s,a)+E_{\epsilon,\delta,1})\vee 1},$$
which is based on the application of empirical Bernstein inequality.
\end{proof}

The next lemma claims that the value function maintained in Algorithm \ref{UCBVI} is optimistic.
\begin{lemma}
\label{lem:opt}
Fix some $\delta \in (0,1]$, with probability at least $1-3\delta$, $\widetilde{V}_h^k(s) \ge  V_h^*(s)$ for all $(k,h,s)$.
\end{lemma}

\begin{proof}[\bf Proof of Lemma \ref{lem:opt}]
   For a fixed $k$, consider $h=H+1,H,\ldots, 1$. In the base case $h=H+1$, it trivially holds since $\widetilde{V}_{H+1}^k(s) = 0 = V_{H+1}^*(s)$. Assume that  $\widetilde{V}_{h+1}^k(s) \ge V_{h+1}^*(s)$ for all $s$. Then, by the update rule, we have 
   \begin{align*}
       \widetilde{{Q}}_h^k(s,a) = \max\{-(H-h+1)\tau, \min\{(H-h+1)\tau,\widetilde{r}_h^k(s,a) + (\widetilde{P}_h^k\widetilde{V}_{h+1}^k)(s,a) + \beta_h^k(s,a)\} \}
   \end{align*}
   
   First, we would like to show that the truncation at $-(H-h+1)\tau$ does not affect the analysis. To see this, first, observe that under Lemma~\ref{ConcenPrivateVI}
   \begin{equation}
   \label{eq:relation}
       \begin{aligned}
       \widetilde{r}_h^k(s,a) + (\widetilde{P}_h^k\widetilde{V}_{h+1}^k)(s,a)+ \beta_h^k(s,a) &\overset{(a)}{\ge} r_h(s,a) + (\widetilde{P}_h^k\widetilde{V}_{h+1}^k)(s,a) + \beta_{h}^{k,pv}(s,a)\\
       &\overset{(b)}{\ge} r_h(s,a) + (\widetilde{P}_h^k{V}_{h+1}^*)(s,a) + \beta_{h}^{k,pv}(s,a)\\
       &\overset{(c)}{\ge} r_h(s,a) + (P_hV_{h+1}^*)(s,a) = Q_h^*(s,a) \\
       &\ge -(H-h+1)\tau 
   \end{aligned}
   \end{equation}
   
   where (a) holds by the first result in Lemma~\ref{ConcenPrivateVI}; (b) holds by induction; (c) holds by the second result in Lemma~\ref{ConcenPrivateVI}. This directly implies that 
   \begin{align*}
       \widetilde{{Q}}_h^k(s,a) = \min\{(H-h+1)\tau ,\widetilde{r}_h^k(s,a) + (\widetilde{P}_h^k\widetilde{V}_{h+1}^k)(s,a) + \beta_h^k(s,a)\}
   \end{align*}
    Hence, if the maximum is attained at $(H-h+1)\tau$, then $\widetilde{{Q}}_h^k(s,a) \ge Q_h^*(s,a)$ trivially holds since $Q_h^*(s,a) \in [-(H-h+1)\tau ,(H-h+1)\tau ]$. Otherwise, by Eq.~\eqref{eq:relation}, we also have $\widetilde{{Q}}_h^k(s,a) \ge Q_h^*(s,a)$. Therefore, we have $\widetilde{Q}_h^k(s,a) \ge Q_h^*(s,a)$, and hence $\widetilde{V}_h^k(s) \ge V_h^*(s)$.
\end{proof}

\begin{proof}[\bf Proof of Theorem \ref{thm:RegretVI}]
By the optimistic result in Lemma~\ref{lem:opt}, we have 
\begin{align}
\label{eq:vi_regret_start}
    Reg(T) = \sum_{k=1}^K (V_1^{*}(s_1)-V_1^{\pi_k}(s_1)) \le \sum_{k=1}^K(\widetilde{V}_1^{k}(s_1)-{V}_1^{\pi_k}(s_1))
\end{align}
Now, we turn to upper bound $\widetilde{V}_h^k(s_h^k)-V_h^{\pi_k}(s_h^k)$ by a recursive form. 

First, observe that 
\begin{align*}
    (\widetilde{V}_h^k-V_h^{\pi_k} )(s_h^k)  = (\widetilde{Q}_h^k-Q_h^{\pi_k} )(s_h^k,a_h^k),
\end{align*}
which holds since the action executed by $\pi_k$ at step $h$, and the action used to update $\widetilde{V}_h^k$ is the same. Now, to bound the $Q$-value difference, we have 
\begin{equation}
\begin{aligned}
\label{eq:q_decom}
    &(\widetilde{Q}_h^k-Q_h^{\pi_k} )(s_h^k,a_h^k)\\
    \overset{(a)}{\le}& 2\beta_h^{k,r}(s_h^k,a_h^k) + (\widetilde{P}_h^k \widetilde{V}_{h+1}^k-P_h V_{h+1}^{\pi_k} )(s_h^k,a_h^k) + \beta_h^{k,pv}(s,a)\\
    = &\left[(\widetilde{P}_h^k -P_h) \widetilde{V}_{h+1}^k\right](s_h^k,a_h^k) + \left[ P_h(\widetilde{V}_{h+1}^k-V_{h+1}^{\pi_k} )\right](s_h^k,a_h^k) + 2\beta_h^{k,r}(s_h^k,a_h^k) + \beta_h^{k,pv}(s,a)\\
    = & \left[( \widetilde{P}_h^k-P_h ) {V}_{h+1}^*\right](s_h^k,a_h^k) + \left[(P_h-\widetilde{P}_h^k ) ( {V}_{h+1}^*-\widetilde{V}_{h+1}^k)\right](s_h^k,a_h^k) 
    + \left[ P_h(\widetilde{V}_{h+1}^k -V_{h+1}^{\pi_k})\right](s_h^k,a_h^k)\\
    &+ 2\beta_h^{k,r}(s_h^k,a_h^k) + \beta_h^{k,pv}(s,a)\\
    \overset{(b)}{\le}&\left[(\widetilde{P}_h^k -P_h) (\widetilde{V}_{h+1}^k- {V}_{h+1}^*)\right](s_h^k,a_h^k) 
    + \left[ P_h(\widetilde{V}_{h+1}^k -V_{h+1}^{\pi_k})\right](s_h^k,a_h^k)+ 2\beta_h^{k,r}(s_h^k,a_h^k) + 2\beta_h^{k,pv}(s,a)
\end{aligned}
\end{equation}
where (a) we have used the reward concentration result in Lemma~\ref{ConcenPrivateVI}; (b) holds by the transition concentration result in  Lemma~\ref{ConcenPrivateVI}.
Thus, so far we have arrived at 
\begin{align}
     (\widetilde{V}_h^k -V_h^{\pi_k})(s_h^k)  \le & \left[(\widetilde{P}_h^k-P_h ) (\widetilde{V}_{h+1}^k-{V}_{h+1}^* )\right](s_h^k,a_h^k) +\left[ P_h( \widetilde{V}_{h+1}^k-V_{h+1}^{\pi_k} )\right](s_h^k,a_h^k)\nonumber \\
    &+ 2\beta_h^{k,r}(s_h^k,a_h^k) + 2\beta_h^{k,pv}(s_h^k,a_h^k)\label{eq:vi_recur_start}.
\end{align}

We will first carefully analyze the first term. In particular, let $G:= (\widetilde{V}_{h+1}^k-{V}_{h+1}^* )$, we have 
\begin{align}
&\left[(\widetilde{P}_h^k-P_h ) (\widetilde{V}_{h+1}^k- {V}_{h+1}^*)\right](s_h^k,a_h^k)\nonumber\\
    =&\sum_{s'}\left( \widetilde{P}_h^k(s'|s_h^k,a_h^k) -P_h(s'|s_h^k,a_h^k)\right)G(s')\nonumber\\
    \overset{(a)}{\le} & c \sum_{s'} \left( \sqrt{\frac{\ln(2SAT/\delta)P_h(s'|s_h^k,a_h^k)}{{(\widetilde{N}_h^k(s,a) +  E_{\epsilon,\delta,1} )\vee 1}}} +\frac{\ln(2SAT/\delta)}{(\widetilde{N}_h^k(s,a) +  E_{\epsilon,\delta,1})\vee 1}+ \frac{2 E_{\epsilon,\delta,1} + E_{\epsilon,\delta,3}}{(\widetilde{N}_h^k(s,a) +  E_{\epsilon,\delta,1})\vee 1}\right) G(s')\nonumber\\
    \overset{(b)}{\le} & \sum_{s'} \left(\frac{P_h(s'|s_h^k,a_h^k)}{H} G(s')  \right)+ c \sum_{s'} \left(\frac{H \ln(2SAT/\delta)} {(\widetilde{N}_h^k(s,a) +  E_{\epsilon,\delta,1})\vee 1}\right)G(s')\nonumber \\
    &+ c\sum_{s'} \left(\frac{\ln(2SAT/\delta)}{(\widetilde{N}_h^k(s,a) +  E_{\epsilon,\delta,1})\vee 1}\right)G(s') + c\sum_{s'}\left(\frac{2 E_{\epsilon,\delta,1} + E_{\epsilon,\delta,3}}{(\widetilde{N}_h^k(s,a) +  E_{\epsilon,\delta,1})\vee 1}\right)G(s')\nonumber\\
    \overset{(c)}{\le} & \sum_{s'} \left(\frac{P_h(s'|s_h^k,a_h^k)}{H}  G(s')\right) +  \sum_{s'} \left(\frac{c'H^2\tau \ln(2SAT/\delta)} {(\widetilde{N}_h^k(s,a) +  E_{\epsilon,\delta,1})\vee 1}\right)+c\sum_{s'}\left(\frac{4\tau H E_{\epsilon,\delta,1} + 2\tau HE_{\epsilon,\delta,3}}{(\widetilde{N}_h^k(s,a) +  E_{\epsilon,\delta,1})\vee 1}\right)\label{eq:vi_recur_temp},
\end{align}
where (a) holds by the third result in Lemma~\ref{ConcenPrivateVI} and $c$ is some absolute constant; (b) holds by $\sqrt{xy} \le x + y$ for positive numbers $x,y$; (c) holds since $G(s') \le 2H\tau$ by the boundedness of $V$-value. Now, plugging the definition for $G(s')$ into~\eqref{eq:vi_recur_temp}, yields
\begin{align}
&\left[(\widetilde{P}_h^k-P_h ) (\widetilde{V}_{h+1}^k-{V}_{h+1}^* )\right](s_h^k,a_h^k)\nonumber\\
    \le &\frac{1}{H} \left[P_h(\widetilde{V}_{h+1}^k-V_{h+1}^* )\right](s_h^k,a_h^k) + \frac{c'\tau SH^2 \ln(2SAT/\delta)} {(\widetilde{N}_h^k(s,a) +  E_{\epsilon,\delta,1})\vee 1} + \frac{4c\tau SH E_{\epsilon,\delta,1} + 2c\tau SHE_{\epsilon,\delta,3}}{(\widetilde{N}_h^k(s,a) +  E_{\epsilon,\delta,1})\vee 1}\nonumber\\
    \overset{(a)}{\le}&\frac{1}{H} \left[P_h(\widetilde{V}_{h+1}^k-V_{h+1}^* )\right](s_h^k,a_h^k) + \xi_h^k + \zeta_h^k\nonumber\\
    \overset{(b)}{\le}& \frac{1}{H}\left[P_h(\widetilde{V}_{h+1}^k-V_{h+1}^{\pi_k} )\right](s_h^k,a_h^k)+\xi_h^k + \zeta_h^k\label{eq:vi_recur_end},
\end{align}
where (a) holds by definitions $\xi_h^k=:\frac{c'\tau SH^2 \ln(2SAT/\delta)} {(\widetilde{N}_h^k(s,a) +  E_{\epsilon,\delta,1})\vee 1}$ and $\zeta_h^k:=\frac{4c\tau SH E_{\epsilon,\delta,1} + 2c\tau SHE_{\epsilon,\delta,3}}{(\widetilde{N}_h^k(s,a) +  E_{\epsilon,\delta,1})\vee 1}$; (b) holds since $V_{h+1}^{\pi_k} \le V_{h+1}^*$. Plugging~\eqref{eq:vi_recur_end} into~\eqref{eq:vi_recur_start}, yields the following recursive formula. 
\begin{align*}
      (\widetilde{V}_h^k-V_h^{\pi_k} )(s_h^k) &\overset{(a)}{\le} \left(1+\frac{1}{H}\right)\left[P_h(\widetilde{V}_{h+1}^k- V_{h+1}^{\pi_k})\right](s_h^k,a_h^k)+\xi_h^k + \zeta_h^k + 2\beta_h^k\\
     &\overset{(b)}{=} \left(1+\frac{1}{H}\right)\left[(\widetilde{V}_{h+1}^k- V_{h+1}^{\pi_k})(s_{h+1}^k) + \chi_{h}^k\right] + \xi_h^k + \zeta_h^k + 2\beta_h^k
\end{align*}
where in (a), we let $\beta_h^k:= \beta_h^{k,r}(s_h^k,a_h^k) + \beta_h^{k,pv}(s_h^k,a_h^k)$ for notation simplicity; (b) holds by definition $\chi_h^k:=\left[P_h(\widetilde{V}_{h+1}^k- V_{h+1}^{\pi_k})\right](s_h^k,a_h^k) - (\widetilde{V}_{h+1}^k- V_{h+1}^{\pi_k})(s_{h+1}^k)$. Based on this, we have the following bound on $(\widetilde{V}_1^k - V_1^{\pi_k})(s_1^k)$,
\begin{align}
     &(\widetilde{V}_1^k - V_1^{\pi_k})(s_1^k)\nonumber\\
     \le & \left(1+\frac{1}{H}\right)(\chi_{1}^k + \xi_1^k + \zeta_1^k + 2\beta_1^k) + \left(1+\frac{1}{H}\right)^2(\chi_{2}^k + \xi_2^k + \zeta_2^k + 2\beta_2^k) + \ldots\nonumber \\
     &+ \left(1+\frac{1}{H}\right)^H(\chi_{H}^k + \xi_H^k + \zeta_H^k + 2\beta_H^k)\nonumber\\
     \le & 3\sum_{h=1}^H (\chi_{h}^k + \xi_h^k + \zeta_h^k + 2\beta_h^k)\label{eq:vi_regret_end}.
\end{align}

Therefore, plugging~\eqref{eq:vi_regret_end} into~\eqref{eq:vi_regret_start}, we have the regret decomposition as follows.
\begin{align*}
    Reg(T) \le 3\sum_{k=1}^K\sum_{h=1}^H (\chi_{h}^k + \xi_h^k + \zeta_h^k + 2\beta_h^k)
\end{align*}
We are only left to bound each of them. To start with, we focus on the bonus term. We first focus on  $\beta_{h}^{k,r}(s,a)$ as shown in Lemma~\ref{ConcenPrivateVI}. By definition, we have 
\begin{align*}
    \sum_{k=1}^K\sum_{h=1}^H \beta_h^{k,r}(s,a)=\underbrace{\sum_{k=1}^K\sum_{h=1}^H \frac{2\tau E_{\varepsilon, \delta, 1}}{\left(\tilde{N}_{h}^{k}(s_h^k, a_h^k)+E_{\varepsilon, \delta, 1}\right) \vee 1}}_{\mathcal{O}_1}+\underbrace{10 u^{\frac{1}{1+v}}\sum_{k=1}^K\sum_{h=1}^H\left(\frac{H \log^{1.5}K \log (3SAT/\delta)}{\epsilon\left((\tilde{N}_h^{k}(s_h^k,a_h^k)+E_{\varepsilon, \delta, 1})\vee 1\right)}\right)^{\frac{v}{1+v}}}_{\mathcal{O}_2},
\end{align*}
The first term can be upper bounded as follows $(T=KH)$ under assumption \ref{Assum1}
\begin{align*}
    \mathcal{O}_1 &\le  2\tau E_{\epsilon,\delta,1}\sum_{k=1}^K\sum_{h=1}^H\frac{1}{N_h^k(s_h^k,a_h^k)
    \vee 1}\\
    &=  2\tau E_{\epsilon,\delta,1} \sum_{h,s,a}\sum_{i=1}^{N_h^K(s,a)} \frac{1}{i}\\
    &\le c'E_{\epsilon,\delta,1} \sum_{h,s,a}\ln(N_h^K(s,a))\\
    &= \widetilde{O}\left( HSAE_{\epsilon,\delta,1}\right).
\end{align*}
where $\widetilde{O}(\cdot)$ hides $polylog(S,A,T,1/\delta)$ factors.

The second term can be upper bounded as follows under Assumption~\ref{ConcenPrivateVI}.
\begin{align*}
    \mathcal{O}_2 &\le c u^{\frac{1}{1+v}}\left(\frac{H \log^{1.5}K \log (3SAT/\delta)}{\epsilon}\right)^{\frac{v}{1+v}}\sum_{k=1}^K\sum_{h=1}^H \left(\frac{1}{{N}_h^{k}(s_h^k,a_h^k)\vee 1}\right)^{\frac{v}{1+v}}\\
    & = c u^{\frac{1}{1+v}}\left(\frac{H \log^{1.5}K \log (3SAT/\delta)}{\epsilon}\right)^{\frac{v}{1+v}}\sum_{h,s,a}\sum_{i=1}^{N_h^K(s,a)} \frac{1}{i^{\frac{v}{1+v}}}\\
    &\le c^\prime u^{\frac{1}{1+v}}\left(\frac{H \log^{1.5}K \log (3SAT/\delta)}{\epsilon}\right)^{\frac{v}{1+v}}\sum_{h,s,a} (N_h^K(s,a))^{\frac{1}{1+v}}\\
    & \overset{(a)}{\le} c^\prime u^{\frac{1}{1+v}}\left(\frac{H \log^{1.5}K \log (3SAT/\delta)}{\epsilon}\right)^{\frac{v}{1+v}}\left(\sum_{h,s,a} 1 \right)^{\frac{v}{1+v}}\left(\sum_{h,s,a} N_h^K(s,a) \right)^{\frac{1}{1+v}}\\
    & \le \widetilde{O}\left(u^{\frac{1}{1+v}}\left(\frac{H^2SA}{\epsilon}\right)^{\frac{v}{1+v}}T^{\frac{1}{1+v}}\right)
\end{align*}
where (a) is based on Hölder's inequality with $x_k=1, y_k=(N_h^K(s,a))^{\frac{1}{1+v}}, q=1+v$ in Lemma \ref{holder}.

Putting them together, we have the upper bound for the summation over $\beta_h^{k,r}(s,a)$,
$$\sum_{k=1}^K\sum_{h=1}^H \beta_h^{k,r}(s,a)=\widetilde{O}\left( HSAE_{\epsilon,\delta,1}+u^{\frac{1}{1+v}}\left(\frac{H^2SA}{\epsilon}\right)^{\frac{v}{1+v}}T^{\frac{1}{1+v}}\right)$$

Now, we study the upper bound of $\beta_h^{k,pv}(s_h^k,a_h^k)$. By definition, we have 
\begin{align*}
    &\sum_{k=1}^K\sum_{h=1}^H \beta_h^{k,pv}(s_h^k,a_h^k) \\
    =&\underbrace{\tau H\sum_{k=1}^K\sum_{h=1}^H  \sqrt{\frac{2\ln(4SAT/\delta)} {{(\widetilde{N}_h^k(s_h^k,a_h^k) +  E_{\epsilon,\delta,1} )\vee 1}}}   }_{\mathcal{T}_1}+  \underbrace{\tau H\sum_{k=1}^K\sum_{h=1}^H\frac{ (S E_{\epsilon,\delta,3} + 2 E_{\epsilon,\delta,1})}{(\widetilde{N}_h^k(s_h^k,a_h^k) +  E_{\epsilon,\delta,1})\vee 1}}_{\mathcal{T}_2}.
\end{align*}
The first term can be upper bounded as follows ($T:= KH$) under Assumption~\ref{Assum1}.
\begin{align*}
    \mathcal{T}_1 &\le \tau H{\sqrt{2\ln(4SAT/\delta)}}\sum_{k=1}^K\sum_{h=1}^H \sqrt{\frac{1}{{{N}_h^k(s,a) \vee 1}}}\\
    &=\tau {\sqrt{2\ln(4SAT/\delta)}} \sum_{h,s,a}\sum_{i=1}^{N_h^K(s,a)} \frac{1}{\sqrt{i}}\\
    &\le c' H{\sqrt{2\ln(4SAT/\delta)}}\sum_{h,s,a} \sqrt{N_h^K(s,a)}\\
    &\le c' H{\sqrt{2\ln(4SAT/\delta)}}\sqrt{ \left(\sum_{h,s,a} 1\right) \left( \sum_{h,s,a} N_h^K(s,a)\right) }\\
    &=\widetilde{O}\left(\sqrt{H^3SAT}\right).
\end{align*}
The second term can be upper bounded as follows under Assumption~\ref{Assum1}. 
\begin{align*}
    \mathcal{T}_2 &\le c H(SE_{\epsilon,\delta,3} + E_{\epsilon,\delta,1})\sum_{k=1}^K\sum_{h=1}^H\frac{1}{N_h^k(s_h^k,a_h^k)
    \vee 1}\\
    &= c H(SE_{\epsilon,\delta,3} + E_{\epsilon,\delta,1}) \sum_{h,s,a}\sum_{i=1}^{N_h^K(s,a)} \frac{1}{i}\\
    &\le c'H(SE_{\epsilon,\delta,3} + E_{\epsilon,\delta,1}) \sum_{h,s,a}\ln(N_h^K(s,a))\\
    &= \widetilde{O}\left(H^2S^2AE_{\epsilon,\delta,3} + H^2SAE_{\epsilon,\delta,1}\right).
\end{align*}
Putting them together, we have the following bound on the summation over $\beta_h^k$.
\begin{align*}
    \sum_{k=1}^K\sum_{h=1}^H \beta_h^k = \widetilde{O}\left(\sqrt{H^3SAT} + H^2S^2AE_{\epsilon,\delta,3} +  H^2SAE_{\epsilon,\delta,1}+u^{\frac{1}{1+v}}\left(\frac{H^2SA}{\epsilon}\right)^{\frac{v}{1+v}}T^{\frac{1}{1+v}}\right).
\end{align*}

By following the same analysis as in $\mathcal{T}_2$, we can bound the summation over $\xi_h^k=:\frac{c'\tau SH^2 \ln(2SAT/\delta)} {(\widetilde{N}_h^k(s,a) +  E_{\epsilon,\delta,1})\vee 1}$ and $\zeta_h^k:=\frac{4c\tau SH E_{\epsilon,\delta,1} + 2c\tau SHE_{\epsilon,\delta,3}}{(\widetilde{N}_h^k(s,a) +  E_{\epsilon,\delta,1})\vee 1}$ as follows.
\begin{align*}
    &\sum_{k=1}^K\sum_{h=1}^H\xi_h^k = \widetilde{O}\left( H^3S^2A\right)\\
    &\sum_{k=1}^K\sum_{h=1}^H\zeta_h^k = \widetilde{O}\left(H^2S^2A(E_{\epsilon,\delta,3} + E_{\epsilon,\delta,1})\right).
\end{align*}

Finally, we are going to bound the summation over $\chi_h^k:=\left[P_h(\widetilde{V}_{h+1}^k- V_{h+1}^{\pi_k})\right](s_h^k,a_h^k) - ( \widetilde{V}_{h+1}^k-V_{h+1}^{\pi_k})(s_{h+1}^k)$, which turns out to be a martingale difference sequence. In particular, we define a filtration $\mathcal{F}_h^k$ that includes all the randomness up to the $k$-th episode and the $h$-th step. Then, we have $\mathcal{F}_1^1 \subset \mathcal{F}_2^1 \ldots \subset \mathcal{F}_H^1\subset \mathcal{F}_1^2\subset \mathcal{F}_2^2 \ldots $. Also, we have $(\widetilde{V}_{h+1}^k - V_{h+1}^{\pi_k}) \in \mathcal{F}_1^k \subset \mathcal{F}_h^k$ since they are decided by data collected up to episode $k-1$. A bit abuse of notation, we define $Y_{h+1}^k :=\chi_h^k$. Then, we have 
\begin{align}
    \mathbb{E}\left[Y_{h+1}^k |\mathcal{F}_h^k\right] = 0.
\end{align}
This holds since the expectation only captures randomness over $s_{h+1}^k$. Thus, $Y_{h+1}^k$ is a martingale difference sequence. Moreover, we have $|Y_{h+1}^k| \le 4H\tau$ a.s. By Azuma-Hoeffding inequality, we have with probability at least $1-\delta$
\begin{align*}
    \sum_{k=1}^K\sum_{h=1}^H \chi_h^k = \sum_{k=1}^K\sum_{h=1}^H Y_{h+1}^k = c'\sqrt{H^2T\ln(2/\delta)} = \widetilde{O}\left(\sqrt{H^2T} \right)
\end{align*}

Putting everything together, and applying union bound on all high-probability events,  we have shown that with probability at least $1-\delta$,
\begin{align*}
\label{eq:rt_vi_proof}
    Reg(T) = \tilde{O}\left(\sqrt{SAH^3T} + S^2AH^3 + S^2AH^2E_{\epsilon,\delta,1} + S^2AH^2E_{\epsilon,\delta,3}+u^{\frac{1}{1+v}}\left(\frac{H^2SA}{\epsilon}\right)^{\frac{v}{1+v}}T^{\frac{1}{1+v}}\right).
\end{align*}

Based on the value of $E_{\epsilon,\delta,1},E_{\epsilon,\delta,3}$ in Lemma \ref{centralError}, we obtain
$$Reg(T) = \tilde{O}\left(\sqrt{SAH^3T} + S^2AH^3/\epsilon + u^{\frac{1}{1+v}}\left(\frac{SAH^2}{\epsilon}\right)^{\frac{v}{1+v}}T^{\frac{1}{1+v}}\right).$$

\end{proof}

\subsection{ Proof of Theorem \ref{thm:RegretVI2}} 

Similar to the JDP case we first provide some concentration bounds. 
\begin{lemma}[Concentration bounds of locally private estimators]
\label{ConcenPrivateVI_LDP}
Fix any $\epsilon \in (0,1]$ and $\delta \in (0,1)$ and take $B_n=\left(\frac{u\epsilon\sqrt{n}}{H \log(6SAT/\delta)}\right)^{\frac{1}{1+v}}$ in equation \eqref{trunctedR}. Then, under Assumption \ref{Assum1}, with probability at least $1- 3\delta$, uniformly over all $(s,a,h,k)$,
{\footnotesize	$|\widetilde{r}_{h}^{k}(s, a)-{r}_{h}(s, a)|\le \beta_h^{k,r}(s,a),\quad
 \left|\left(\widetilde{P}_{h}^{k}-P_{h}\right) V_{h+1}^{*}(s, a)\right| \leq \beta_{h}^{k, p v}(s, a), \quad
\left|P_{h}\left(s^{\prime} \mid s, a\right)-\widetilde{P}_{h}^{k}\left(s^{\prime} \mid s, a\right)\right| \leq C\sqrt{\frac{P_h(s^\prime |s,a)\ln(2SAT/\delta)}{(\tilde{N}_h^k(s,a)+E_{\epsilon,\delta,1})\vee 1}}+\frac{C\ln(2SAT/\delta)+2E_{\epsilon,\delta,1}+E_{\epsilon,\delta,3}}{(\tilde{N}_h^k(s,a)+E_{\epsilon,\delta,1})\vee 1}
$}
where $C$ is a positive constant,$(PV_{h+1})(s,a)=\sum_{s^\prime}P(s^\prime|s,a)V_{h+1}(s^\prime)$, {\footnotesize$\beta_h^{k,r}(s,a)=\frac{2\tau E_{\varepsilon, \delta, 1}}{\left(\tilde{N}_{h}^{k}(s, a)+E_{\varepsilon, \delta, 1}\right) \vee 1}+16 u^{\frac{1}{1+v}}\left(\frac{H  \log (6SAT/\delta)}{\epsilon\sqrt{(\tilde{N}_h^{k}(s,a)+E_{\varepsilon, \delta, 1})\vee 1}}\right)^{\frac{v}{1+v}},$} and $\beta_{h}^{k, p v}(s, a)$ is defined in Lemma \ref{ConcenPrivateVI}.
\end{lemma}

\begin{remark}
Compared with the JDP case, we can see there are several differences. Firstly, due to the error caused by the noise we added to guarantee privacy becomes larger, in the LDP case we need to make $B_n$ smaller than it in the JDP case and finally, we can get the optimal one $B_n=\left(\frac{u\epsilon\sqrt{n}}{H \log(6SAT/\delta)}\right)^{\frac{1}{1+v}}$. It also indicates that the bound for LDP will be less than that for  JDP as we leverage less data information. Secondly, due to the stronger privacy guarantee in  the local model, we can see the second term of 
in $\beta_h^{k,r}(s,a)$ in above Lemma is worse than it in Lemma \ref{ConcenPrivateVI} by a factor of $\left({\sqrt{(\tilde{N}_h^{k}(s,a)+E_{\varepsilon, \delta, 1})\vee 1}}\right)^{-\frac{v}{1+v}}$. 
\end{remark}

\begin{proof}[\bf Proof of Lemma \ref{ConcenPrivateVI_LDP}]
Following the similar argue as proof of Lemma \ref{ConcenPrivateVI}, we can get \begin{equation}
\label{eq:conc2}
    |\widetilde{r}_{h}^{k}(s, a)-{r}_{h}(s, a)|\le \frac{2\tau E_{\varepsilon, \delta, 1}}{\left(\tilde{N}_{h}^{k}(s, a)+E_{\varepsilon, \delta, 1}\right) \vee 1}+\frac{L_{r,k} (N_{h}^{k}(s, a) \vee 1)}{\left(\widetilde{N}_{h}^{k}(s, a)+E_{\varepsilon, \delta, 1}\right) \vee 1}+\frac{E_{\varepsilon, \delta,k, 2}}{\left(\widetilde{N}_{h}^{k}(s, a)+E_{\varepsilon, \delta, 1}\right) \vee 1}
\end{equation}
where $E_{\epsilon, \delta,k, 2}= \frac{12HB_{N_{h}^{k}(s, a)} }{\epsilon}  \sqrt{N_{h}^{k}(s, a) \log (6SAT / \delta)}$ and $L_{r,k}=\sqrt{\frac{2uB_{N_h^{k}(s,a)}^{1-v}\log(\frac{2SAT}{\delta})}{N_h^{k}(s,a)\vee 1}}+\frac{B_{N_h^{k}(s,a)}\log\frac{2SAT}{\delta}}{3(N_h^{k}(s,a)\vee 1)}+  \frac{1}{N_h^{k}(s,a)\vee 1}\sum_{i=1}^{N_h^{k}(s,a)}\frac{u}{B_{i}^v}$.

Let $B_n=\left(\frac{u\epsilon\sqrt{n}}{H \log(6SAT/\delta)}\right)^{\frac{1}{1+v}}$, then 
$$L_{r,k}=\sqrt{\frac{2uB_{N_h^{k}(s,a)}^{1-v}\log(\frac{2SAT}{\delta})}{N_h^{k}(s,a)\vee 1}}+16u^{\frac{1}{1+v}}\left(\frac{H  \log (6SAT/\delta)}{\epsilon\sqrt{(\tilde{N}_h^{k}(s,a)+E_{\varepsilon, \delta, 1})\vee 1}}\right)^{\frac{v}{1+v}}.$$
\end{proof}

\begin{proof}[\bf Proof of Theorem \ref{thm:RegretVI2}]
Following the same idea in the proof of Theorem \ref{thm:RegretVI}, we need to calculate $\beta_{h}^{k,r}(s,a)$ as shown in Lemma~\ref{ConcenPrivateVI_LDP}. By definition, we have 
\begin{align*}
    \sum_{k=1}^K\sum_{h=1}^H \beta_h^{k,r}(s,a)=\underbrace{\sum_{k=1}^K\sum_{h=1}^H \frac{2\tau E_{\varepsilon, \delta, 1}}{\left(\tilde{N}_{h}^{k}(s_h^k, a_h^k)+E_{\varepsilon, \delta, 1}\right) \vee 1}}_{\mathcal{O}_1}+\underbrace{16 u^{\frac{1}{1+v}}\sum_{k=1}^K\sum_{h=1}^H\left(\frac{H \log (6SAT/\delta)}{\epsilon\sqrt{(\tilde{N}_h^{k}(s_h^k,a_h^k)+E_{\varepsilon, \delta, 1})\vee 1}}\right)^{\frac{v}{1+v}}}_{\mathcal{O}_2},
\end{align*}
The first term can be upper bounded as follows $(T=KH)$ under assumption \ref{Assum1}
\begin{align*}
    \mathcal{O}_1 = \widetilde{O}\left( HSAE_{\epsilon,\delta,1}\right).
\end{align*}
where $\widetilde{O}(\cdot)$ hides $polylog(S,A,T,1/\delta)$ factors.

The second term can be upper bounded as follows under Assumption~\ref{ConcenPrivateVI}.
\begin{align*}
    \mathcal{O}_2 &\le c u^{\frac{1}{1+v}}\left(\frac{H  \log (6SAT/\delta)}{\epsilon}\right)^{\frac{v}{1+v}}\sum_{k=1}^K\sum_{h=1}^H \left(\frac{1}{\sqrt{{N}_h^{k}(s_h^k,a_h^k)\vee 1}}\right)^{\frac{v}{1+v}}\\
    & = c u^{\frac{1}{1+v}}\left(\frac{H  \log (6SAT/\delta)}{\epsilon}\right)^{\frac{v}{1+v}}\sum_{h,s,a}\sum_{i=1}^{N_h^K(s,a)} \frac{1}{i^{\frac{v}{2(1+v)}}}\\
    &\le c^\prime u^{\frac{1}{1+v}}\left(\frac{H \log (6SAT/\delta)}{\epsilon}\right)^{\frac{v}{1+v}}\sum_{h,s,a} (N_h^K(s,a))^{\frac{2+v}{2(1+v)}}\\
    & \overset{(a)}{\le} c^\prime u^{\frac{1}{1+v}}\left(\frac{H  \log (3SAT/\delta)}{\epsilon}\right)^{\frac{v}{1+v}}\left(\sum_{h,s,a} 1 \right)^{\frac{v}{2(1+v)}}\left(\sum_{h,s,a} N_h^K(s,a) \right)^{\frac{2+v}{2(1+v)}}\\
    & \le \widetilde{O}\left(u^{\frac{1}{(1+v)}}\left(\frac{H^3SA}{\epsilon^2}\right)^{\frac{v}{2(1+v)}}T^{\frac{2+v}{2(1+v)}}\right)
\end{align*}
where (a) is based on Hölder's inequality with $x_k=1, y_k=(N_h^K(s,a))^{\frac{2+v}{2(1+v)}}, q=\frac{2(1+v)}{2+v}$ in Lemma \ref{holder}.

Hence,
$$Reg(T) = \tilde{O}\left(\sqrt{SAH^3T} + S^2AH^3 + S^2AH^2E_{\epsilon,\delta,1} + S^2AH^2E_{\epsilon,\delta,3}+u^{\frac{1}{(1+v)}}\left(\frac{H^3SA}{\epsilon^2}\right)^{\frac{v}{2(1+v)}}T^{\frac{2+v}{2(1+v)}}\right).$$
Plugging the value of error bound in Lemma \ref{ErrorLDP} to the regret bound, we obtain
$$Reg(T) = \tilde{O}\left(\sqrt{SAH^3T}  + \frac{S^2A\sqrt{H^5T}}{\epsilon} + u^{\frac{1}{(1+v)}}\left(\frac{H^3SA}{\epsilon^2}\right)^{\frac{v}{2(1+v)}}T^{\frac{2+v}{2(1+v)}}\right).$$
\end{proof}

\section{Algorithm and Proofs of Section \ref{Sec:PO}}

\begin{algorithm}[htb]
\caption{Private-Heavy-UCBPO}
\label{alg:PO}
\begin{algorithmic}[1]
\REQUIRE{Number of episodes $K$, time horizon $H$, privacy level $\epsilon > 0$, a PRIVATIZER (LOCAL or CENTRAL), confidence level $\delta \in (0,1]$ and parameter $\eta > 0$}
\STATE Initialize policy $\pi_h^1(a|s)=1/A$ for all $(s,a,h)$\;
\STATE Initialize private counts $\widetilde{R}_h^1(s,a)=0$, $\widetilde{N}_h^{1}(s,a)=0$ and $\widetilde{N}_h^{1}(s,a,s')=0$ for all $(s,a,s',h)$\;
\STATE Set precision levels $E_{\epsilon,\delta,1}, E_{\epsilon,\delta,2},E_{\epsilon,\delta,3}$ of the PRIVATIZER \;
\FOR{$k=1,2,3,\ldots,K$}
\STATE Initialize private value estimates: $\widetilde{V}^k_{H+1}(s) = 0$ \;
\FOR{$h=H,H-1,\ldots,1$} 
    \STATE Compute $\widetilde{r}_h^k(s,a)$ and $\widetilde{P}_h^k(s,a)$ $\forall (s,a)$ as in~\eqref{PrivateMean} using the private counts\;
    \STATE Set exploration bonus using Lemma~\ref{POBonus}: $\beta_h^k(s,a) = \beta_h^{k,r}(s,a) + \tau H\beta_h^{k,p}(s,a)$ $\forall (s,a)$ \;
    \STATE Compute: $\forall (s,a)$, \ 
    \STATE $\begin{aligned}
    \ \ \ \ \widetilde{{Q}}_h^k(s,a) &= \max\{-(H-h+1)\tau, \min\{(H-h+1)\tau,\widetilde{r}_h^k(s,a)\\
    &+ \sum_{s'\in \mathcal{S}} \widetilde{V}_{h+1}^k(s')\widetilde{P}_h^k(s'|s,a) + \beta_h^{k}(s,a)\} \}
    \end{aligned}$ 
    \STATE Compute private value estimates: $\forall s$, $\widetilde{V}_h^k(s) = \sum_{a \in \mathcal{A}}\widetilde{Q}_h^k(s,a) \pi_h^k(a|s)$ 
   \ENDFOR
\STATE Roll out a trajectory $(s_1^k,a_1^k,r_1^k,\ldots,s_{H+1}^k)$ by acting the policy $\pi^k=(\pi_h^k)_{h=1}^{H}$\;
\STATE Receive private counts $\widetilde{R}_h^{k+1}(s,a)$, $\widetilde{N}_h^{k+1}(s,a)$, $\widetilde{N}_h^{k+1}(s,a,s')$ from the PRIVATIZER\;
\STATE Update policy: $\forall (s,a,h)$,
        $\pi_h^{k+1}(a|s) = \frac{\pi_h^k(a|s)\exp(-\eta \widetilde{Q}_h^k(s,a)) }{\sum_{a \in \mathcal{A}} \pi_h^k(a|s)\exp(-\eta \widetilde{Q}_h^k(s,a))}$
\ENDFOR 
\end{algorithmic}
\end{algorithm}
\subsection{Proof of Theorem \ref{thm:regPO}}
Before showing the proof of Theorem \ref{thm:regPO}, we first prove the following lemma. 
\begin{lemma}[Concentration bounds of private estimators]
\label{POBonus}
Fix any $\epsilon \in (0,1]$ and $\delta \in (0,1)$ and take $B_n=\left(\frac{\epsilon u n}{H \log^{1.5}K \log (3SAT/\delta)}\right)^{\frac{1}{1+v}}$ in equation \eqref{trunctedR}. Then, under Assumption \ref{Assum1}, with probability at least $1- 2\delta$, uniformly over all $(s,a,h,k)$,
$$|\widetilde{r}_{h}^{k}(s, a)-{r}_{h}(s, a)|\le \beta_h^{k,r}(s,a),\, \|{P_h(\cdot|s,a) - \widetilde{P}_h^k(\cdot|s,a)}\|_1 \le \beta_h^{k,p}(s,a),$$
where 
\begin{align*}
    &\beta_h^{k,r}(s,a)=\frac{2\tau E_{\varepsilon, \delta, 1}}{\left(\tilde{N}_{h}^{k}(s, a)+E_{\varepsilon, \delta, 1}\right) \vee 1}+10 u^{\frac{1}{1+v}}\left(\frac{H \log^{1.5}K \log (3SAT/\delta)}{\epsilon\left((\tilde{N}_h^{k}(s,a)+E_{\varepsilon, \delta, 1})\vee 1\right)}\right)^{\frac{v}{1+v}},\\
    &\beta_h^{k,p}(s,a):=\frac{\sqrt{4S\ln({6AT}/{\delta)}}}{\sqrt{ (\tilde{N}_h^{k}(s,a)+E_{\varepsilon, \delta, 1})\vee 1}}+ \frac{SE_{\epsilon,\delta,3}+2 E_{\epsilon,{\delta}, 1}}{(\tilde{N}_h^{k}(s,a)+E_{\varepsilon, \delta, 1})\vee 1}.
\end{align*}
\end{lemma}

\begin{proof}[\bf Proof of Lemma \ref{POBonus}]
$\beta_h^{k,r}(s,a)$ is defined in Lemma  \ref{ConcenPrivateVI}. Now we prove the concentrated upper bound for transition probability. From Theorem 2.1 in \cite{weissman2003inequalities} and union bound over all $s,a,h,k$, we obtain with probability at least $1-\delta/2$
\begin{equation}
\label{eq:nonPriTran}
 \|P_h(\cdot|s,a) - \bar{P}_h^k(\cdot|s,a)\|_1  \le \sqrt{\frac{ 4S\ln(6AT/\delta)}{ {N_h^k(s,a) \vee 1}}}   
\end{equation}
where $\bar{P}_h^k(s'|s,a):= \frac{{N}_h^k(s,a,s')}{ {{N}_h^k(s,a)}\vee 1}$ is non-private empirical transition probability.

 Now, we turn to bound the transition dynamics. The error between the true transition probability and the private estimate can be decomposed as 
    \begin{align*}
        &\sum_{s'} |P_h(s'|,s,a) - \widetilde{P}_h^k(s'|s,a)|\\
        =&\sum_{s'} \left| \frac{\widetilde{N}_h^k(s,a,s')}{(\widetilde{N}_h^k(s,a) +  E_{\epsilon,\delta,1})\vee 1} - P_h(s'|s,a)\right|\\
        \le& \underbrace{\sum_{s'}\left|  \frac{{N}_h^k(s,a,s')}{(\widetilde{N}_h^k(s,a) +  E_{\epsilon,\delta,1})\vee 1} - P_h(s'|s,a) \right|}_{\mathcal{P}_1} +  \underbrace{\sum_{s'}\left|  \frac{\widetilde{N}_h^k(s,a,s') - {N}_h^k(s,a,s')}{(\widetilde{N}_h^k(s,a) +  E_{\epsilon,\delta,1})\vee 1} \right|}_{\mathcal{P}_2}.
    \end{align*}
        For $\mathcal{P}_1$, we have 
    \begin{align*}
        &\mathcal{P}_1=\sum_{s'}\left|\frac{N_h^k(s,a,s')}{N_h^k(s,a)\vee 1}\frac{N_h^k(s,a)\vee 1}{(\widetilde{N}_h^k(s,a) +  E_{\epsilon,\delta,1})\vee 1} - P_h(s'|s,a)\right|\\
        =&\sum_{s'}\left|\left( \frac{N_h^k(s,a,s')}{N_h^k(s,a)\vee 1} - P_h(s'|s,a) \right) \frac{N_h^k(s,a)\vee 1}{(\widetilde{N}_h^k(s,a) +  E_{\epsilon,\delta,1})\vee 1} + P_h(s'|s,a)\left(\frac{N_h^k(s,a)\vee 1}{(\widetilde{N}_h^k(s,a) +  E_{\epsilon,\delta,1})\vee 1}-1 \right) \right|\\
        \le& \frac{N_h^k(s,a)\vee 1}{(\widetilde{N}_h^k(s,a) +  E_{\epsilon,\delta,1})\vee 1} \|\bar{P}_h^k(\cdot|s,a) - P_h(\cdot|s,a)\|_1 + \sum_{s'}\left( P_h(s'|s,a) \frac{2 E_{\epsilon,\delta,1}}{(\widetilde{N}_h^k(s,a) +  E_{\epsilon,\delta,1})\vee 1}\right)\\
        \overset{(a)}{\le}&\frac{N_h^k(s,a)\vee 1}{(\widetilde{N}_h^k(s,a) +  E_{\epsilon,\delta,1})\vee 1} \frac{\sqrt{4S\ln({6AT}/{\delta})}}{\sqrt{N_h^k(s,a)\vee 1}} + \frac{2 E_{\epsilon,\delta,1}}{(\widetilde{N}_h^k(s,a) + E_{\epsilon,\delta,1})\vee 1}\\
        \le & \frac{\sqrt{4S\ln({6AT}/{\delta})}}{\sqrt{(\widetilde{N}_h^k(s,a) +  E_{\epsilon,\delta,1} )\vee 1}} + \frac{2 E_{\epsilon,\delta,1}}{(\widetilde{N}_h^k(s,a) +  E_{\epsilon,\delta,1})\vee 1},
    \end{align*}
   where (a) holds by the concentration of transition probability in inequality \eqref{eq:nonPriTran}.
   For $\mathcal{P}_2$, we have
    \begin{align*}
        \mathcal{P}_2 \le \sum_{s'}\frac{|E_{\epsilon,\delta,3}|}{(\widetilde{N}_h^k(s,a) +  E_{\epsilon,\delta,1})\vee 1} = \frac{S E_{\epsilon,\delta,3}}{(\widetilde{N}_h^k(s,a) +  E_{\epsilon,\delta,1})\vee 1}.
    \end{align*}
    Putting together $\mathcal{P}_1$ and $\mathcal{P}_2$, yields
    \begin{align*}
       \|P_h(\cdot|s,a) - \widetilde{P}_h^k(\cdot|s,a)\|_1 \le  \frac{\sqrt{4S\ln({6AT}/{\delta})}}{\sqrt{(\widetilde{N}_h^k(s,a) +  E_{\epsilon,1} )\vee 1}}+ \frac{SE_{\epsilon,\delta,3}+2E_{\epsilon,\delta,1}}{(\widetilde{N}_h^k(s,a) +  E_{\epsilon,\delta,1})\vee 1}.
    \end{align*}
\end{proof}

\begin{proof}[\bf Proof of Theorem \ref{thm:regPO}]
we first decompose the regret by using the extended value difference lemma~\citep[Lemma 1]{shani2020optimistic}.
\begin{equation*}
\begin{aligned}
    Reg(T)&=\sum\nolimits_{k=1}^K \left(V_1^{\pi^*}(s_1^k)-V_1^{\pi^k}(s_1^k) \right) = \sum\nolimits_{k=1}^K \left(V_1^{\pi^*}(s_1^k) -\widetilde{V}_1^k(s_1^k) + \widetilde{V}_1^k(s_1^k) - V_1^{\pi^k}(s_1^k)\right)\\
    &= \underbrace{\sum\nolimits_{k=1}^K\sum\nolimits_{h=1}^H \mathbb{E}\left[{ \langle{\widetilde{Q}_h^k(s_h,\cdot)},{\pi_h^*(\cdot|s_h)-\pi_h^k(\cdot|s_h) }\rangle| s_1^k,\pi^*}\right] }_{\mathcal{T}_1}\\
    &\quad\quad+\underbrace{\sum\nolimits_{k=1}^K\sum\nolimits_{h=1}^H \mathbb{E}\left[{r_h(s_h,a_h) + P_h(\cdot|s_h,a_h)\widetilde{V}_{h+1}^k- \widetilde{Q}_h^k(s_h,a_h)|s_1^k, \pi^*} \right]}_{\mathcal{T}_2}\\
    & \quad \quad +\underbrace{\sum\nolimits_{k=1}^K \left(\widetilde{V}_1^k(s_1^k)-V_1^{\pi^k}(s_1^k)\right)}_{\mathcal{T}_3}.
\end{aligned}
\end{equation*}
We then need to bound each of the three terms. 

\textbf{Analysis of $\mathcal{T}_1$}. To start with, we can bound $\mathcal{T}_1$ by following standard mirror descent analysis under KL divergence. Specifically, by~\citep[Lemma 6.7]{orabona2023modern}, we have for any $h \in [H]$, $s \in \mathcal{S}$ and any policy $\pi$
\begin{align*}
    \sum_{k=1}^K \langle{\widetilde{Q}_h^k(s,\cdot)},{\pi^*_h(\cdot|s)-\pi_h^k(\cdot|s)}\rangle &\le \frac{\log A}{\eta} + \frac{\eta}{2}\sum_{k=1}^K\|\widetilde{Q}_h^k(s,a))\|_{\infty}^2\\
    &\overset{(a)}{\le}  \frac{\log A}{\eta} + \frac{\eta \tau^2 H^2 K}{2},
\end{align*}
where (a) holds by  $\widetilde{Q}_h^k(s,a) \in [-\tau H,\tau H]$ for any $a \in \mathcal{A}$, which follows from the truncated update of $Q$-value in Algorithm~\ref{alg:PO} (line 10). Thus, we can bound $\mathcal{T}_1$ as follows.
\begin{align*}
    \mathcal{T}_1 &= \sum_{k=1}^K \sum_{h=1}^H \mathbb{E}\left[{\langle{\widetilde{Q}_h^k(s_h^k,\cdot)},{\pi_h^*(\cdot|s_h^k)-\pi_h^k(\cdot|s_h^k) }\rangle|s_1^k, \pi^* } \right]\le  \frac{H\log A}{\eta} + \frac{\eta \tau^2 H^3 K}{2}.
\end{align*}
Choosing $\eta = \sqrt{2\log A/(\tau^2H^2K)}$, yields
\begin{align}
    \mathcal{T}_1 \le \sqrt{2\tau^2 H^4 K \log A}.
\end{align}

\textbf{Analysis of $\mathcal{T}_2$}.
   First, by the update rule of $Q$-value in Algorithm~\ref{alg:PO} and  $P_h(\cdot|s,a)V_{h+1}:=\sum_{s'} P_h(s'|s,a) V_{h+1}(s')$, we have 
\begin{align*}
    \widetilde{Q}_h^k(s,a) &= \max\{-(H-h+1)\tau, \min\{(H-h+1)\tau,\widetilde{r}_h^k(s,a)  + \sum_{s'\in \mathcal{S}} \widetilde{V}_{h+1}^k(s')\widetilde{P}_h^k(s'|s,a) + \beta_h^{k}(s,a)\} \}\\
    &= \max\{-(H-h+1)\tau, \min\{(H-h+1)\tau,\widetilde{r}_h^k(s,a)  +  \widetilde{P}_h(\cdot|s_h^k,a_h^k)\widetilde{V}_{h+1}^k +\beta_h^{k}(s,a)\} \}\\
    &\le \max\left\{-(H-h+1)\tau,{\widetilde{r}_h^k(s,a)}+\beta_h^{k,r}(s,a) +  \widetilde{P}_h(\cdot|s_h^k,a_h^k)\widetilde{V}_{h+1}^k + \tau H\beta_h^{k,p}(s,a) \right\}\nonumber\\
    &\le \max\left\{0,{\widetilde{r}_h^k(s,a)}+\beta_h^{k,r}(s,a) +  \widetilde{P}_h(\cdot|s_h^k,a_h^k)\widetilde{V}_{h+1}^k + \tau H\beta_h^{k,p}(s,a) \right\}\nonumber\\
    &\overset{(a)}{\le} \max\left\{0, {\widetilde{r}_h^k(s,a)}+\beta_h^{k,r}(s,a) \right\} + \max\left\{0, \widetilde{P}_h(\cdot|s_h^k,a_h^k)\widetilde{V}_{h+1}^k+ \tau H\beta_h^{k,p}(s,a)   \right\}
\end{align*}
where (a) holds since for any $a,b$, $\max\{a+b,0\} \le \max\{a,0\} + \max\{b,0\}$. Thus, for any $(k,h,s,a)$, we have 
\begin{align}
    & r_h(s,a) +P_h(\cdot|s,a)\widetilde{V}_{h+1}^k-\widetilde{Q}_h^k(s,a)\nonumber\\
    \le & r_h(s,a) + P_h(\cdot|s,a)\widetilde{V}_{h+1}^k-\max\left\{0, {\widetilde{r}_h^k(s,a)}+\beta_h^{k,r}(s,a) \right\} - \max\left\{0,\widetilde{P}_h^{k}(\cdot|s,a)\widetilde{V}_h^k + \tau H\beta_h^{k,p}(s,a)   \right\} \nonumber\\
    =& r_h(s,a) + P_h(\cdot|s,a)\widetilde{V}_{h+1}^k+\min\left\{0, -{\widetilde{r}_h^k(s,a)}-\beta_h^{k,r}(s,a) \right\} + \min\left\{0,-\widetilde{P}_h^{k}(\cdot|s,a)\widetilde{V}_h^k - \tau H\beta_h^{k,p}(s,a)   \right\} \nonumber\\
    =& \min\left\{ r_h(s,a), { r_h(s,a)-\widetilde{r}_h^k(s,a)}-\beta_h^{k,r}(s,a)\right\} \label{eq:c_t3}\\
    &+ \min\left\{P_h(\cdot|s,a)\widetilde{V}_{h+1}^k, P_h(\cdot|s,a)\widetilde{V}_{h+1}^k-\widetilde{P}_h^{k}(\cdot|s,a)\widetilde{V}_h^k - \tau H\beta_h^{k,p}(s,a)  \right\}\label{eq:p_t3}.
\end{align}
We are going to show that both~\eqref{eq:c_t3} and~\eqref{eq:p_t3} are less than zero for all $(k,h,s,a)$ with high probability by Lemma~\ref{POBonus}. First, conditioned on the first result in Lemma~\ref{POBonus}, we have
$${ r_h(s,a)-\widetilde{r}_h^k(s,a)}-\beta_h^{k,r}(s,a) \le 0,$$
so ~\eqref{eq:c_t3} is less than zero. Further, we have conditioned on the second result in Lemma~\ref{POBonus}
\begin{align}
    &P_h(\cdot|s,a)\widetilde{V}_{h+1}^k-\widetilde{P}_h^{k}(\cdot|s,a)\widetilde{V}_h^k - \tau H\beta_h^{k,p}(s,a) \nonumber \\
    \overset{(a)}{\le} &\|{\widetilde{P}_h^{k}(\cdot|s,a) - P_h(\cdot|s,a)}\|_1 \|{\widetilde{V}_{h+1}^k}\|_{\infty}- \tau H\beta_h^{k,p}(s,a)\nonumber\\
    \overset{(b)}{\le} & \tau H \|{\widetilde{P}_h^{k}(\cdot|s,a) - P_h(\cdot|s,a)}\|_1- \tau H\beta_h^{k,p}(s,a)\nonumber\\
    \overset{(c)}{\le} & 0\label{eq:conc_pv}
\end{align}
where (a) holds by Holder's inequality; (b) holds since $-\tau H\le \widetilde{V}_{h+1}^k \le \tau H$ based on our update rule; (c) holds by Lemma~\ref{POBonus}, so ~\eqref{eq:p_t3} is less than 0. Thus, we have shown that 
\begin{align}
    \mathcal{T}_2 \le 0.
\end{align}

\textbf{Analysis of $\mathcal{T}_3$}. Assume  the event in Assumption~\ref{Assum1} hold (which implies the concentration results in Lemma~\ref{POBonus}). We have 
\begin{align}
    \mathcal{T}_3 = &\sum_{k=1}^K \left(\widetilde{V}_1^k(s_1)- V_1^{\pi_k}(s_1)\right) \nonumber\\
   \overset{(a)}{=}& \sum_{k=1}^K\sum_{h=1}^H\mathbb{E}\left[{ \widetilde{Q}_{h+1}^k(s_h,a_h)-r_h(s_h,a_h) - P_h(\cdot |s_h,a_h) \widetilde{V}_{h+1}^k |s_1^k, \pi_k}\right]\nonumber\\
   \overset{(b)}{=}&\sum_{k=1}^K\sum_{h=1}^H\mathbb{E}\left[{ \min\left \{ \widetilde{r}_h^k(s_h,a_h) + \beta_h^{k,r}(s_h,a_h)+ \widetilde{P}_h^k(\cdot|s_h,a_h)\widetilde{V}_{h+1}^k + \tau H\beta_h^{k,p}(s_h,a_h) ,(H-h+1)\tau \right\} |s_1^k, \pi_k}\right] \nonumber\\
  -&\sum_{k=1}^K\sum_{h=1}^H\mathbb{E}\left[{ r_h(s_h,a_h) + P_h(\cdot |s_h,a_h) \widetilde{V}_{h+1}^k|s_1^k, \pi_k}\right]\label{eq:t1_po}
\end{align}
where (a) holds by the extended value difference lemma~\citep[Lemma 1]{shani2020optimistic}; (b) holds by the update rue of Q-value in Algorithm~\ref{alg:PO}. Note that here we can directly remove the truncation at $-(H-h+1)\tau$ since by Lemma~\ref{POBonus}, $\widetilde{r}_h^k(s_h,a_h) + \beta_h^{k,r}(s_h,a_h)+ \widetilde{p}(\cdot|s_h,a_h)\widetilde{V}_{h+1}^k + \tau H\beta_h^{k,p}(s_h,a_h) \ge r_h(s,a) + P_h(\cdot|s,a) \widetilde{V}_{h+1}^k \ge -(1+H-h)\tau $. 

Now, observe that for any $(k,h,s,a)$, we have 
\begin{align}
     & \min\left \{ \widetilde{r}_h^k(s_h,a_h) + \beta_h^{k,r}(s,a)+ \widetilde{P}_h^k(\cdot|s,a)\widetilde{V}_{h+1}^k + \tau H\beta_h^{k,p}(s,a) ,(H-h+1)\tau \right\}-r_h(s,a) - P_h(\cdot |s,a) \widetilde{V}_{h+1}^k \nonumber\\
     &\le {\widetilde{r}_h^k(s,a)}-r_h(s,a) + \beta_h^{k,r}(s,a) +  \widetilde{P}_h^k(\cdot |s,a) \widetilde{V}_{h+1}^k-P_h(\cdot |s,a) \widetilde{V}_{h+1}^k + \tau H\beta_h^{k,p}(s,a)\nonumber\\
     &\overset{(a)}{\le} 2 \beta_h^{k,r}(s,a) + 2\tau H \beta_h^{k,p}(s,a)\label{eq:t1_po_all},
\end{align}
where (a) holds by Lemma~\ref{POBonus} and a similar analysis as in~\eqref{eq:conc_pv}. 
Plugging~\eqref{eq:t1_po_all} into~\eqref{eq:t1_po}, yields 
\begin{align}
    \mathcal{T}_3  \le \underbrace{\sum_{k=1}^K\sum_{h=1}^H\mathbb{E}\left[{ 2\beta_h^{k,r}(s_h,a_h)|s_1^k, \pi^k}\right]}_{\text{Term(i)}} + \underbrace{\tau H\sum_{k=1}^K\sum_{h=1}^H\mathbb{E}\left[{ 2\beta_h^{k,p}(s_h,a_h)|s_1^k, \pi^k}\right]}_{\text{Term(ii)}}
\end{align}
By the definition of $\beta_h^{k,r}$ and $\beta_h^{k,p}$ in Lemma~\ref{POBonus} and Assumption~\ref{Assum1}, we have with probability $1-2\delta$,

\begin{align}
    \text{Term(i)} &\le  2\sum_{k=1}^K\sum_{h=1}^H\mathbb{E}\left[\frac{2\tau E_{\varepsilon, \delta, 1}}{{N}_{h}^{k}(s, a) \vee 1}+10 u^{\frac{1}{1+v}}\left(\frac{H \log^{1.5}K \log (3SAT/\delta)}{\epsilon\left({N}_{h}^{k}(s, a)\vee 1\right)}\right)^{\frac{v}{1+v}}\right], \label{RegTerm1}\\
    \text{Term(ii)} &\le 2\tau H\sum_{k=1}^K\sum_{h=1}^H \mathbb{E}\left[\frac{\sqrt{4S\ln({6AT}/{\delta)}}}{\sqrt{ {N}_{h}^{k}(s, a)\vee 1}}+ \frac{SE_{\epsilon,\delta,3}+2 E_{\epsilon,{\delta}, 1}}{{N}_{h}^{k}(s, a)\vee 1}\right].
\end{align}

In order to bound the two terms above, we use the following lemmas.
\begin{lemma}
\label{lem:nonst}
With probability $1-2\delta$, we have 
$$
    \sum_{k=1}^K\sum_{h=1}^H \mathbb{E}\left[{ \frac{1}{N_h^k(s_h,a_h) \vee 1}|\mathcal{F}_{k-1}}\right] = O\left(SAH\ln (KH) + H\ln(H/\delta)\right),
$$
and 
$$
    \sum_{k=1}^K\sum_{h=1}^H \mathbb{E}\left[{ \frac{1}{\sqrt{N_h^k(s_h,a_h)\vee 1} }|\mathcal{F}_{k-1}}\right] = O\left(\sqrt{SAH^2K} + SAH\ln KH + H\ln(H/\delta)\right),
$$
where the filtration $\mathcal{F}_k$ includes all the events until the end of episode $k$.
\end{lemma}
The results of Lemma \ref{lem:nonst} have been proved in Lemma A.2 of \cite{chowdhury2021differentially}.

In order to bound \eqref{RegTerm1}, we use the following standard Bernstein-type concentration inequality for martingale from Lemma 9 in \cite{jin2020learning}.

\begin{lemma}
\label{lem:MDS}
Let $Y_1,\ldots,Y_K$ be a martingale difference sequence with respect to a filtration $\mathcal{F}_0, \mathcal{F}_1, \ldots, \mathcal{F}_K$. Assume $Y_k \le R$ a.s. for all $i$. Then, for any $\delta \in (0,1)$ and $\lambda \in [0,1/R]$, with probability $1-\delta$, we have 
\begin{align*}
    \sum_{k=1}^K Y_k \le \lambda \sum_{k=1}^K\mathbb{E}\left[ Y_k^2|\mathcal{F}_{k-1}\right] + \frac{\ln(1/\delta)}{\lambda}.
\end{align*}
\end{lemma}

Now we can use the above lemma to prove the following lemma which is the key point to bound \eqref{RegTerm1}.

\begin{lemma}
\label{lem:heavyCount}
With probability $1-\delta$, we have 
$$
    \sum_{k=1}^K\sum_{h=1}^H \mathbb{E}\left[{ \frac{1}{(N_h^k(s_h,a_h) \vee 1)^{\frac{v}{1+v}}}|\mathcal{F}_{k-1}}\right] = O\left((SAH)^{\frac{v}{1+v}}T^{\frac{1}{1+v}} + H\ln(H/\delta) \right),
$$
where the filtration $\mathcal{F}_k$ includes all the events until the end of episode $k$.
\end{lemma}
\begin{proof}[\bf Proof of Lemma \ref{lem:heavyCount}]
Let $\mathcal{I}_h^k(s,a)$ be the indicator of whether the pair $(s,a)$ at step $h$ and episode $k$ so that $\mathbb{E}\left[\mathcal{I}_h^k(s,a) |\mathcal{F}_{k-1}\right] = w_h^k(s,a)$, which is the probability of visiting state-action pair $(s,a)$ at step $h$ and episode $k$. First note that 
    \begin{align*}
         &\sum_{k=1}^K\sum_{h=1}^H \mathbb{E}\left[{ \frac{1}{(N_h^k(s_h,a_h) \vee 1)^{\frac{v}{1+v}}}|\mathcal{F}_{k-1}}\right] \\
         = & \sum_{k=1}^K \sum_{h,s,a} w_h^k(s,a)  \frac{1}{(N_h^k(s,a) \vee 1)^{\frac{v}{1+v}}}\\
         = & \sum_{k=1}^K\sum_{h,s,a}  \frac{\mathcal{I}_h^k(s,a)}{(N_h^k(s,a) \vee 1)^{\frac{v}{1+v}}} + \sum_{k=1}^K\sum_{h,s,a}  \frac{w_h^k(s,a) - \mathcal{I}_h^k(s,a)}{(N_h^k(s,a) \vee 1)^{\frac{v}{1+v}}}.
    \end{align*}
The first term can be bounded as follows.
\begin{align*}
    \sum_{k=1}^K\sum_{h,s,a}  \frac{\mathcal{I}_h^k(s,a)}{(N_h^k(s,a) \vee 1)^{\frac{v}{1+v}}} &\le \sum_{h,s,a}  \sum_{k=1}^K \frac{1}{(N_h^k(s,a) \vee 1)^{\frac{v}{1+v}}}\\
    & =\sum_{h,s,a} \sum_{i=1}^{N_h^K(s,a)} \frac{1}{i^{\frac{v}{1+v}}}\\
    &\le c'\sum_{h,s,a} \left(N_h^K(s,a)\right)^{\frac{1}{1+v}}\\
    &\overset{(a)}{\le} \left(\sum_{h,s,a} 1\right)^{\frac{v}{1+v}}\left(\sum_{h,s,a} N_h^K(s,a)\right)^{\frac{1}{1+v}} \\
    &=O\left( (SAH)^{\frac{v}{1+v}}T^{\frac{1}{1+v}}\right).
\end{align*}

To bound the second term, we will use Lemma~\ref{lem:MDS}. In particular, consider $Y_{k,h} :=\sum_{s,a}  \frac{w_h^k(s,a) - \mathcal{I}_h^k(s,a)}{(N_h^k(s,a) \vee 1)^{\frac{v}{1+v}}} \le 1$, $\lambda = 1$, and the fact that for any fixed $h$,
\begin{align*}
    \mathbb{E}\left[ Y_{k,h}^2|\mathcal{F}_{k-1}\right] &\le \mathbb{E}\left[ \left(\sum_{s,a} \frac{\mathcal{I}_h^k(s,a)}{(N_h^k(s,a) \vee 1)^{\frac{v}{1+v}}}\right)^2\mid\mathcal{F}_{k-1}\right]\\
    &\overset{(a)}{=} \mathbb{E}\left[ \sum_{s,a} \frac{\mathcal{I}_h^k(s,a)}{(N_h^k(s,a) \vee 1)^{\frac{2v}{1+v}}}\mid \mathcal{F}_{k-1}\right]\\
    &\le  \sum_{s,a} \frac{w_h^k(s,a)}{(N_h^k(s,a) \vee 1)^{\frac{v}{1+v}}}.
\end{align*}
where $(a)$ is based on $\mathcal{I}_h^k(s,a)\mathcal{I}_h^k(s^\prime,a^\prime)=0 $ for $s\neq s^\prime $ or $ a \neq a^\prime $.
Then, via Lemma~\ref{lem:MDS}, we have with probability at least $1-\delta$, 
\begin{align*}
    \sum_{k=1}^K\sum_{h,s,a}  \frac{w_h^k(s,a) - \mathcal{I}_h^k(s,a)}{(N_h^k(s,a) \vee 1)^{\frac{v}{1+v}}} = \sum_{h=1}^H\sum_{k=1}^K Y_{k,h} &\le \sum_{h=1}^H \sum_{k=1}^K\sum_{s,a} \frac{w_h^k(s,a)}{(N_h^k(s,a) \vee 1)^{\frac{v}{1+v}}} + H\ln(H/\delta)\\
    & = O\left((SAH)^{\frac{v}{1+v}}T^{\frac{1}{1+v}} + H\ln(H/\delta)\right),
\end{align*}
Then we complete the proof of the lemma.    
\end{proof}

From Lemma \ref{lem:nonst} and Lemma \ref{lem:heavyCount}, we can get the upper bounds for Term(i) and Term (ii):
\begin{align}
    \text{Term(i)} &= \tilde{O}\left(SAH E_{\epsilon,\delta,1}+u^{\frac{1}{1+v}}\left(\frac{SAH^2}{\epsilon}\right)^{\frac{v}{1+v}}T^{\frac{1}{1+v}}\right)\\
    \text{Term(ii)} &= \tilde{O}\left(\sqrt{S^2AH^4K} + \sqrt{S^3A^2H^4} + E_{\epsilon,\delta,3}S^2AH^2 + E_{\epsilon,\delta,1}SAH^2\right)
\end{align}

Hence, $\mathcal{T}_3 = \tilde{O}\left(\sqrt{S^2AH^3T}+\frac{S^2AH^3}{\epsilon}+u^{\frac{1}{1+v}}\left(\frac{SAH^2}{\epsilon}\right)^{\frac{v}{1+v}}T^{\frac{1}{1+v}}\right)$

Finally, we can get the upper bound of regret.
\end{proof}

\begin{lemma}[Concentration bounds of locally private estimators]
Fix any $\epsilon \in (0,1]$ and $\delta \in (0,1)$ and take $B_n=\left(\frac{u\epsilon\sqrt{n}}{H \log(6SAT/\delta)}\right)^{\frac{1}{1+v}}$ in equation \eqref{trunctedR}. Then, under Assumption \ref{Assum1}, with probability at least $1- 2\delta$, uniformly over all $(s,a,h,k)$,
$$|\widetilde{r}_{h}^{k}(s, a)-{r}_{h}(s, a)|\le \beta_h^{k,r}(s,a),\, \|{P_h(\cdot|s,a) - \widetilde{P}_h^k(\cdot|s,a)}\|_1 \le \beta_h^{k,p}(s,a),$$

where
\begin{align*}
   &\beta_h^{k,r}(s,a)=\frac{2\tau E_{\varepsilon, \delta, 1}}{\left(\tilde{N}_{h}^{k}(s, a)+E_{\varepsilon, \delta, 1}\right) \vee 1}+16 u^{\frac{1}{1+v}}\left(\frac{H  \log (6SAT/\delta)}{\epsilon\sqrt{(\tilde{N}_h^{k}(s,a)+E_{\varepsilon, \delta, 1})\vee 1}}\right)^{\frac{v}{1+v}},\\
   &\beta_h^{k,p}(s,a):=\frac{\sqrt{4S\ln({6AT}/{\delta)}}}{\sqrt{ (\tilde{N}_h^{k}(s,a)+E_{\varepsilon, \delta, 1})\vee 1}}+ \frac{SE_{\epsilon,\delta,3}+2 E_{\epsilon,{\delta}, 1}}{(\tilde{N}_h^{k}(s,a)+E_{\varepsilon, \delta, 1})\vee 1}.
\end{align*}
\end{lemma}
In fact, $\beta_h^{k,r}(s,a)$ is the same as the form defined in Lemma \ref{ConcenPrivateVI_LDP} since we use the same mean estimation for truncated heavy-tailed rewards in the LDP model. Moreover, $\beta_h^{k,p}(s,a)$ is the same as the one in Lemma \ref{POBonus}.

\subsection{Proof of Theorem \ref{thm:regPO2}}
\begin{proof}[\bf Proof of Theorem \ref{thm:regPO2}]
Similar to the idea of Theorem \ref{thm:regPO}'s proof, we first decompose the regret by using the extended value difference lemma~\citep[Lemma 1]{shani2020optimistic}.
\begin{equation*}
\begin{aligned}
    Reg(T)&=\sum\nolimits_{k=1}^K \left(V_1^{\pi^*}(s_1^k)-V_1^{\pi^k}(s_1^k) \right) = \sum\nolimits_{k=1}^K \left(V_1^{\pi^*}(s_1^k) -\widetilde{V}_1^k(s_1^k) + \widetilde{V}_1^k(s_1^k) - V_1^{\pi^k}(s_1^k)\right)\\
    &= \underbrace{\sum\nolimits_{k=1}^K\sum\nolimits_{h=1}^H \mathbb{E}\left[{ \langle{\widetilde{Q}_h^k(s_h,\cdot)},{\pi_h^*(\cdot|s_h)-\pi_h^k(\cdot|s_h) }\rangle| s_1^k,\pi^*}\right] }_{\mathcal{T}_1}\\
    &\quad\quad+\underbrace{\sum\nolimits_{k=1}^K\sum\nolimits_{h=1}^H \mathbb{E}\left[{r_h(s_h,a_h) + P_h(\cdot|s_h,a_h)\widetilde{V}_{h+1}^k- \widetilde{Q}_h^k(s_h,a_h)|s_1^k, \pi^*} \right]}_{\mathcal{T}_2}\\
    & \quad \quad +\underbrace{\sum\nolimits_{k=1}^K \left(\widetilde{V}_1^k(s_1^k)-V_1^{\pi^k}(s_1^k)\right)}_{\mathcal{T}_3}.
\end{aligned}
\end{equation*}
We then need to bound each of the three terms. 

By~\citep[Lemma 17]{shani2020optimistic} and choosing $\eta = \sqrt{2\log A/(\tau^2H^2K)}$, we obtain $ \mathcal{T}_1 \le \sqrt{2\tau^2 H^4 K \log A}.
$. Furthermore, due to update rule of $Q$-function and Lemma \ref{POBonus}, we have $\mathcal{T}_2 \le 0.$

Now we focus on bounding $\mathcal{T}_3$.  By the extended value difference lemma~\citep[Lemma 1]{shani2020optimistic}, the update rue of Q-value in Algorithm~\ref{alg:PO} and the bonus term according to Lemma \ref{POBonus} , we can decompose the term into two parts:
\begin{align}
    \mathcal{T}_3  \le \underbrace{\sum_{k=1}^K\sum_{h=1}^H\mathbb{E}\left[{ 2\beta_h^{k,r}(s_h,a_h)|s_1^k, \pi^k}\right]}_{\text{Term(i)}} + \underbrace{\tau H\sum_{k=1}^K\sum_{h=1}^H\mathbb{E}\left[{ 2\beta_h^{k,p}(s_h,a_h)|s_1^k, \pi^k}\right]}_{\text{Term(ii)}}
\end{align}
where 
\begin{align}
    \text{Term(i)} &\le  2\sum_{k=1}^K\sum_{h=1}^H\mathbb{E}\left[\frac{2\tau E_{\varepsilon, \delta, 1}}{{N}_{h}^{k}(s, a) \vee 1}+16 u^{\frac{1}{1+v}}\left(\frac{H  \log (6SAT/\delta)}{\epsilon\sqrt{{N}_h^{k}(s,a)\vee 1}}\right)^{\frac{v}{1+v}}\right],\\
    \text{Term(ii)} &\le 2\tau H\sum_{k=1}^K\sum_{h=1}^H \mathbb{E}\left[\frac{\sqrt{4S\ln({6AT}/{\delta)}}}{\sqrt{ {N}_{h}^{k}(s, a)\vee 1}}+ \frac{SE_{\epsilon,\delta,3}+2 E_{\epsilon,{\delta}, 1}}{{N}_{h}^{k}(s, a)\vee 1}\right].
\end{align}
From the same argument in the proof of Theorem \ref{thm:regPO}, we can obtain
$$ \text{Term(ii)} = \tilde{O}\left(\sqrt{S^2AH^4K} + \sqrt{S^3A^2H^4} + E_{\epsilon,\delta,3}S^2AH^2 + E_{\epsilon,\delta,1}SAH^2\right).$$
Then the only thing left is to bound the Term(i). Using a similar idea in the proof of Lemma \ref{lem:heavyCount} we can get 
$$
    \sum_{k=1}^K\sum_{h=1}^H \mathbb{E}\left[{ \frac{1}{(N_h^k(s_h,a_h) \vee 1)^{\frac{v}{2(1+v)}}}|\mathcal{F}_{k-1}}\right] = O\left((SAH)^{\frac{v}{2(1+v)}}T^{\frac{2+v}{2(1+v)}} + H\ln(H/\delta) \right),
$$
Based on the first result in Lemma \ref{lem:nonst}, we have 
$$
\text{Term(i)} = \tilde{O}\left(SAH E_{\epsilon,\delta,1}+u^{\frac{1}{1+v}}\left(\frac{SAH^3}{\epsilon^2}\right)^{\frac{v}{2(1+v)}}T^{\frac{2+v}{2(1+v)}}\right).$$

Finally, based on the results of Lemma \ref{ErrorLDP}, we can derive the result of regret:
$$Reg(T)=\tilde{O}\left(\sqrt{S^2AH^3T}+\frac{S^2A\sqrt{H^5T}}{\epsilon} +u^{\frac{1}{1+v}}\left(\frac{H^3SA}{\epsilon^2}\right)^{\frac{v}{2(1+v)}}T^{\frac{2+v}{2(1+v)}}\right).$$
\end{proof}


\section{Proofs of Section \ref{Sec:LowerBou}}
\label{Appen:Lower}
\subsection{Proof of Theorem \ref{thm:LowBounMAB}}
\begin{proof}[\bf Proof of Theorem \ref{thm:LowBounMAB}]
Firstly, we construct the environments which are hard to distinguish. We define the instance $\bar{P}_1$ in which the optimal arm (denote by $a_1$) follows the reward distribution 
\[
    \nu_1=\left(1-\frac{\gamma^{1+v}}{2}\right)\delta_0 +\frac{\gamma^{1+v}}{2}\delta_{1/\gamma},
\]
where $\gamma=(5\Delta)^{\frac{1}{v}}$ with $\Delta$ is a constant to be specified later and $\Delta \in \left(0,\frac{1}{5} \right)$, and $\delta_x$ is the Dirac distribution on $x$ and the distribution $p\cdot\delta_x+(1-p)\cdot\delta_y$ takes the value $x$ with probability $p$ and the value $y$ with probability $1-p$. It is easy to verify that  $\mathbb{E}[\nu_1]=\frac{5}{2}\Delta$, and the $(1+v)$-th raw moment of $v_1$ is $u(\nu_1)=\frac{1}{2} \le 1$.

Any other sub-optimal arm $a\neq a_1$ in $\bar{P}_1$ follows the same reward distribution
 \[
    \nu_a=\left(1-\frac{\gamma^{1+v}}{2}+\Delta\gamma\right)\delta_0 +\left(\frac{\gamma^{1+v}}{2}-\Delta\gamma\right)\delta_{1/\gamma}.
\]
Note that for all $a\neq a_1$ $\mathbb{E}[\nu_a]=\frac{3}{2}\Delta$, $u(\nu_a)=\frac{1}{2}-\frac{1}{5}=\frac{3}{10}<1$.

For algorithm $\mathcal{M}$ and instance $\bar P_1$, we denote $
    i={\arg \min} _{a\in \{2,\cdots,A\}}\mathbb{E}_{\mathcal{M}\bar{P}_1}[N_a(K)].$ where  $\mathbb{E}_{\mathcal{M}{P}}$ is the expectation over the the probability measure $\mathbb{P}_{\mathcal{M} P}$ induced by the algorithm $\mathcal{M}$ and the instance $P$. Thus, $\mathbb{E}_{\mathcal{M}\bar{P}_1}[N_i(K)]\leq\frac{K}{A-1}$.

Now, consider another instance $\bar{P}_i$ where $\nu_1,\cdots,\nu_A$ are the same as those in $\bar P_1$ except the $i$-th arm such that
\[
    \nu_i^\prime=\left(1-\frac{\gamma^{1+v}}{2}-\Delta\gamma\right)\delta_0 +\left(\frac{\gamma^{1+v}}{2}+\Delta\gamma\right)\delta_{1/\gamma}.
\]
Note that now $\mathbb{E}[\nu_i^\prime]=\frac{7}{2}\Delta$, $u(\nu_i^\prime)=\frac{7}{10}<1$. Then in $\bar P_i$, the arm $i$ is optimal.

Now by the classic regret decomposition, we obtain 
$${Reg}_{K,\bar P_1}^{\mathcal{M}} =  (K-\mathbb{E}_{\mathcal{M}\bar{P}_1}[N_1(K)])\Delta \ge \mathbb{P}_{\mathcal{M}\bar{P}_1}^K \left[N_1(K) \le \frac{K}{2}\right]\frac{K\Delta}{2}. $$
$${Reg}_{K,\bar P_i}^{\mathcal{M}} = \Delta \mathbb{E}_{\mathcal{M}\bar{P}_i}[N_1(K)] + \sum_{a\notin\{1,i\}}2\Delta \mathbb{E}_{\mathcal{M}\bar{P}_i}[N_a(K)] \ge \mathbb{P}_{\mathcal{M}\bar{P}_i}^K \left[N_1(K) \ge \frac{K}{2}\right]\frac{K\Delta}{2}.$$

By applying the Bretagnolle–Huber inequality (\cite{lattimore2020bandit}, Theorem 14.2), we have 

$$\begin{aligned}
    {Reg}_{K,\bar P_1}^{\mathcal{M}}+ {Reg}_{K,\bar P_i}^{\mathcal{M}} 
    & \ge \frac{K\Delta}{2}\left(\mathbb{P}_{\mathcal{M}\bar{P}_1}^K \left[N_1(K) \le \frac{K}{2}\right]+\mathbb{P}_{\mathcal{M}\bar{P}_i}^K \left[N_1(K) \ge \frac{K}{2}\right]\right). \\
    & \ge \frac{K\Delta}{4} \exp{\left(-\text{KL}\left(\mathbb{P}_{\mathcal{M}\bar{P}_1}^K \|\mathbb{P}_{\mathcal{M}\bar{P}_i}^K\right)\right)}
\end{aligned}$$

\begin{lemma}[Upper Bound on KL-divergence for Bandits with $\epsilon$-DP \cite{https://doi.org/10.48550/arxiv.2209.02570}]
If $\mathcal{M}$ is a mechanism satisfying $\epsilon$-DP, then for two instances $P_1=(\nu_a: a \in [A])$ and $P_2=(\nu_a^\prime: a \in [A])$ we have
$$\text{KL}\left(\mathbb{P}_{\mathcal{M}{P}_1}^K \|\mathbb{P}_{\mathcal{M}{P}_2}^K\right) \le 6\epsilon \mathbb{E}_{\mathcal{M}{P}_1}\left[\sum_{t=1}^K \text{TV}(\nu_{a_t}\|\nu^\prime_{a_t})\right]$$
where $\text{TV}(\nu_{a}\|\nu^\prime_{a})$ is the total-variation distance between $\nu_a$ and $\nu^\prime_a$.
\end{lemma}

Based on the above lemma, we can get the upper bound of the KL-Divergence between the marginals.
$$
\begin{aligned}
   \text{KL}\left(\mathbb{P}_{\mathcal{M}{\bar P}_1}^K \|\mathbb{P}_{\mathcal{M}{\bar P}_i}^K\right) & \le 6\epsilon \mathbb{E}_{\mathcal{M}{P}_1}\left[\sum_{t=1}^K \text{TV}(\nu_{a_t}\|\nu^\prime_{a_t})\right] \\
   & \le 6\epsilon \mathbb{E}_{\mathcal{M}{P}_1}[N_i(K)]\text{TV}(\nu_i\|\nu^\prime_{i})
\end{aligned}
$$
since $\bar P_1$ and $\bar P_i$ only differ in the arm $i$.

Thus, $$\begin{aligned}
    {Reg}_{K,\bar P_1}^{\mathcal{M}}+ {Reg}_{K,\bar P_i}^{\mathcal{M}} 
    & \ge \frac{K\Delta}{4} \exp{(-6\epsilon \mathbb{E}_{\mathcal{M}{P}_1}[N_i(K)] \cdot 2\Delta\gamma)}\\
    & \ge \frac{K\Delta}{4} \exp{\left(-\frac{12\cdot 5^{\frac{1}{v}} \epsilon K \Delta^{\frac{1+v}{v}}}{A-1}\right)}.
\end{aligned}$$
Taking $\Delta= \left(\frac{A-1}{K\epsilon}\right)^{\frac{v}{1+v}}$, we get the result
$${Reg}_{K,\bar P_1}^{\mathcal{M}} \ge \Omega\left(\left(\frac{A}{\epsilon}\right)^{\frac{v}{1+v}}K^{\frac{1}{1+v}}\right).$$
\end{proof}

\subsection{Proof of Theorem \ref{JDPlowerBoun}}
In order to give a lower bound of our problem in JDP, we first construct hard instances of MDPs as shown in Figure \ref{JDPfigure}. Based on these instances and inspired by \cite{vietri2020private}, we provide the lower bound by leveraging the lower bound in the above Theorem \ref{thm:LowBounMAB}. The key idea of the reduction from MDPs in JDP to MAB in DP is that we consider a setting where the initial state of each episode is public information. This means each user $k$ will release her/his first state $s_1^k$ in addition to sending it to the agent. Below we first define JDP algorithms for such a setting. 

\begin{definition}[$\epsilon$-JDP for RL with public initial state \cite{vietri2020private}]
We first define two sequences of inputs $(U_K,S_1)$ and $(U_K^\prime,S_1^\prime)$ for the RL agent as \textit{$k$-neighboring user-state sequences} if $u_{k^\prime}=u_{k^\prime}^\prime$ for all $k^\prime \neq k$ and $S_1=S_1^\prime$ where $S_1=(s_1^1,\dots,s_1^K)$ is the sequence of initial states. Then a randomized RL mechanism $\mathcal{M}$ is $\epsilon$-JDP under continual observation in the public initial state setting if for all $k \in [K]$, all $k$-neighboring user-state sequences $(U_K,S_1), (U_K^\prime,S_1^\prime)$ and all events $\mathcal{A}_{-k} \subset \mathcal{A}^{(K-1)H}$, we have
$
\mathbb{P}\left[\mathcal{M}_{-k}\left(U_{K},S_1\right) \in \mathcal{A}_{-k}\right] \leq e^\varepsilon \mathbb{P}\left[\mathcal{M}_{-k}\left(U_{K}^{\prime},S_1^\prime\right) \in \mathcal{A}_{-k}\right].
$
\end{definition}


\begin{lemma}[Lemma 11 in \cite{vietri2020private}]
\label{JDP2Public}
Any RL mechanism $\mathcal{M}$ satisfying $\epsilon$-JDP also satisfies $\epsilon$-JDP in the public initial state setting.
\end{lemma}

Based on the above lemma, the RL with heavy-tailed rewards under $\epsilon$-JDP problem is converted to the problem under $\epsilon$-JDP in the public initial state setting.

The relationship between  $\epsilon$-DP MAB mechanisms and  $\epsilon$-JDP MDP in the public initial state setting mechanisms  is the   following: 
We collect the first actions taken by the agent in all episodes $k$ with a fixed initial state $s_1^k=s \in [n]$ from an $\epsilon$-JDP mechanism for MDPs 
in the public initial state setting. And such an operation simulates the execution of an $\epsilon$-DP MAB algorithm. Specifically, let $\mathcal{M}$ be a JDP mechanism for MDPs with a public initial state  and $(U,S_1)$ be a user-state sequence with initial states from some set $S_1$. Let $\mathcal{M}(U,S_1)=(\vec{a}^1,\dots,\vec{a}^K) \in \mathcal{A}^{KH}$ be the collection of all outputs produced by the mechanism on inputs $U$ and $S_1$. For every $s\in S_1$ we denote trace $\mathcal{M}_{1,s}(U,S_1)$ as the restriction of the previous $\mathcal{M}(U,S_1)$ which just contains the first actions from all episodes starting with s together with the actions predicted by the policy at states $s$:
$
\mathcal{M}_{1, s}\left(U, S_1\right):=\left(a_1^{k_{s, 1}}, \ldots, a_1^{k_{s, K_s}}\right),
$
where $K_s$ is the number of occurrences of $s$ in $S_1$ and $k_{s,1},\dots,k_{s,K_s}$ are the indices of these occurrences. Furthermore, given $s \in S_1$ we write $U_s=(u_{k_{s,1}},\dots,u_{k_{s,K_s}})$ to denote the set of users whose initial state equals to $s$. Then we have the following result. 

\begin{lemma}[Lemma 9 in \cite{vietri2020private}]
\label{Public2MAB}
Let $(U,S_1)$ be a user-state input sequence with initial states from some set $S_1$. Suppose $\mathcal{M}$ is an RL mechanism that satisfies $\epsilon$-JDP in the public initial state setting. Then, for any $s \in S_1$ the trace $\mathcal{M}_{1, s}\left(U, S_1\right)$ is the output of an $\epsilon$-DP MAB mechanism on input $U_s$.
\end{lemma}

\begin{proof}[\bf Proof of Theorem \ref{JDPlowerBoun}]

We utilize the construction of hard MDP instances in Figure \ref{JDPfigure}. From Lemma \ref{JDP2Public} and Lemma \ref{Public2MAB}, we reduce the problem to learning $n=S-2$ MAB instances satisfying $\epsilon$-DP where each MAB is visited $K_s$ many times for all $s \in [S-2]$. Now we can use the result in Theorem \ref{thm:LowBounMAB} which states that for each initial state $s\in [n]$, the lower bound for the regret of any $\epsilon$-DP algorithm for the MAB problem with $A$ arms can be expressed as $\Omega\left(\left(\frac{A}{\epsilon}\right)^{\frac{v}{1+v}}K_s^{\frac{1}{1+v}}\right)$ where $K_s$ is the total number of arm pulls. Considering our construction of the MDP, a state is chosen uniformly at random at the start of the episode. By combining the regret corresponding to each initial state $s\in [n]$, the regret of the RL mechanism must be at least 
$$\Omega\left(\left(\frac{A}{\epsilon}\right)^{\frac{v}{1+v}}\sum_{s \in [S-2]}K_s^{\frac{1}{1+v}}\right)$$
where $K_s$ is a random variable. To establish a lower bound for the term $\sum_{s \in [S-2]}K_s^{\frac{1}{1+v}}$, we utilize the Markov inequality from Lemma \ref{MarkovIne}, resulting in:
$$\sum_{s \in [S-2]}K_s^{\frac{1}{1+v}}=(S-2)\mathbb{E}[K_s^{\frac{1}{1+v}}]\ge (S-2) \left(\frac{K}{S-2}\right)^{\frac{1}{1+v}}P\left[K_s^{\frac{1}{1+v}} \ge \left(\frac{K}{S-2}\right)^{\frac{1}{1+v}}\right].$$
The event $K_s^{\frac{1}{1+v}} \ge \left(\frac{K}{S-2}\right)^{\frac{1}{1+v}}$ occurs only when $K_s \ge \frac{K}{S-2}$. Since each $s\in [n]$ is chosen with equal probability at the beginning of the episodes, the expected number of pulls is $\mathbb{E}[K_s]=\frac{K}{S-2}$. Thus, each random variable $K_s$ follows a binomial distribution $Bin(K,\frac{1}{S-2})$ with mean $\frac{K}{S-2}$ therefore the probability that $K_s \ge \frac{K}{S-2}$ is $\frac{1}{2}$. By substituting this probability term, we can deduce that the total regret of the RL algorithm is lower bounded by:
$$\Omega\left(\left(\frac{SA}{\epsilon}\right)^{\frac{v}{1+v}}K^{\frac{1}{1+v}}\right).$$
\end{proof}

\subsection{Proof of Theorem \ref{LDPlowerBoun}}
\begin{proof}[\bf Proof of Theorem \ref{LDPlowerBoun}]

As in the case of Figure \ref{LDPfigure}, we have the transition probabilities for a unique action $a^*$ and leaf $x_{i^*}$ such that: 
\begin{equation}
\label{eq:Ins1}
   P(+|x_{i^*},a^*)=\gamma^{1+v} \ \text{and} \ P(-|x_{i^*},a^*)=1-\gamma^{1+v}. 
\end{equation}
where $\gamma^{1+v}\in (0,\frac{3}{4}]$.
 Each of the other leaves has  transition  probability
 \begin{equation}
 \label{eq:Ins2}
     P(+|x_{i},a)=\frac{1}{2}\gamma^{1+v} \ \text{and} \ P(-|x_{i},a)=1-\frac{1}{2}\gamma^{1+v}.
 \end{equation}
 We denote above instance by $\mathbb{P}_{(x_{i^*},a^*)}$.
 
In order to get the regret lower bound, we also consider another instance $\mathbb{P}_0$ where for all leaf states and any action, the transition probabilities are 
\begin{equation}
 P(+|x_{i},a)=\frac{1}{2}\gamma^{1+v} \ \text{and} \ P(-|x_{i},a)=1-\frac{1}{2}\gamma^{1+v}.
 \end{equation}
Based on the above transition probabilities, it's easy to check for each state-action pair, the $(1+v)$-th moment of reward is bounded by $1$ since the agent will receive the reward of $1/\gamma$ or $0$ at state $+$ or $-$ respectively. All other states have a reward of $0$ and every other transition is deterministic.

Then for a policy $\pi$, the value function can be written:
$$V^{\pi}(0)=\frac{1}{\gamma}P(s_{d+1}=+)=\frac{1}{\gamma}\left(\frac{1}{2}\gamma^{1+v}+\frac{1}{2}\gamma^{1+v}P(s_{d}=x_{i^*},a_{d}=a^*)\right).$$
Since $(x_{i^*},a^*)$ is the optimal state-action pair, the regret can be written as:
$$Reg(T)=\frac{1}{2}\gamma^v\left(K-\sum_{k=1}^K P(s_{d}^k=x_{i^*},a_{d}^k=a^*)\right)=\frac{1}{2}\gamma^v K\left(1-\frac{1}{K}\sum_{k=1}^K P(s_{d}^k=x_{i^*},a_{d}^k=a^*)\right)$$
where $\sum_{k=1}^K P(s_{d}=x_{i^*},a_{d}=a^*)=\mathbb{E}_{(x_{i^*},a^*)}\left[N_d^K(x_{i^*},a^*)\right]=\mathbb{E}_{(x_{i^*},a^*)}\left[\sum_{k=1}^K\mathbb{I}(s_d^k=x_{i^*},a_{d}^k=a^*)\right]$ and $\mathbb{E}_{(x_{i^*},a^*)}$ is the expectation on the instance described in equations \ref{eq:Ins1} and \ref{eq:Ins2}. Thus, we have 
\begin{equation}
\label{eq:regDecom}
   Reg(T)=\frac{1}{2}\gamma^v K\left(1-\frac{1}{K}\mathbb{E}_{(x_{i^*},a^*)}\left[N_d^K(x_{i^*},a^*)\right]\right). 
\end{equation}
$N_d^K(x_{i^*},a^*)$ is a function of the history observed by the algorithm. Since we consider the LDP setting, this history can be written as:$$\mathcal{M}(\mathcal{H}_K)=\{\mathcal{M}(X_\ell)|\ell\le K\}$$
where $X_\ell=\{(s_{\ell,h},a_{\ell,h},r_{\ell,h})|h\le H\}$ is the trajectory observed by the user for episode $\ell$ and $\mathcal{M}$ is a privacy mechanism which maintains $\epsilon$-LDP. Thus, $N_d^K(x_{i^*},a^*)$ is a function of $\mathcal{M}(\mathcal{H}_K)$.

Now we focus on getting upper bound on $\mathbb{E}_{(x_{i^*},a^*)}\left[N_d^K(x_{i^*},a^*)\right]$. Since $N_d^K(x_{i^*},a^*)$ is a function of $\mathcal{M}(\mathcal{H}_K)$ and $N_d^K(x_{i^*},a^*)/K \in [0,1]$, Lemma \ref{Lemm:KL} gives us 
$$\operatorname{kl}\left(\frac{1}{K}\mathbb{E}_0\left[N_d^K(x_{i^*},a^*)\right],\frac{1}{K}\mathbb{E}_{(x_{i^*},a^*)}\left[N_d^K(x_{i^*},a^*)\right]\right)\le \operatorname{KL}\left(\mathbb{P}_0(\mathcal{M}(\mathcal{H}_K))\|\mathbb{P}_{(x_{i^*},a^*)}(\mathcal{M}(\mathcal{H}_K))\right)$$
where $\mathbb{E}_0$ is the expectation on the instance where for all leaf states and any action, the transition probabilities are 
\begin{equation}
 P(+|x_{i},a)=\frac{1}{2}\gamma^{1+v} \ \text{and} \ P(-|x_{i},a)=1-\frac{1}{2}\gamma^{1+v}.
 \end{equation}
By Pinsker's inequality, $(p-q)^2 \le \frac{1}{2} \operatorname{kl}(p,q)$, it implies 
$$\frac{1}{K}\mathbb{E}_{(x_{i^*},a^*)}\left[N_d^K(x_{i^*},a^*)\right]\le \frac{1}{K}\mathbb{E}_0\left[N_d^K(x_{i^*},a^*)\right]+\sqrt{\frac{1}{2}\operatorname{KL}\left(\mathbb{P}_0(\mathcal{M}(\mathcal{H}_K))\|\mathbb{P}_{(x_{i^*},a^*)}(\mathcal{M}(\mathcal{H}_K))\right)}.$$
Using the chain rule we have:
$$
\operatorname{KL}\left(\mathbb{P}_0\left(\mathcal{M}\left(\mathcal{H}_K\right)\right) \| \mathbb{P}_{(x_{i^*},a^*)}\left(\mathcal{M}\left(\mathcal{H}_K\right)\right)\right)=\sum_{k=1}^K \mathbb{E}_{\mathcal{H}_{k-1} \sim \mathbb{P}_0}\left(\operatorname{KL}\left(\mathbb{P}_0\left(\cdot \mid \mathcal{M}\left(\mathcal{H}_{k-1}\right)\right) \| \mathbb{P}_{(x_{i^*},a^*)}\left(\cdot \mid \mathcal{M}\left(\mathcal{H}_{k-1}\right)\right)\right)\right)
.$$
where $\mathcal{M}\left(\mathcal{H}_{k-1}\right)$ means the results of privacy mechanism on history $\mathcal{H}_{k-1}$.

Because $\mathcal{M}$ is an $\epsilon$-LDP mechanism, from Theorem 1 in \cite{duchi2013local} we have 
$$
\operatorname{KL}\left(\mathbb{P}_0\left(\cdot \mid \mathcal{M}\left(\mathcal{H}_{k-1}\right)\right) \| \mathbb{P}_{(x_{i^*},a^*)}\left(\cdot \mid \mathcal{M}\left(\mathcal{H}_{k-1}\right)\right)\right) \leq 4(\exp (\varepsilon)-1)^2 \operatorname{KL}\left(\mathbb{P}_0\left(\cdot \mid \mathcal{H}_{k-1}\right) \| \mathbb{P}_{(x_{i^*},a^*)}\left(\cdot \mid \mathcal{H}_{k-1}\right)\right).
$$
 Thus $$\operatorname{KL}\left(\mathbb{P}_0\left(\mathcal{M}\left(\mathcal{H}_K\right)\right) \| \mathbb{P}_{(x_{i^*},a^*)}\left(\mathcal{M}\left(\mathcal{H}_K\right)\right)\right) \le 4(\exp (\varepsilon)-1)^2\operatorname{KL}\left(\mathbb{P}_0\left(\mathcal{H}_K\right) \| \mathbb{P}_{(x_{i^*},a^*)}\left(\mathcal{H}_K\right)\right) $$
 
 Lemma 5 in \cite{domingues2021episodic} ensures that:
 $$\operatorname{KL}\left(\mathbb{P}_0\left(\mathcal{H}_K\right) \| \mathbb{P}_{(x_{i^*},a^*)}\left(\mathcal{H}_K\right)\right)=\mathbb{E}_0\left[N_d^K(x_{i^*},a^*)\right]\operatorname{KL}(P_0(\cdot|x_{i^*},a^*)\|P_{(x_{i^*},a^*)}(\cdot|x_{i^*},a^*)).$$
 By using $\operatorname{KL}(\operatorname{Ber}(p)\|\operatorname{Ber}(q)) \le \frac{(p-q)^2}{q(1-q)}$, we obtain
 $$\operatorname{KL}(P_0(\cdot|x_{i^*},a^*)\|P_{(x_{i^*},a^*)}(\cdot|x_{i^*},a^*))=\operatorname{KL}\left(\operatorname{Ber}\left(\frac{\gamma^{1+v}}{2}\right)\|\operatorname{Ber}\left(\gamma^{1+v}\right)\right)\le \frac{\gamma^{1+v}}{4(1-\gamma^{1+v})}\le \gamma^{1+v}$$
 where the last inequality holds when $\gamma^{1+v} \in (0,\frac{3}{4}]$.
 According to the fact that $e^\epsilon-1 \approx \epsilon$ when $\epsilon$ is small, we have
 $$\frac{1}{K}\mathbb{E}_{(x_{i^*},a^*)}\left[N_d^K(x_{i^*},a^*)\right]\le \frac{1}{K}\mathbb{E}_0\left[N_d^K(x_{i^*},a^*)\right]+\sqrt{2\epsilon^2 \gamma^{1+v}\mathbb{E}_0\left[N_d^K(x_{i^*},a^*)\right]}.$$
 
 Now, let's assume that $I=(x_{i^*},a^*)$ is distributed uniformly over $\{x_1,\dots,x_L\}\times[A]$. That is to say, that the leaf $i^* \sim \mathcal{U}([L])$ and given the realization of $i^*$, $a^*$ is drawn uniformly in the action set of node $x_{i^*}$, i.e., $a^* \sim \mathcal{U}([A])$. we denote the expectation over the random variable $(x_{i^*},a^*)$ by $\mathbb{E}_I$. It then holds that:
 $$\mathbb{E}_I\mathbb{E}_0\left[N_d^K(x_{i^*},a^*)\right]=\mathbb{E}_0\sum_{k=1}^K\sum_{l=1}^L\sum_{a=1}^A \frac{1}{LA}\mathbb{I}\{s_d^k=x_l,a_d^k=a\}=\frac{K}{LA}.$$
 Then thanks to Jensen's inequality the regret in \eqref{eq:regDecom} is lower bound by 
 $$\mathbb{E}_I [Reg(T)]\ge \frac{1}{2}\gamma^v K \left(1-\frac{1}{LA}-\sqrt{\frac{2K\epsilon^2 \gamma^{1+v}}{LA}}\right).$$
 
 Take $\gamma=\left(\frac{LA}{32 K\epsilon^2}\right)^{\frac{1}{1+v}}$, we have 
 $$\max_{I\in \{x_1,\dots,x_L\}\times[A]}Reg(T) \ge \mathbb{E}_I [Reg(T)]\ge \Omega\left(\left(\frac{SA}{\epsilon^2}\right)^{\frac{v}{1+v}}K^{\frac{1}{1+v}}\right)$$
 where the last inequality holds since $L \ge (S-2)/2$.
\end{proof}

\section{Experiments}
In this section, we conduct proof-of-concept numerical experiments to verify our theoretical results for both policy-based and value-based algorithms.
\subsection{Setting}
We consider the standard tabular MDP environment \texttt{RiverSwim}~\citep{osband2013more}, illustrated in Fig.~\ref{fig:river}. It consists of six states and two actions `left' and `right', i.e., $S = 6$ and $A=2$. An agent starts with the left side and tries to reach the right side. At each step, if the agent chooses action `left', she will always succeed (the dotted arrow). Otherwise, the agent often fails (the solid arrow). The agent only receives a small reward of $0.005$ if she reaches the leftmost side, but obtains a large reward of $1$ once she arrives at the rightmost state. The agent gets a reward of $0$ for the intermediate states. Thus, this MDP naturally requires sufficient exploration to obtain the optimal policy. 
\begin{figure}[h]
\centering
\includegraphics[width=5in]{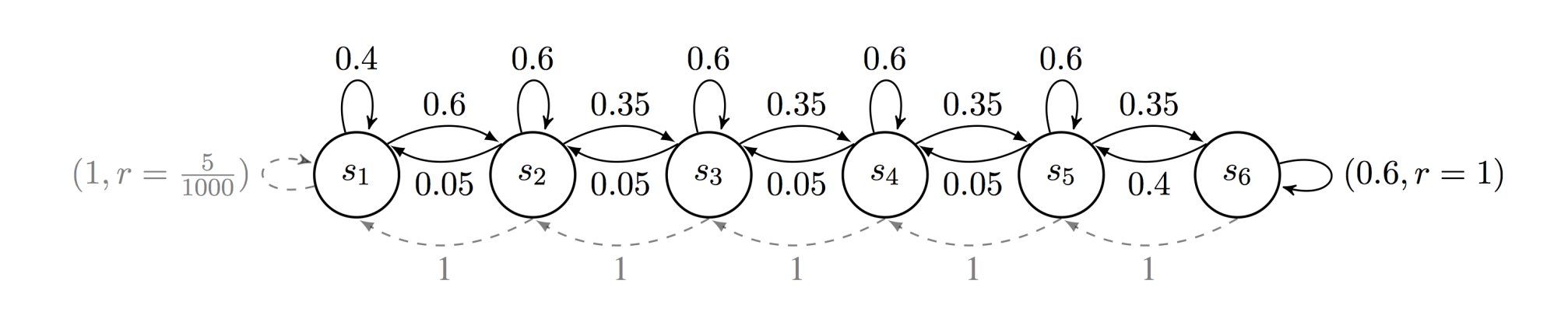}
\caption{\texttt{RiverSwim} MDP -- solid and dotted arrows denote the transitions under actions `right' and `left', respectively~\citep{osband2013more}. }
\label{fig:river}
\end{figure}

To generate heavy-tailed rewards, we use symmetric
$\alpha$-stable Levy distribution as in \cite{zhuang2021no}. The heaviness of the
tail is controlled by the parameter $\alpha$. In particular, $\alpha'$-th moments of $\alpha$-stable distributions are bounded for any $\alpha' \leq \alpha$. We denote this distribution as $\mathcal{L}(\alpha, \beta,\mu,\sigma)$, where $\beta$ is the skewness parameter, $\mu$ is the mean, and $\sigma$ is the shape parameter.
In all experiments, we set $\alpha = 2$ (i.e., the second moment of rewards is bounded). We consider only symmetric distributions (i.e., $\beta=0$) with unit shape (i.e., $\sigma=1$). We generate the heavy-tailed rewards corresponding to mean values $\mu \in \lbrace 0,1,0.005\rbrace$ as specified in the \texttt{RiverSwim} environment.

\subsection{Results}

\begin{figure}[ht]
\centering
  \includegraphics[width=0.49\linewidth]{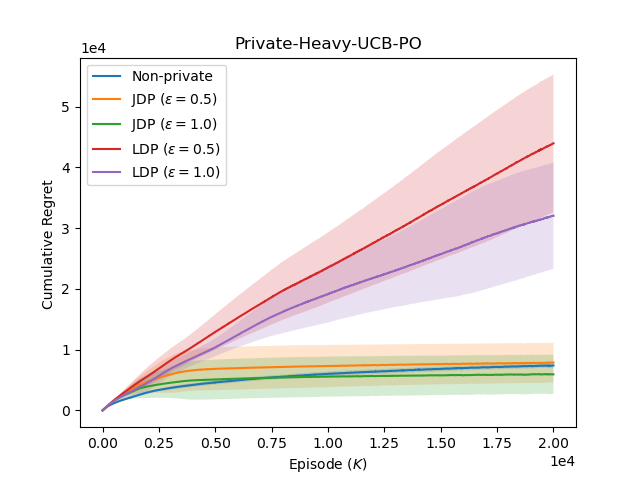}  
  \includegraphics[width=0.49\linewidth]{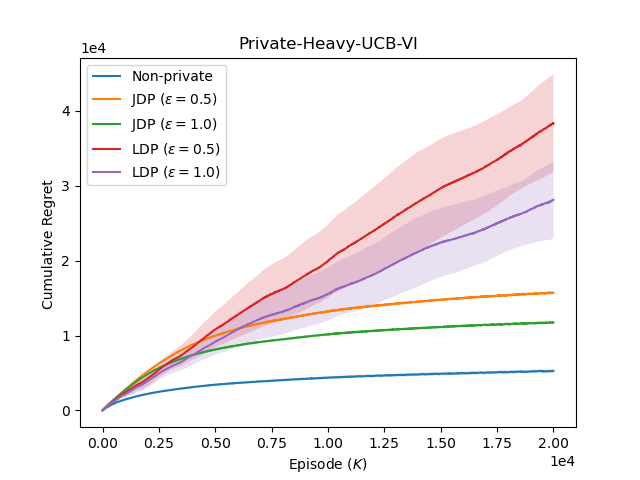}  
\caption{Cumulative regret vs. Episode for policy optimization and value iteration under heavy-tailed rewards with varying privacy levels $\epsilon \in \lbrace 0.5,1\rbrace$.}
\label{fig:sim}
\end{figure}

We evaluate both Private-Heavy-UCBVI and Private-Heavy-UCBPO under different privacy budgets $\epsilon$. As baselines, we design non-private UCBVI \citep{azar2017minimax} and OPPO \citep{shani2020optimistic} algorithms under heavy-tailed noise following the high-level approach of \cite{zhuang2021no}. 
We set all the parameters in our proposed algorithms in the same order as the theoretical results. We tune the learning rate $\eta$ and the scaling of the confidence interval to obtain the best results. We run $10$ independent experiments, each consisting of $K = 2 \cdot 10^4$ episodes. Each episode is reset every $H=20$ step. We plot the average cumulative regret along with the standard deviation for each setting, as shown in Fig.~\ref{fig:sim}

As suggested by our theoretical results, in both PO and VI experiments, we see that the cost of privacy under JDP becomes negligible as the number of episodes increases (since JDP doesn't increase the order of regret). However, under the stricter LDP requirement, the cost of privacy remains high (since LDP results in a higher-order term in regret). Furthermore, it is worth noting that the cost of privacy increases as the protection level increases, i.e., the value of $\epsilon$ decreases.

\end{document}